\documentclass{article}

\PassOptionsToPackage{numbers, compress}{natbib}

\usepackage[preprint]{neurips_2025}

\usepackage[utf8]{inputenc} 
\usepackage[T1]{fontenc}    
\usepackage{hyperref}       
\usepackage{amsthm}
\usepackage{url}            
\usepackage{booktabs}       
\usepackage{amsfonts}       
\usepackage{nicefrac}       
\usepackage{microtype}      
\usepackage[table]{xcolor}         

\usepackage{xspace}
\usepackage{fancyvrb}
\usepackage{bbm}
\usepackage{hyperref}
\usepackage[ruled,linesnumbered,noend]{algorithm2e}
\usepackage{amsmath}
\usepackage{cleveref}

\usepackage{listings} 
\usepackage{alltt}
\usepackage{latexsym}
\usepackage{url}
\usepackage{paralist}
\usepackage{pdfpages}
\usepackage{listings} 
\usepackage{booktabs}  
\usepackage{subcaption} 
\usepackage{hhline}
\usepackage{multicol}
\usepackage{makecell}
\usepackage{todonotes}
\usepackage{wrapfig}

\usepackage{forest}
\usepackage{siunitx}
\usepackage{float}

\title{Constrained Sampling for Language Models Should Be Easy: An MCMC Perspective}

\author{%
  Emmanuel Anaya Gonzalez$^{1}$\thanks{Equal contribution.} \quad
  Sairam Vaidya$^{1}$\footnotemark[1] \quad
  Kanghee Park$^{1}$ \quad
  Ruyi Ji$^{2}$ \\
  \textbf{Taylor Berg-Kirkpatrick}$^{1}$ \quad
  \textbf{Loris D'Antoni}$^{1}$ \\
  $^1$UCSD \quad $^2$ Peking University \\
  \texttt{\{fanayagonzalez,smahadevaganapathy,kap022,tberg,ldantoni\}@ucsd.edu} \\
  \texttt{jiruyi910387714@pku.edu.cn} \\
}

\begin{document}

\newcommand{\scaps}[1]{\textsc{#1}\xspace}
\newcommand{\tool}{}
\newcommand{\miss}{\textcolor{magenta}{MISSING}\xspace}

\newcommand{\mcL}{\mathcal{L}}
\newcommand{\opn}{\operatorname}

\newcommand\done[1]{\checkmark \sout{#1}}

\newcommand\taylor[1]{\nbc{TB}{#1}{tbcolor}}
\newcommand\nadia[1]{\nbc{NP}{#1}{npcolor}}
\newcommand\loris[1]{\nbc{LD}{#1}{ldcolor}}
\newcommand\kh[1]{\nbc{KP}{#1}{brown}}
\newcommand\eag[1]{\nbc{EAG}{#1}{eagcolor}}
\newcommand\sairam[1]{\nbc{SV}{#1}{svcolor}}
\newcommand\ruyi[1]{\nbc{RJ}{#1}{rycolor}}

\newcommand{\khchanged}[1]{\textcolor{violet}{#1}}

\newcommand{\indicator}{\mathbbm{1}}
\newcommand{\lang}{\mathcal{L}}
\newcommand{\tvd}[2]{\Vert #1, #2 \Vert_{\textsc{TV}}}
\newcommand{\dis}[2]{\text{KL}(#1 \Vert #2)}
\newcommand{\outprob}[1]{\probgrammarpg{\prob}{\mathcal O}_{#1}}
\newcommand{\langpre}{\lang_{\mathrm{prefix}}}
\newcommand{\ruleset}{\mathcal{R}}
\newcommand{\start}{S}
\newcommand{\nterm}{A}
\newcommand{\sent}{w}
\newcommand{\ev}{\mathbb{E}}
\newcommand{\prob}{P}
\newcommand{\probgrammar}{\prob^{\grammar}}
\newcommand{\probgrammarpg}[2]{{#1}^{#2}}
\newcommand{\probgcd}{\tilde\prob^\grammar_{\mathrm{GCD}}}
\newcommand{\evgrammar}{c}
\newcommand{\tevgrammar}{\tilde c}
\newcommand{\step}{\Rightarrow}
\newcommand{\manystep}{\Rightarrow^*}

\newcommand{\vocab}{\mathcal{V}}

\newcommand{\grammatical}{V_\grammar}
\newcommand{\nongrammar}{NG}

\newcommand{\nbc}[3]{
	{\colorbox{#3}{\bfseries\sffamily\scriptsize\textcolor{white}{#1}}}
	{\textcolor{#3}{\sf\small \textit{#2}}}
}

\definecolor{npcolor}{RGB}{255,0,255}
\definecolor{tbcolor}{RGB}{0,150,0}
\definecolor{ldcolor}{RGB}{200,30,50}
\definecolor{rycolor}{RGB}{150, 0, 0}
\definecolor{eagcolor}{RGB}{230,120,20}
\definecolor{svcolor}{RGB}{0,102,204}

\def\sectionautorefname{Sec.}
\def\subsectionautorefname{Sec.}
\def\subsectionautorefname{Sec.}
\def\subsubsectionautorefname{Sec.}
\def\figureautorefname{Fig.}
\def\tableautorefname{Tab.}
\def\equationautorefname{Eq.}
\def\algorithmautorefname{Alg.}

\newtheorem{example}{Example}
\def\exampleautorefname{Ex.}
\newtheorem{theorem}{Theorem}
\def\theoremautorefname{Thm.}

\newcommand{\etc}{\emph{etc}\xspace}
\newcommand{\ie}{\emph{i.e.\@}\xspace}
\newcommand{\Ie}{\emph{I.e.\@}\xspace}
\newcommand{\eg}{\emph{e.g.\@}\xspace}
\newcommand{\Eg}{\emph{E.g.\@}\xspace}
\newcommand{\etal}{\emph{et~al.\@}\xspace}
\newcommand{\vs}{\emph{vs.\@}\xspace}
\newcommand{\wrt}{\emph{wrt.\@}\xspace}

\newcommand{\tname}[1]{\textsc{#1}\xspace}

\newcommand{\asap}{ASAp\xspace}
\newcommand{\algo}{S\xspace}

\newcommand{\sygus}{\textsc{SyGuS}\xspace}

\newcommand{\samples}{S}
\newcommand{\grammar}{\mathcal{G}}
\newcommand{\nonterms}{\mathcal{N}}
\newcommand{\terms}{\Sigma}


\lstdefinelanguage{nolang}{
	literate=%
    {->}{$\rightarrow$}2
    {=.=}{$\doteq$}1
    {==}{$=$}1
    {!=}{$\neq$}1
    {&&}{$\land$}1
    {||}{$\lor$}1
    {<}{$<$}1
    {>}{$>$}1
    {<=}{$\le$}1
    {>=}{$\ge$}1
	,
	numbers=none,
	basicstyle=\ttfamily,
	commentstyle=\itshape\color{commentgreen},
	keywordstyle=\bfseries,
	ndkeywordstyle=\bfseries
}

\lstset{
  language=nolang,
  floatplacement={tbp},
  captionpos=b
}


\definecolor{terminal}{RGB}{0,100,0} 
\definecolor{nonterminal}{RGB}{100,0,100} 
\definecolor{operator}{RGB}{0,0,200}  

\newcommand{\TERM}[1]{\textcolor{terminal}{{#1}}}
\newcommand{\NT}[1]{\textcolor{nonterminal}{{#1}}}
\newcommand{\OP}[1]{\textcolor{operator}{{#1}}}

\maketitle

\begin{abstract}
Constrained decoding enables Language Models (LMs) to produce samples that provably satisfy hard constraints.
However, existing constrained-decoding approaches often distort the underlying model distribution, a limitation that is especially problematic in applications like program fuzzing, where one wants to generate \textit{diverse} and \textit{valid} program inputs for testing purposes. 
We propose a new constrained sampling framework based on Markov Chain Monte Carlo (MCMC) that simultaneously satisfies three core desiderata: \textit{constraint satisfying} (every sample satisfies the constraint), \textit{monotonically converging} (the sampling process converges to the true conditional distribution), and \textit{efficient} (high-quality samples emerge in few steps). Our method constructs a proposal distribution over valid outputs and applies a Metropolis-Hastings acceptance criterion based on the LM's likelihood, ensuring principled and efficient exploration of the constrained space. Empirically, our sampler outperforms existing methods on both synthetic benchmarks and real-world program fuzzing tasks.
\end{abstract}

\section{Introduction}
\label{sec:introduction}

Language Models (LMs) have revolutionized a wide range of domains, from code generation~\cite{humaneval} to automated reasoning~\cite{yang2023leandojo}.
Yet, ensuring that their outputs satisfy hard structural constraints---such as syntactic validity in domain-specific languages---remains an significant challenge~\cite{geng2025jsonschemabenchrigorousbenchmarkstructured}. 
While constrained decoding methods~\cite{lmql, dong2022codepad, geng2024grammarconstrained, poesia2022synchromesh, picard, shin2021constrained, stengel2023zero, ugare2024improving,park2025flexibleefficientgrammarconstraineddecoding,AgrawalKGLR23, li2024guiding, wang2023grammar} can enforce these constraints, they often distort the underlying generative distribution learned by the model, degrading performance in downstream tasks~\cite{tam-etal-2024-speak,park2024grammaraligned}. 

This tradeoff is particularly detrimental in applications that rely not on a single high-quality sample, but on \textit{diverse} samples from the constrained distribution. 
A compelling example is \textit{program fuzzing}, a technique for discovering software bugs by automatically generating test inputs that explore different execution paths in a program. 
Modern fuzzers bootstrap this process using a small set of seed inputs, and the effectiveness of these seeds hinges on both their correctness---the seeds should be syntactically valid inputs that the program will not reject---and their distributional diversity---the seeds should exercise different execution paths.
LMs offer a powerful mechanism to generate such seeds---but only if we can sample efficiently and faithfully from the constrained distribution. 


Can we design \textit{constrained sampling algorithms} that when executed for a given amount of time $t$ are
\textit{(i) Constraint satisfying:} produce a sample within $t$ time and that sample satisfies the constraint; 
\textit{(ii) Monotonically Converging:} converge to the true conditional distribution when $t$ goes to infinity; and  
\textit{(iii) Efficient:} produce good estimates of the constrained distribution for low values of $t$?

\noindent We answer this question in the affirmative. Focusing on constraints expressed as \textit{context-free grammars}, we introduce a family of sampling algorithms rooted in Markov Chain Monte Carlo (MCMC) techniques.
Our key insight is to construct proposal distributions that generate only constraint-satisfying samples and use a Metropolis-Hastings criterion guided by the LM's likelihood function to accept/reject candidates.
Unlike rejection sampling, every candidate in our method is valid by construction---ensuring \textit{constraint satisfaction}.
Our use of likelihood-aware MCMC transitions ensures \textit{monotonic convergence}, and crucially, our empirical results reveal that the resulting chains converge \textit{efficiently} to the desired conditional distribution.

We make the following contributions.
First, we formalize the \textit{desiderata} for constrained sampling---constraint satisfying, monotonically converging, and efficient---and show that existing methods fail to satisfy all three (\Cref{sec:background}).
Second, we introduce an effective \textit{MCMC-based framework} for constrained sampling that satisfies all three desiderata and instantiate it with three concrete proposal distributions~(\Cref{sec:mcmc}).
Third, we validate our approach on both synthetic distributions and real-world fuzzing targets such as libxml2 and sqlite. The samples produced from our approach exhibit better KL divergence from the target distribution than competing approaches. Most importantly, program fuzzers (i.e., automated random testers) seeded with our samples consistently achieve higher code coverage on real-world tasks compared to fuzzers seeded using existing approaches (\Cref{sec:evaluation}).

\section{The Ideal Properties of Constrained Sampling}
\label{sec:background}

In this section, we formalize the problem of sampling from a language model (LM) conditioned on a constraint (i.e., constrained sampling), define our key desiderata of a good constrained sampling algorithm, and describe how existing constrained sampling algorithms do not meet such desiderata.
We follow the definitions proposed by \citet{park2024grammaraligned}.

\paragraph{Language Models.}

An (autoregressive) \textit{language model} defines a probability distribution $\prob$ over sequences of tokens (i.e., sentences) $w \in \vocab^\ast$, where $\vocab$ denotes the vocabulary.  
The probability of a sequence is computed as the product of conditional probabilities for each token in the sequence: 
\[\prob(\sent_1 \ldots \sent_n) = \Pi^n_{i=1} \prob(\sent_i \mid \sent_{1:i-1})\]

\paragraph{Constraints and Grammars.}
Given a LM $\prob$ and a constraint $\varphi$, the goal of constrained sampling is to sample sequences that satisfy the constraint.
Constraints are typically specified as a language, which could be a regular language, context-free grammar (CFG), or other types of logical conditions that sequences must fulfill.
In this work, we focus on constraints that can be defined by a context-free grammar, an expressive and formal way to define sets of valid sequences.
For example, context-free grammars can express what programs in a given programming language are syntactically valid and what format a JSON object should abide to when being used to transfer data.
While we focus on grammars, the techniques we propose can be applied to other types of constraints.

Formally, a context-free grammar $\grammar = (\terms, \nonterms, \start, \ruleset)$ consists of a set of terminal symbols $\terms$, a set of non-terminal symbols $\nonterms$, a start nonterminal symbol $\start \in \nonterms$, and a set of production rules $\ruleset$. 
A sequence $\sent$ is a valid sentence belonging to the language $\lang(\grammar)$ if it is derivable from the start symbol $\start$ by applying a sequence of rules from $\ruleset$. 
Each step in this derivation transforms a string $\alpha \nterm \gamma$ into $\alpha \beta \gamma$ using a rule $\nterm \to \beta \in \ruleset$, and denoted by $\alpha \nterm \gamma \Rightarrow \alpha \beta \gamma$.

\begin{figure}
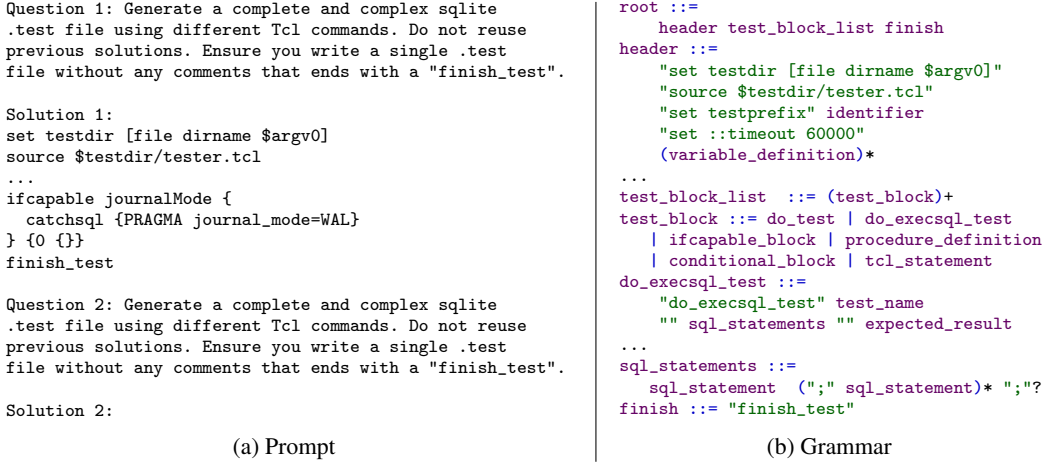

    \centering
    \scriptsize
    \begin{subfigure}{0.53\textwidth}
        \centering
        \input{codes/sqlite-prompt}
        \vspace{-0.5em}
        \caption{Prompt}
        \label{fig:fuzzing-prompt}
    \end{subfigure}
    \hspace{1.5em}\vrule\hspace{1em}
    \begin{subfigure}{0.4\textwidth}
        \centering
        \input{codes/sqlite-grammar}
        \vspace{-0.5em}
        \caption{Grammar}
        \label{fig:fuzzing-grammar}
    \end{subfigure}
       
    \caption{(a) Prompt to generate seed test cases for fuzzing the SQLite engine.
    (b) Simplified version of the SQLite test-script grammar written in EBNF notation.
    The goal of the problem is to generate multiple diverse seeds that trigger different code paths in the library being tested.
} 
    \label{fig:fuzzing-example}
\end{figure}

\begin{example}[SQLite Test-Script Grammar]
\Cref{fig:fuzzing-example} illustrates a typical use-case: we ask a language model to generate SQLite regression test files ({\tt{.test}}) that drives the SQLite engine down as many distinct execution paths as possible (\Cref{fig:fuzzing-prompt}).
To exercise a specific component of the database, each file must satisfy
the syntactic and semantic restrictions encoded in the SQLite
test-script grammar shown in \Cref{fig:fuzzing-grammar}---e.g., the file should include a mandatory timeout directive {\tt{set ::timeout 60000}} in the header and well-formed {\tt{do\_execsql\_test}} blocks that wrap one or
        more SQL statements in braces and specify an expected result.

A complete derivation therefore starts from the non-terminal
{\tt{root}}, transforms (i.e., using $\Rightarrow$) into {\tt{header}} {\tt{test\_block\_list}}
{\tt{finish}}, and continues recursively until only quoted
terminals remain. (See Appendix~\ref{app:parse-tree} for a condensed parse tree.)
\end{example}

\paragraph{Grammar-Aligned Sampling.}

Grammar-aligned sampling aims to sample sequences from $\prob$ that belong to the language $\lang(\grammar)$, while preserving the model's underlying distribution.
This can be viewed as sampling from the constrained distribution $\probgrammar$, which is proportional to the original model distribution but restricted to sequences that satisfy the constraint. Mathematically, for a given grammar $\grammar$ and model $\prob$, we want to sample from
\[
\probgrammarpg{\prob}{\grammar}(\sent)  = \frac{\indicator[\sent \in \lang(\grammar)] \cdot \prob(\sent)}{\sum_{\sent'} \indicator[\sent' \in \lang(\grammar)] \cdot \prob(\sent')}
\]

\subsection{Limitations of Existing Approaches}
\label{sec:limitations-of-existing-approaches}

\paragraph{Rejection Sampling.}
A common method for obtaining constrained samples from a language model is \textit{rejection sampling}, which repeatedly draws outputs from the model and discards ones that do not satisfy the constraint.
While samples accepted via this process correctly follow the language model's distribution conditioned on the constraint,  there is no guarantee on how many samples one needs to reject before getting a sample satisfying the constraint.
This inefficiency becomes especially pronounced when the constraint describes a pattern that is infrequent or not naturally favored by the model's learned distribution, often requiring many rejections before a valid sample is drawn. 
For example, for the problem in \Cref{fig:fuzzing-example}, out of 500 samples produced by \texttt{Llama-3.1-8B-Instruct}, only 2 (0.4\%) satisfied the constraint imposed by the grammar.

\paragraph{Constrained Decoding.} 

The inefficiency of rejection sampling has led to \textit{constrained decoding algorithms} \cite{picard, lmql, geng2024grammarconstrained} that at each decoding step evaluate the LM next tokens against the specified constraints. 
Invalid tokens that cause the generated sequence to not satisfy the given constraint are masked from the probability distribution, forcing the model to select tokens that will lead to \textit{constraint-satisfying} sequences. 
In particular, when the constraint is a context-free grammar, the technique is known as Grammar-Constrained Decoding (GCD) \cite{geng2024grammarconstrained}. 

As shown by \citet{park2024grammaraligned}, constrained decoding \textit{does not preserve the underlying distribution of the model}. 
If we define the prefix language $\langpre(\grammar) = \{\sent \in \terms^* \mid \sent v \in \lang(\grammar)\}$ of a grammar $\grammar$ as the set containing all possible prefixes of sentences in the grammar's language, the distribution captured by GCD is the following incorrect conditional distribution:
\[
\probgcd(\sent_i \mid \sent_{1:i-1}) = \frac{\prob(\sent_i\mid \sent_{1:i-1}) \cdot \indicator[\sent_{1:i} \in \langpre(\grammar)]}{\sum_{\sent'_i} \prob(\sent'_i\mid \sent_{1:i-1}) \cdot \indicator[\sent_{1:i-1},\sent'_i \in \langpre(\grammar)]}
\]
For example, for the problem in \Cref{fig:fuzzing-example}, regardless of how many samples one generates using GCD with \texttt{Llama-3.1-8B-Instruct} the empirical KL divergence between the samples and \(\probgrammarpg{\prob}{\grammar}\) does not decrease.
If we assume that the LM is good at sampling good seeds for a fuzzer, this unfaithfulness to the target distribution results in worse samples that cover fewer code paths (as shown in~\Cref{sec:evaluation:fuzzing}).

\paragraph{Adaptive Sampling with Approximate Expected Futures (ASAp).}

\citet{park2024grammaraligned} showed how to correct next-token conditional distribution for grammar-aligned sampling using the notion of \textit{Expected Future Grammaticality} (EFG), defined as 
$\evgrammar(\sent_{1:i})= \ev_{\prob(\sent_{i+1:n} \mid \sent_{1:i})} [\indicator[\sent \in \lang(\grammar)]]$---i.e., the probability that sampling a continuation of the  prefix $\sent_{1:i}$ will lead to a valid sequence in the grammar. 
The conditional probability required by grammar-aligned sampling can be then written as:
\begin{equation}
\label{eq:prob-gad-by-weight}
\probgrammarpg{\prob}{\grammar}(\sent_i \mid \sent_{1:i-1}) = \frac{\prob(\sent_i \mid \sent_{1:i-1}) \cdot c(\sent_{1:i})}{\sum_{\sent'_i} \prob(\sent'_i \mid \sent_{1:i-1}) \cdot \evgrammar(\sent_{1:i-1}, \sent'_i)}
\end{equation}
\citet{park2024grammaraligned} proposed Adaptive Sampling with Approximate expected futures (ASAp) to approximate grammar-aligned sampling. ASAp iteratively overapproximates expected future grammaticality by removing probability mass associated with invalid prefixes identified from previous samples. 
While in the limit this approach reaches the desired distribution, it does not do so \textit{monotonically}---i.e., it can produce intermediate EFG approximations that are very skewed.
\citet{park2024grammaraligned} empirically showed that it can take thousands of samples for ASAp to start converging, making the algorithm practically \textit{not efficient}.
For example, for the problem in \Cref{fig:fuzzing-example}, 100 samples generated using ASAp with \texttt{Llama-3.1-8B-Instruct} exhibited worse empirical KL divergence from \(\probgrammarpg{\prob}{\grammar}\) than even GCD!

\subsection{Desired Properties of a Grammar-Aligned Sampler} \label{section:desired-properties}

We formulate properties we claim a good constrained-sampling algorithm should satisfy.
Because convergence in the limit is impractical, we assume that an algorithm $\algo$ is given a finite amount of time $t$ to produce a sample, and we want the algorithm to satisfy three properties:
\begin{description}
  \item[Constraint Satisfying:] $\algo$ produces a sample within $t$ time and that sample satisfies the constraint; 
  \item[Monotonically Converging:] the total variance distance between the output distribution of $S$ and the target distribution $\probgrammarpg{\prob}{\grammar}$ monotonically decreases and converges to $0$ as $t$ approaches infinity;
  \item[Efficient:] $\algo$ produces good estimates of the constrained distribution for low values of $t$.
\end{description}

By focusing on these properties, we aim to provide a sampling approach that balances efficiency, reliability, and correctness, addressing the shortcomings of existing methods, which fail to achieve all three goals simultaneously. In \Cref{sec:mcmc}, we present our MCMC-based approach for constrained decoding, which provably guarantees the first two properties and is in practice efficient (\Cref{sec:evaluation}).
%






\section{Grammar-Aligned MCMC Sampling}
\label{sec:mcmc}

We propose a constrained sampling framework based on Markov Chain Monte Carlo (MCMC) that operates strictly within the space of grammar-valid sequences. Rather than relying on rejection of ungrammatical outputs, our method uses grammar-constrained decoding (GCD) to iteratively refine samples through local proposals that are always constraint-satisfying---a property guaranteed by GCD. This framework provides a principled mechanism to balance computational cost and sampling fidelity: as the chain progresses, it converges toward the true grammar-aligned distribution.
Intuitively, our approach uses MCMC to turn any GCD implementation into a grammar-aligned sampler.


\subsection{Constrained Generation via Metropolis-Hastings}

We use the Metropolis-Hastings algorithm \cite{hastings1970}, a standard MCMC method, to construct a Markov chain whose stationary distribution matches a desired target $\pi$ over a set of states $S$.
Given a \textit{proposal distribution} $q(y \mid x)$, which defines how to sample a candidate state $y$ from a current state $x$, the algorithm accepts the proposed candidate with probability $\alpha(x, y)$, defined as:
\begin{gather}
    \alpha(x, y) 
        = \min \left\{ 
            1,  
            \frac
                {\pi(y) q(x \mid y)}
                {\pi(x) q(y \mid x)}
        \right\}
    \label{eq:acc_prob}
\end{gather}
This acceptance rule compares how likely $y$ is under the target distribution to how likely $x$ is, adjusted by the relative likelihood of proposing each direction, $q(x \mid y)$ and $q(y \mid x)$. Intuitively, proposals that improve the target probability are usually accepted, while worse ones are accepted with a controlled probability to maintain exploration. The acceptance rule ensures detailed balance—a property that guarantees the target distribution is stationary for the Markov chain—which in turn ensures that the chain will converge to the desired distribution over time.

\SetKwProg{Fn}{Function}{:}{}
\SetKwFunction{Prop}{Propose}
\SetKwFunction{Alpha}{AcceptProb}
\SetKwFunction{Bern}{Bernoulli}
\SetKwFunction{Sample}{SamplePrefix}
\SetKwFunction{Generate}{Generate}
\SetCommentSty{normalfont}
\begin{algorithm}[t]
    \caption{The Metropolis-Hastings algorithm instantiated for grammar-aligned sampling.}
    \label{alg:MH}
    \KwData{the LM $P$, the grammar $\mathcal G$, a parameter $k$ denoting the chain length, and a configurable distribution $p_{\textsc{pos}}^\sent$ for sampling a random prefix from a given sequence.}
    \KwResult{a proposed sequence $s'$.}
    \LinesNumbered
    \begin{multicols}{2}
    $\sent_0 \gets $ a GCD sample; \\ 
    \ForEach{$i \in 1\dots k$}{
         $\sent \gets$ \Prop{$\sent_{i-1}$}; \\ 
         $a \gets \alpha(\sent_{i-1}, \sent$); \\
         $b \gets$ \Bern{$a$}; \\
         \leIf{$b$}{$\sent_i \gets \sent$}{$\sent_i \gets \sent_{i-1}$}
     }
     \Return $\sent_k$; \\
     \Fn{\Prop{$\sent$}}{
         $i \gets $ an index sampled from $p_{\textsc{pos}}^\sent$; \\
         $\sent^p \gets \sent_{1:i}$; \\
         \Return a GCD sample with the prefix $\sent^p$; \\
     }
     \nonumber 
     \vspace{3.2em}
     \nllabel{noline}
     \end{multicols}
    \vspace{0.3em}
\end{algorithm}

\autoref{alg:MH} presents our instantiation of the Metropolis-Hastings algorithm for grammar-aligned sampling, where the target distribution is the constrained distribution $\probgrammarpg{\prob}{\grammar}(\sent)$.
The algorithm starts with a random GCD sample (Line 1)---which is guaranteed to be constraint-satisfying---and refines it toward the target distribution $\probgrammarpg{\prob}{\grammar}(\sent)$ by running the constructed Markov chain for $k$ steps (Lines 2--7).
In each step, the algorithm simulates the Markov chain by first drawing a new sample from the proposal distribution (Line 3) and then accepting it with probability $\alpha(\sent_{i-1}, \sent)$ (Lines 4--6).
Finally, the last sample is returned as the result (Line 8).

One important point to note is the computation of the acceptance probability $\alpha(\sent_{i-1}, \sent')$ (Line 4), whose definition relies on the target probabilities $\probgrammarpg{\prob}{\grammar}(\sent_{i-1})$ and $\probgrammarpg{\prob}{\grammar}(\sent')$.  
However, these probabilities are difficult to compute in practice because their normalization factor $\sum_{\sent'} \mathbbm{1}[\sent' \in \mathcal{L}(\mathcal{G})] \cdot P(\sent')$ requires summing over the whole language, which is typically intractably large.  
Fortunately, this factor cancels out from the acceptance probability after unfolding the target probabilities, as shown below, which induces an efficient implementation for the function $\alpha$:
\begin{align*}
 \alpha(\sent, \sent')
        = \min \left\{ 
            1,  
            \frac
                {\probgrammarpg{\prob}{\grammar}(\sent') q(\sent \mid \sent')}
                {\probgrammarpg{\prob}{\grammar}(\sent) q(\sent' \mid \sent)}
        \right\}
        = \min \left\{ 
            1,  
            \frac
                {\indicator[\sent' \in \lang(\grammar)] \cdot \prob(\sent') \cdot q(\sent \mid \sent')}
                {\indicator[\sent \in \lang(\grammar)] \cdot \prob(\sent) \cdot q(\sent' \mid \sent)}
        \right\}
\end{align*}

The last component in \autoref{alg:MH} is the proposal distribution (Line 3).
The Metropolis-Hastings algorithm offers considerable flexibility in this distribution — any choice satisfying some mild conditions (i.e., irreducibility and aperiodicity) can yield a theoretically sound sampler.
Hence, we propose a parameterized family of proposal distributions for grammar-aligned sampling to fit different scenarios (Lines 9–12), which we elaborate on in \Cref{subsection:proposal}.
 
\subsection{Proposal Distributions for Grammar-Aligned MCMC} \label{subsection:proposal}
The performance of MCMC for grammar-aligned sampling is tied to the proposal distribution. Different proposals make different trade-offs: some are simple to implement but mix slowly, while others complex ones may offer faster convergence. There is no single best choice across all settings. 

In this section, we present a parameterized family of proposal distributions that is \textit{(i)} tailored to sequences generated by a language model, and \textit{(ii)} maintains constraint satisfaction by construction. 
As shown in Lines 9–12 of \autoref{alg:MH}, the proposal mechanism operates by selecting a prefix of the current sequence and resampling the remainder using GCD. Given a sequence $\sent$ and a distribution $p_\textsc{pos}^{\sent}$ over the token positions $[0, |\sent|]$, we sample an index $i$ from $p_\textsc{pos}^{\sent}$ and extract the prefix $\sent_{1:i}$. A new candidate $\sent'$ is then generated by running a grammar-constrained decoder (GCD) conditioned on this prefix. This strategy allows local, structured edits while ensuring that every proposed sequence remains within the constrained language.
For example, consider a grammar describing only sequences of 0s and 1s and let's assume GCD has produced the sequence 01001.
Our MCMC algorithm can sample any prefix of this sequence, let's say 010, and produce from it a continuation using GCD, e.g., it could produce the sequence 01011111. Whether the new sequence will be accepted depends on the likelihood of the two sequences according to the LM.

Under this framework, we prove \autoref{alg:MH} exhibits desired properties as a grammar-aligned sampler -- satisfying both the \textit{constraint satisfying} and \textit{monotonically converging} properties (\autoref{section:desired-properties}) if the truncation distribution $p_{\textsc{pos}}^\sent$ always assigns a non-zero probability to the empty prefix (\autoref{appendix:proofs}).
%

To conclude this section, we  describe three concrete instantiations of this framework, each corresponding to a different choice of prefix distribution, $p_{\textsc{pos}}^\sent$.

\textbf{Uniform.} This proposal distribution samples a truncation point uniformly at random from the positions in the current sequence---i.e.,  $p_{\textsc{pos}}^\sent$ is the uniform distribution.
    This proposal creates local moves that preserve partial structure while allowing for diversity in continuations. 

\textbf{Priority.}
    This proposal biases truncation toward parts of the sequence where the model is uncertain or poorly aligned with the grammar, enabling targeted refinement of weaker regions. 
    We use \textit{perplexity}, a common measurement of uncertainty, to carry in this bias and set the truncation distribution $p_{\textsc{pos}}^{\sent}$ to the LM's token-level perplexity, i.e., $p_{\textsc{pos}}^{\sent}(i) \propto \texttt{PP}(P(\cdot \mid \sent_{1:i}))$.
    In this way, positions with higher uncertainty will have a greater probability to be resampled.

\textbf{Restart.}
    This simplest proposal always discards the current sample and generates a new one from scratch using GCD---$p_{\textsc{pos}}^\sent$ selects position 0 with probability 1. Since the proposal is independent of the current state, this reduces to resampled importance sampling \cite{doucet2001smc}. 

\section{Experiments}
\label{sec:evaluation}

In this section, we show that MCMC achieves a better approximation to 
$\probgrammar$ than existing approaches.
In \autoref{sec:evaluation:synth} we demonstrate empirically MCMC converges to $\probgrammar$ using fewer samples than existing sampling approaches on the benchmarks proposed by~\citet{park2024grammaraligned}.
In \autoref{sec:evaluation:fuzzing} we show that MCMC improves the quality of fuzzing algorithm seeds over grammar-constrained decoding (GCD)~\cite{geng2024grammarconstrained} and adaptive sampling (ASAp)~\cite{park2024grammaraligned}.

We implemented our MCMC framework as an extension of the Transformers-GAD library~\cite{park2024grammaraligned}, 
which also provides the implementations of GCD and ASAp used in our evaluation.
%
%
We conduct our experiments using the Llama-3.1-8B model checkpoint provided by Hugging Face (commit 0e9e39f):
\url{https://huggingface.co/meta-llama/Llama-3.1-8B-Instruct}.

\subsection{Benchmarks by~\citet{park2024grammaraligned}}
\label{sec:evaluation:synth}

We compare the MCMC-based sampling algorithms against GCD and ASAp on the three tasks proposed by \citet{park2024grammaraligned} and their respective benchmarks.
Two of our tasks involve synthesizing expressions in an extension of linear integer arithmetic (SLIA) and loop invariants with bit-vector arithmetic (BV4). The problems are expressed as Syntax-Guided Synthesis Problems (SyGuS)~\cite{alur2019syguscomp}, a standardized format where a logical specification and a context-free grammar of first-order terms are provided and the goal is to synthesize a term in the grammar that satisfies the specification.
%
%
We use the same prompts as \citet{park2024grammaraligned}, consisting of 3 in-context examples of the form (specification, solution) and the grammar is then provided as a constraint for grammar-aligned sampling.
Our third task is the constituency parsing (CP) task already used in prior GCD work~\cite{geng2024grammarconstrained} where the grammar is used to help the model produce well-parenthesized parse trees for English sentences.
In total, the set contains 15 SLIA problems, 15 BV4 problems, and 6 CP problems.

We run all 3 variants of MCMC (uniform, priority, restart) for $n\in[1..10]$ steps each (we write MCMC-T ($k=n$) to denote that MCMC with proposal distribution T$ \in\{$Prefix, Priority, Restart$\}$ for $n$ steps). Similarly, we write ASAp ($k=n$).

\paragraph{Measures.} An ideal measure to compare sampling approaches is the distance between the sample distribution and the target $\probgrammar$.
It is, however, impractical because we can neither exhaust the infinite sequence domain nor evaluate any probabilities in $\probgrammar$.
Therefore, we follow \citet{park2024grammaraligned} to use an approximate measure instead, which is the KL divergence to the LM distribution $P$ on the finite set of all observed samples.
This approximation aligns well to the ideal measure because the LM's distribution $P$ is proportional to the target $\probgrammar$ on valid samples, as shown below:
\[
\text{KL}(\tilde\probgrammar \| \prob) {=}\ev_{\tilde\probgrammar}\left[ \log \frac{\tilde\probgrammar}{\prob} \right]
{=} \ev_{\tilde\probgrammar}\left[ \log \frac{\tilde\probgrammar}{C {\cdot} \probgrammar} \right]
{=} \ev_{\tilde\probgrammar}\left[ \log \frac{\tilde\probgrammar}{\probgrammar} \right] - \log C {=}  \text{KL}(\tilde\probgrammar \| \probgrammar) - \log C
\]
where $\tilde{\probgrammar}$ denotes the sample distribution, and $C$ is a constant.

We obtain 100 runs for each MCMC variant for each task, and plot the mean KL divergence and 95\% confidence interval ranges computed from boostrapping.

\begin{figure}
    \centering
    \begin{subfigure}{0.45\textwidth}
        \centering
        \includegraphics[width=.9\linewidth]{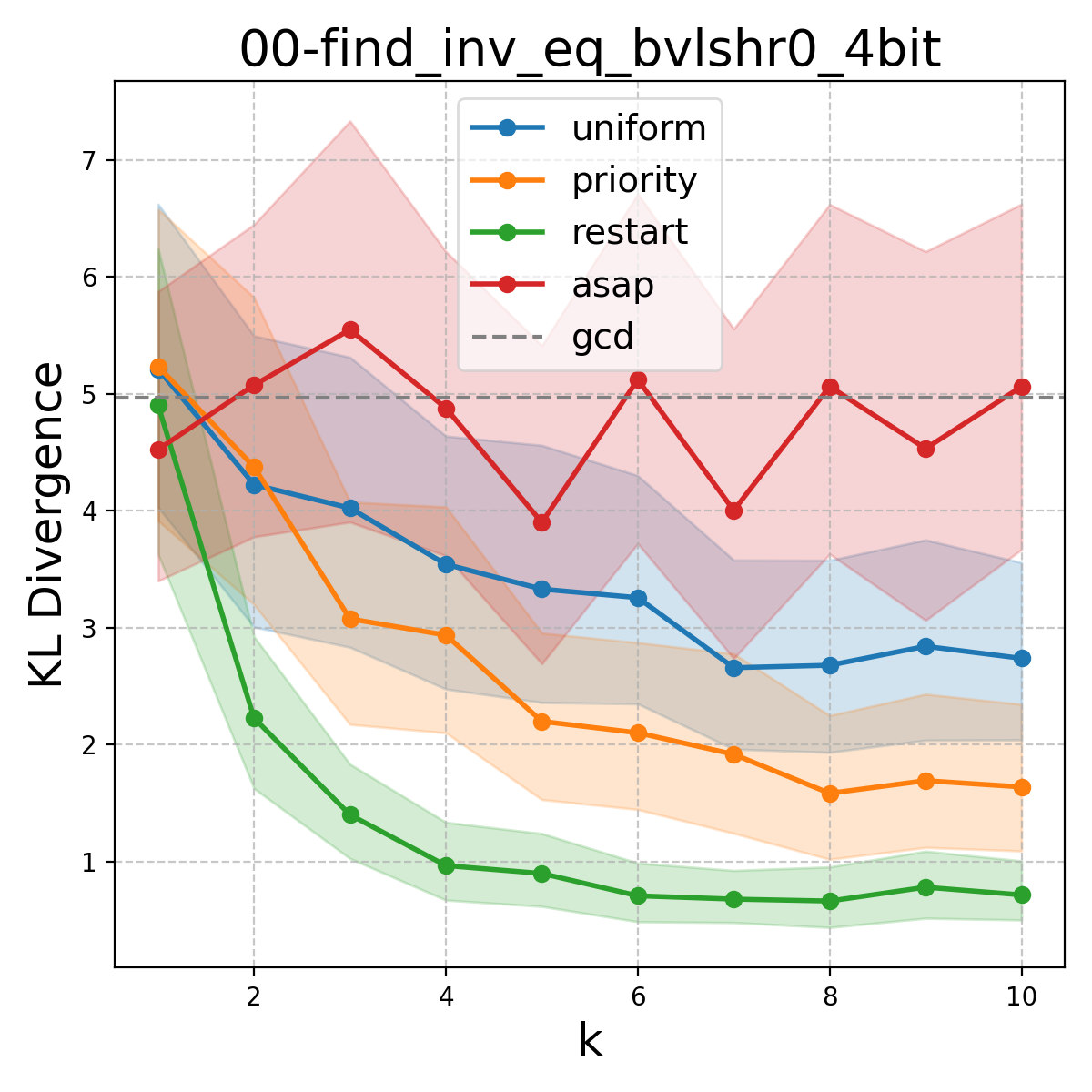}
        \vspace{-0.5em}
        \caption{}
        \label{fig:sygus-repr}
    \end{subfigure}
    \hspace{1.5em}\vrule\hspace{1em}
    \begin{subfigure}{0.45\textwidth}
        \centering
        \includegraphics[width=.9\linewidth]{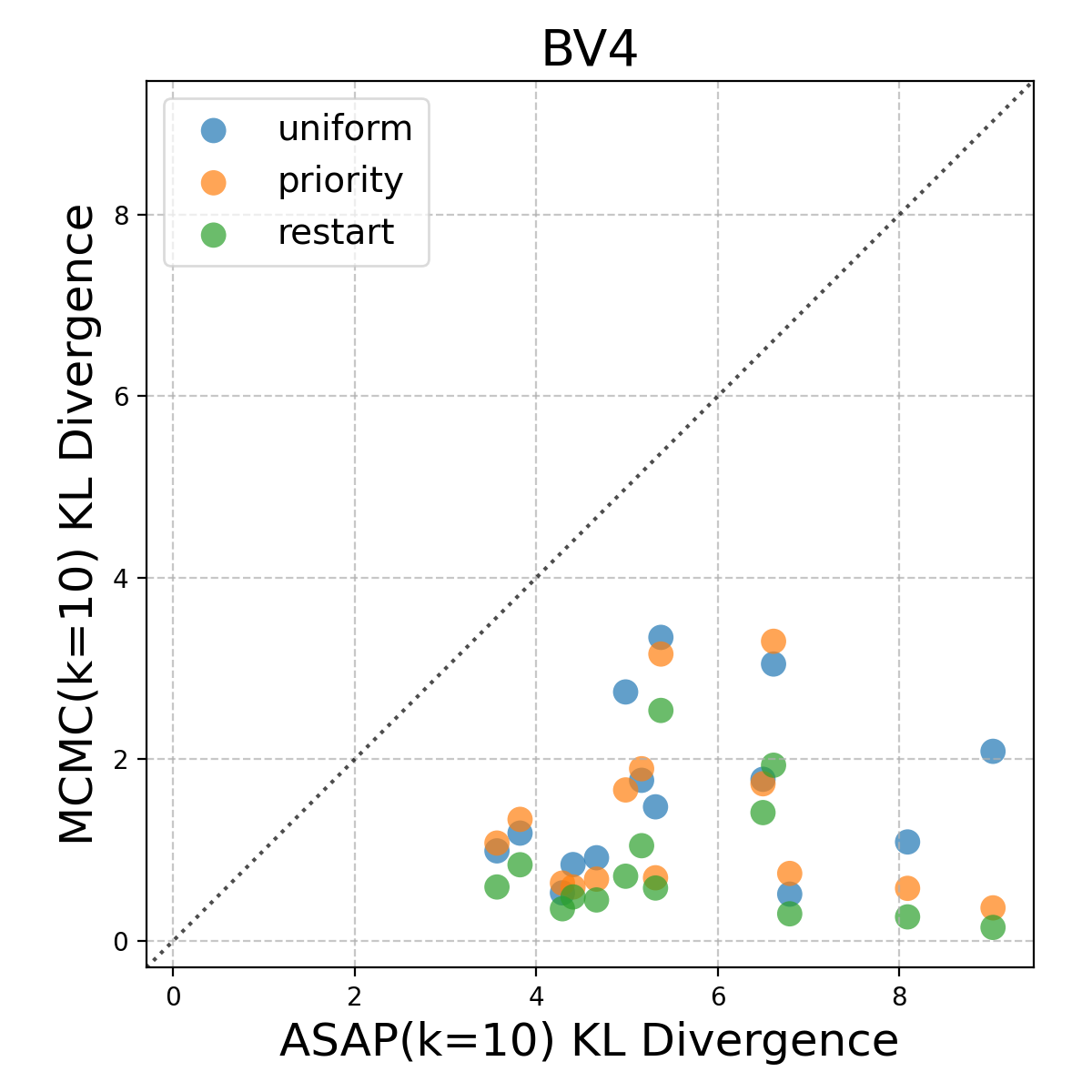}
        \vspace{-0.5em}
        \caption{}
        \label{fig:sygus-scatter}
    \end{subfigure}
       
    \caption{(a) KL divergence for varying number of steps for a representative \sygus benchmark.
    (b) Empirical KL divergence of MCMC vs. ASAp when run for 10 steps on BV4 benchmarks.
    }
    \label{fig:sygus-plots}
\end{figure}

\paragraph{Results and Findings.}
\Cref{fig:sygus-repr} illustrates how the KL divergence for our MCMC approaches monotonically decreases with the number of steps for one representative \sygus benchmark (the other tasks show similar trends). 
The KL divergence decreases and converges in trend for all variants of MCMC, though with some fluctuation caused by randomness.
ASAp only converges in the limit and does not exhibit monotonic convergence.


\Cref{fig:sygus-scatter} illustrates how the KL divergence after running MCMC (all variants) and ASAp for 10 steps.
A point $(i,j)$ denotes a benchmark on which ASAp yields KL divergence $i$ and MCMC yields KL divergence $j$ (i.e., a point below the diagonal denotes that MCMC has converged better).
%
All variants of MCMC display large reductions in KL divergence with respect to GCD (geomean. 2.11$\times$ for Uniform, 2.42$\times$ for Priority, and 5.07$\times$ for Restart),
and even larger reductions when compared to ASAP for the same number of steps (geomean. 3.99$\times$ for Uniform, 5.08$\times$ for Priority, and 8.70$\times$ for Restart).

Restart converges faster than Uniform and Priority, a notable result, since conceptually Restart does not accumulate any information about previous states in the Markov Chain random walk.

\subsection{Fuzzing experiments}
\label{sec:evaluation:fuzzing}

Coverage-guided fuzzing~\cite{cbfmarkov} iteratively mutates seed inputs to randomly generate test cases that exercise as many execution paths as possible in a target binary (a practice that can often reveal bugs in the code under test). The quality of the initial seed corpus---particularly their structural validity and diversity---can significantly impact downstream coverage, i.e., how many execution paths the fuzzer can exercise~\cite{2021seedselectionfuzzing}. 
%
%
In this section, we evaluate whether different techniques for grammar-constrained sampling can be used to produce high-quality seeds for a state-of-the-art program fuzzer, AFL++~\cite{aflplusplus}.
In our experiment, the grammar is used to guarantee that the LM produces inputs for the library under test that are \textit{valid} (they will not be immediately rejected by the program) and can trigger execution of specific components of a library (e.g, forcing the SQLite timeout directive). 

\paragraph{Benchmarks.}
We evaluate our approach on two widely adopted, grammar-intensive targets: XML (using the libxml2 parser~\cite{1998libxml2}) and SQL (using the SQLite engine~\cite{2000sqlite}).
To reflect realistic use cases where a general user prompt remains fixed while grammar constraints evolve to target different code components, we employ one high-level prompt per library (e.g., the one in \Cref{fig:fuzzing-prompt}) and introduce domain-specific constraints that can trigger different code components directly into the grammar.
Our grammars are specializations of publicly available EBNF grammars for XML and SQL~\cite{W3CXML1,W3CXMLNames3,zxteloiv_complexqa_2025}.
For XML, we modify the grammar to enforce mandatory \texttt{<!ENTITY>} and \texttt{<!ELEMENT>} declarations within the DOCTYPE, while for SQL we mandate a \texttt{set ::timeout 60000} directive within each \texttt{.test} file.
A snippet of the grammar we use to test the SQL engine is given in \Cref{fig:fuzzing-grammar} and the full list of targets, versions, and seed formats is given in Appendix~\ref{app:benchmarks}.
A user of the fuzzer can modify these grammars to stress test different components of the software under test.

\paragraph{Measures.}
Aside from the KL divergence, which we measure in the way discussed in \Cref{sec:evaluation:synth}, we want to measure whether a better sampler leads to the fuzzer having higher branch coverage---i.e., the number of unique executed code branches over the total number of branches in the software under test (computed via LLVM instrumentation~\cite{llvm-cov-docs}). 

We generate 100 seed inputs per method, using GCD, ASAp, as well as MCMC-Priority, MCMC-Restart, and MCMC-Uniform, with varying number of steps ($k\in\{2,5,10\}$).
We also include as a baseline Grammarinator~\cite{2018grammarinator}, a fuzzing tool for generating random strings in a grammar that does not use LMs. 
%
For these benchmarks, rejection sampling exhibits impractical acceptance rates (<1\%) (\Cref{app:rej}) and we therefore exclude it from our comparison as it would take more than 10,000 samples to produce 100 valid seeds. 

We describe the full fuzzing protocol in \Cref{app:fuzzing-protocol}. In summary, for every benchmark and method we generate 100 seeds for AFL++ and run it for 1 hour (in this time AFL++ generates thousands on new inputs based on the seeds and executes the software on them).
We then measure mean branch coverage at different time steps with bootstrapped 95\% confidence intervals over five independent 1-hour-long fuzzing trials, following standard fuzzing-evaluation protocols~\cite{bohme2022reliability,klees2018evaluating}. 
Most reachable branches are typically discovered in the first hour, after which coverage tends to plateau~\cite{klees2018evaluating}).


\paragraph{Findings.}
Due to space limitations, we focus on representative results for the SQL benchmark (full results in \Cref{app:overallcov}).
\Cref{fig:kl-div-sql} illustrates how the KL divergence changes at different steps in a similar way as observed in \Cref{sec:evaluation:synth} for synthetic benchmarks---it decreases and converges in trend for all variants of MCMC, though with some fluctuation caused by randomness.

%
\begin{figure}[t]
  \centering
  \subcaptionbox{KL divergence\label{fig:kl-div-sql}}
                [0.48\linewidth]{%
     \includegraphics[width=\linewidth]{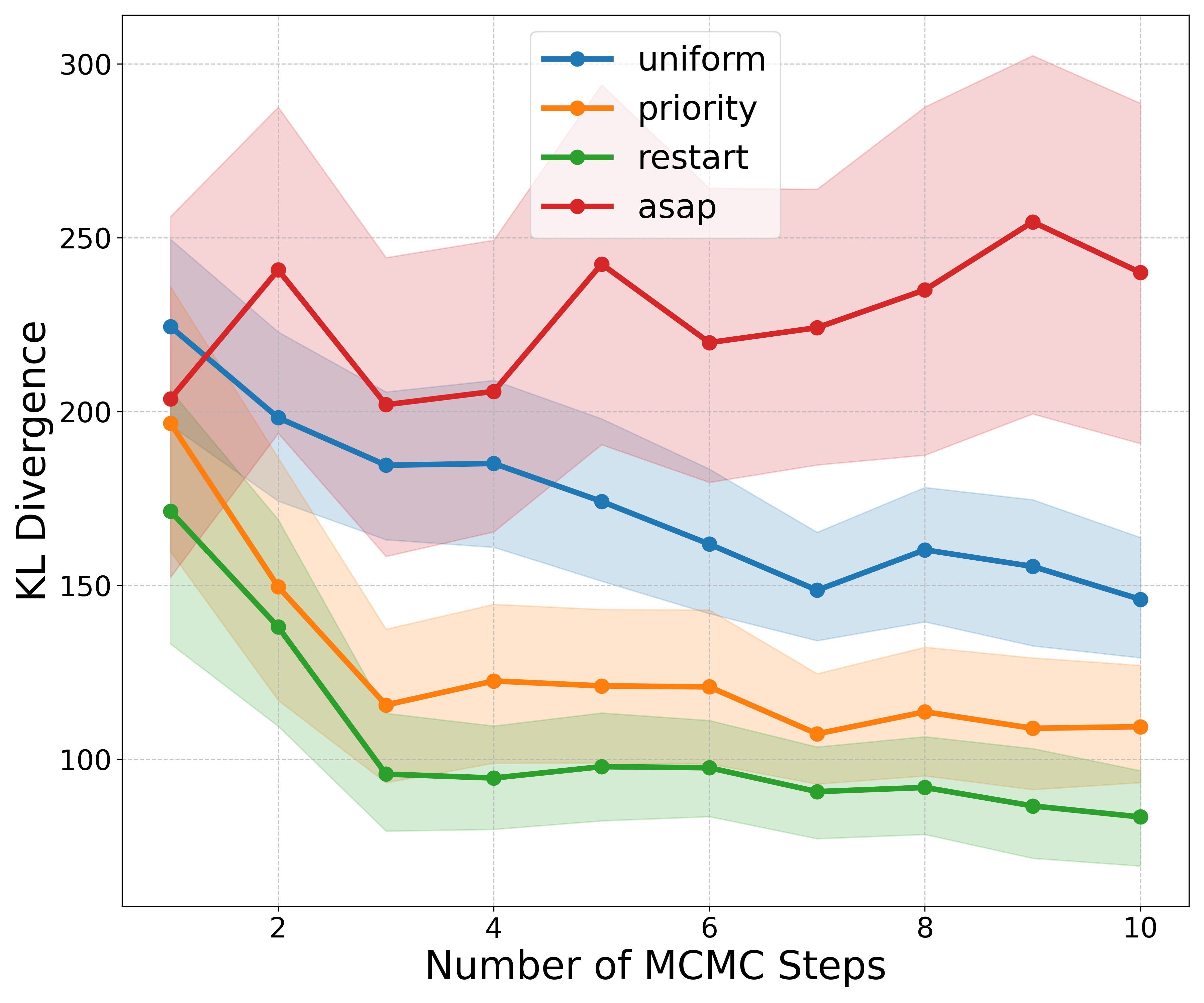}}
  \hfill
  \subcaptionbox{Branch coverage\label{fig:branch-sql}}
                [0.48\linewidth]{%
     \includegraphics[width=\linewidth]{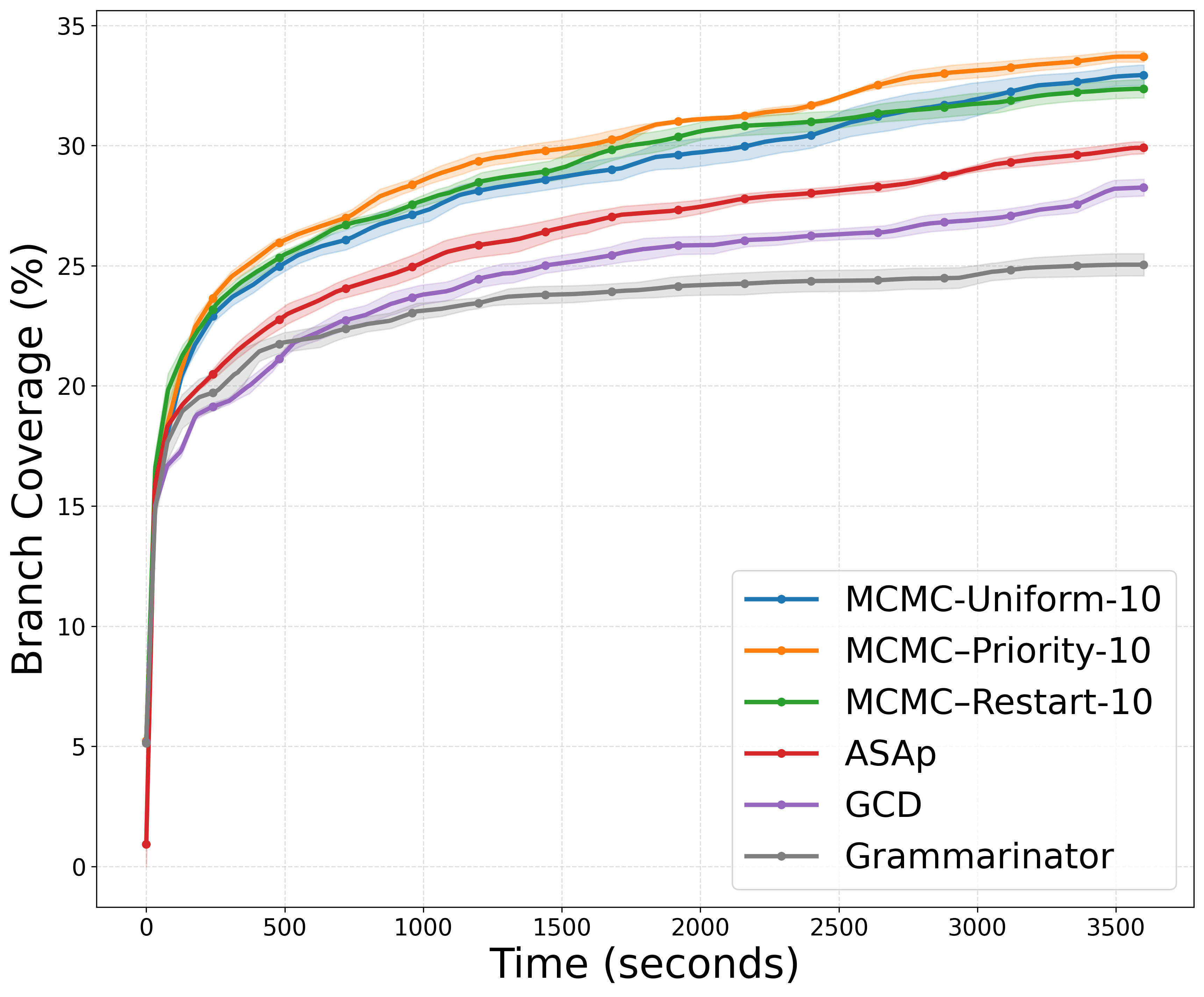}}
  \caption{SQL benchmark: (a) KL divergence 
  for MCMC with varying number of steps. (b) Branch coverage over time.}
  \label{fig:sql-main-results}
\end{figure}

The key result is given in \Cref{fig:branch-sql}. To avoid clutter, \Cref{fig:branch-sql} only reports the results for MCMC when ran for 10 steps (the versions for 2 and 5 can be found in ~\Cref{app:ablation}).

Fuzzing using seeds produced via MCMC leads to significantly higher branch coverage than Grammarinator, GCD, and ASAp. Grammarinator, relying solely on grammar, achieves the lowest coverage, followed by GCD and then ASAP.
Among all approaches, MCMC-Priority $(k=10)$ is the most effective, delivering branch coverages of 33.70\% SQL and 12.81\% XML coverage---corresponding to gains of $1.2\times$ and $1.15\times$ over GCD (28.26\%, 10.67\%) and $1.12\times$ and $1.15\times$ over ASAp (29.91\%, 11.1\%) on the SQL and XML benchmarks, respectively.
This result underscores the benefit of MCMC in diversifying the kind of samples an LM can produce.
As expected the versions of MCMC with $(k=5)$ provide coverages that are lower than at $(k=10)$, but higher than GCD (see \Cref{app:ablation}); at $(k=2)$ the coverage is quite similar to that of GCD.

Intriguingly, the MCMC methods' coverage ranking (\Cref{fig:branch-sql}) is inversely mirrored by their KL divergence from the true conditional grammar-constrained distribution (\Cref{fig:kl-div-sql}), though the overall differences are relatively minor. 
This phenomenon suggests that while LMs provide a strong foundation for producing diverse outputs, approaches like MCMC-Priority may be better at exploring ``variants'' of the same output, which in turn are better for exercising branch coverage in fuzzing. 

In summary, the results support the claim that MCMC-based constrained sampling monotonically converges to the desired distribution, and convergence is indicative of higher coverage in fuzzing.

%

\section{Related Work}\label{sec:related}

\paragraph{Constrained Decoding}
A large body of work has investigated constrained decoding methods that modify the token-by-token decoding process of LMs to enforce syntactic or lexical restrictions. These constraints are often specified using regular languages~\cite{melcer2024constrained, willard2023efficient} or context-free grammars (CFGs)~\cite{lmql, dong2022codepad, geng2024grammarconstrained, poesia2022synchromesh, picard, shin2021constrained, stengel2023zero, ugare2024improving,park2025flexibleefficientgrammarconstraineddecoding,AgrawalKGLR23, li2024guiding, wang2023grammar,anderson2016guided, hokamp2017lexically, hu2019improved, lu-etal-2022-neurologic, lu-etal-2021-neurologic, post2018fast}.
As discussed in \Cref{sec:limitations-of-existing-approaches} and as observed by~\citet{park2024grammaraligned}, approaches based on constrained decoding distort the LM distribution and do not converge to the target constrained distribution.

Gradient-based sampling methods~\cite{amini2024structured, kumar2022gradient, li2022diffusion, qin2022cold} offer a softer alternative, guiding generation toward constraint satisfaction by using relaxed, differentiable surrogates. These methods are better suited for soft or semantic constraints but still suffer from inefficiency and are guaranteed to produce constraint-satisfying outputs.

\paragraph{Constraint-Aligned Decoding}
Several recent works have emphasized the importance of sampling multiple outputs from LMs to approximate the model’s distribution~\cite{huang2024large, llm_sampling_renda_hopkins_2023,park2024grammaraligned}. However, when hard constraints are introduced, typical decoding strategies like top-$k$, nucleus, or beam search become unreliable estimators of the constrained distribution. 

\citet{park2024grammaraligned} introduced  the ASAp to approximate the true constrained distribution, which it has been studied further by~\citet{melcer2024approxaligned}.
However, these methods are slow to converge, requiring thousands of samples, and therefore inefficient in practice.


\paragraph{Controlled generation}

Methods such as GeLaTo \cite{zhang2023tractable} and Ctrl-G \cite{zhang2024adaptablelogicalcontrol} combine autoregressive language models with Hidden Markov Models (HMMs) to guide generation based on constraints. These techniques are specifically designed for constraints that can be represented as deterministic finite automata (DFA). 
Given a prefix, a pretrained HMM is used to approximate the probability that a subsequently generated suffix will satisfy the DFA constraint, effectively estimating the likelihood of a valid continuation. However, this approach is limited to DFA-representable constraints and cannot be easily extended to more general grammars like context-free grammar. Furthermore, these methods require training separate surrogate models. Crucially, they also do not guarantee convergence to the ideal distribution, a limitation they share with other techniques that use approximate inference for intractable conditional distributions (e.g., Feynman-Kac Transformer Models) \cite{qin2022cold, lew2023sequential}.

\section{Conclusion}
\label{sec:conclusion}

We introduced a simple yet effective MCMC-based framework for constrained decoding. 
Unlike prior approaches that suffer from slow convergence or rely on inefficient rejection sampling, our method directly samples from within the constrained space while asymptotically preserving the target distribution defined by the LM. 
Our framework leads to practical improvements in real-world applications where sampling diversing inputs is crucial---most notably, program fuzzing, where high-quality diverse samples translate into higher code coverage. 

\paragraph{Limitations.}
Due to compute limitations, we only provide results for one model. Furthermore, we only fuzz two projects (libxml2 and sqllite) using a single prompt choice and for 1-hour trials.
However, we believe the results support the claims made in the paper: MCMC-based constrained sampling converges to the desired distribution, and convergence translates to higher coverage in fuzzing.
While we only study three possible proposal distributions, our work opens the door for studying richer proposal mechanisms for MCMC, as well as integration with other decoding strategies such as beam or nucleus sampling.

\section*{Acknowledgments}
This work was supported in part by a Microsoft Faculty Fellowship; a UCSD JSOE Scholarship; Google’s Gemma Academic Program
GCP Credit Award; and NSF under grants CCF-2422214, CCF-2402833 and CCF-2211968. 
Any opinions, findings, and conclusions or recommendations expressed in this publication are those of the authors, and do not necessarily reflect the views of the sponsoring entities.
Loris D’Antoni holds concurrent appointments as a Professor at the University of California San Diego and as an Amazon Scholar. This paper describes work performed at the University of California San Diego and is not associated with Amazon.
The authors deeply thank Nadia Polikarpova for her insightful feedback and guidance during the development of this project.

\bibliographystyle{acl_natbib}
\bibliography{ref}


\newpage
\appendix

\begin{center}
{\LARGE\bfseries Appendix}
\end{center}

\section{Hardware and Software}
\label{app:hardware} 
Our experiments were conducted on Ubuntu 22.04 LTS nodes with Intel Xeon Gold 6230 CPUs (2.10 GHz, 10 cores, 20 threads allocated) and 384 GB RAM. For GPU-accelerated workloards, we provisioned 8x NVIDIA GeForce RTX 2080 Ti GPUs. Our implementation is based on Python 3.10.12, PyTorch 2.6.0+cu124, AFL++ 4.00c and LLVM 14.0.0.

\section{Hyperparameters}
\label{app:hyperparams}
For language‐model decoding, we set temperature to 1.0, top-p to 0.9, and top-k to 0 to allow sampling from the full token vocabulary. We limited the maximum number of newly generated tokens to 512 for XML, and 1024 for SQL \texttt{.test} scripts. 

\section{Model Checkpoint}
\label{app:models}
We use the Llama-3.1-8B model checkpoint provided by Hugging Face (commit 0e9e39f):
\url{https://huggingface.co/meta-llama/Llama-3.1-8B-Instruct}.

\section{Fuzzing Experiments Details}
\label{app:exp_detail}








\subsection{Benchmarks}
\label{app:benchmarks}

Table~\ref{tab:benchmarks:app} summarizes the libraries, versions, and seed formats for each target.

\begin{table}[h]
  \centering
  \small
  \caption{Benchmarks, versions, and seed formats.}
  \label{tab:benchmarks:app}  
  \begin{tabular}{l l l l}
    \toprule
    \textbf{Target} & \textbf{Library}        & \textbf{Version} & \textbf{Seed format} \\
    \midrule
    XML\textsuperscript{\cite{W3CXML1,W3CXMLNames3}}   & libxml2   & 2.15.0 & \texttt{.xml}  \\
    SQL\textsuperscript{\cite{zxteloiv_complexqa_2025}}               & sqlite    & 3.49.2 & \texttt{.test} \\
    \bottomrule
  \end{tabular}
\end{table}

\subsection{Parse-Tree Illustration}
\label{app:parse-tree}
Figure~\ref{fig:parse-tree} shows a comprehensive parse tree for a SQLite test case, derived from the grammar in Figure~\ref{fig:fuzzing-grammar}.

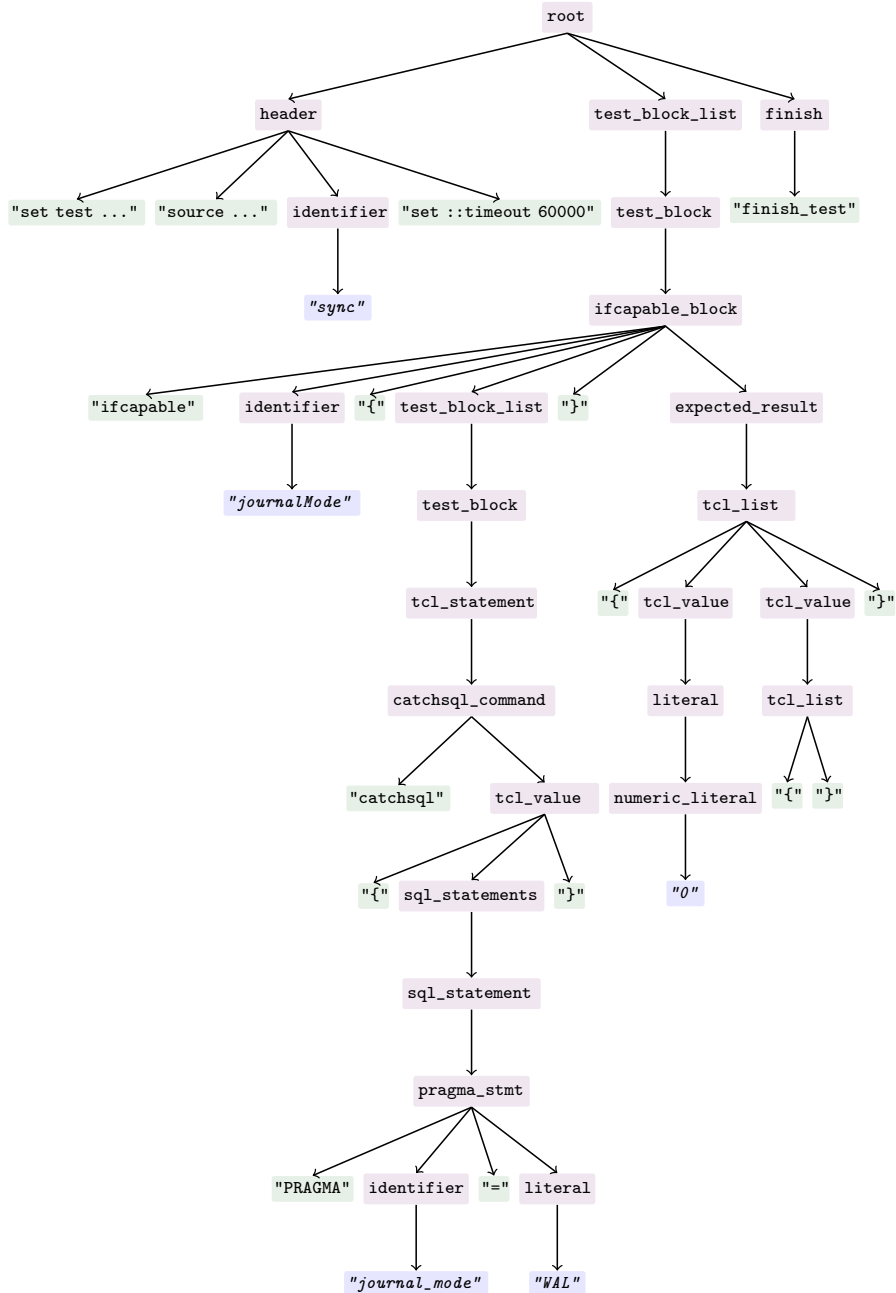
\begin{figure}[t]
  \centering
  \resizebox{0.85\linewidth}{!}{\begin{forest}
  for tree={
    grow=south,
    s sep=1.2mm,
    l sep=9mm,         
    edge={->, semithick},
    parent anchor=south,
    child anchor=north,
    align=center,
    font=\scriptsize, 
    nt/.style={font=\ttfamily\scriptsize, inner sep=2pt, fill=nonterminal!10, text width=1.6cm, align=center, rounded corners=1pt},
    term/.style={font=\ttfamily\scriptsize, rounded corners=1pt, inner sep=1pt, fill=terminal!10, text width=2.1cm, align=center}, 
    termval/.style={font=\ttfamily\scriptsize\itshape, fill=blue!10, rounded corners=1pt, inner sep=1pt, text width=1.8cm, align=center},
    operator/.style={font=\ttfamily\scriptsize, fill=operator!10, rounded corners=1pt, inner sep=1pt, text width=1.8cm, align=center},   
  },
  anchor=west,
  [root, nt, text width=0.55cm 
    [header, nt, text width=0.77cm
      ["set test ...", term, text width=1.8cm]
      ["source ...", term, text width=1.6cm]
      [identifier, nt, text width=1.25cm
          ["sync", termval, text width=0.85cm]
      ]
      ["set ::timeout 60000", term, text width=2.7cm]
    ]
    [test\_block\_list, nt, text width=1.95cm
      [test\_block, nt, text width=1.35cm
        [ifcapable\_block, nt, text width=1.95cm,
          ["ifcapable", term, text width=1.5cm, xshift=-6cm]
          [identifier, nt, text width=1.3cm, xshift=-4cm
            ["journalMode", termval, xshift=-4cm]
          ]
          ["\{", term, text width=0.35cm, xshift=-3cm] 
          [test\_block\_list, nt, text width=1.95cm, xshift=-3cm
            [test\_block, nt, text width=1.35cm, xshift=-3cm
              [tcl\_statement, nt,  text width=1.65cm, xshift=-3cm
                [catchsql\_command, nt, text width=2.15cm, xshift=-3cm
                  ["catchsql", term,  text width=1.35cm, xshift=-4cm]
                  [tcl\_value, nt,  text width=1.35cm, xshift=-2cm
                    ["\{", term, text width=0.35cm, xshift=-3cm]
                    [sql\_statements, nt,  text width=1.85cm, xshift=-2cm
                      [sql\_statement, nt,  text width=1.75cm, xshift=-2cm
                        [pragma\_stmt, nt, text width=1.45cm, xshift=-2cm
                          ["PRAGMA", term, text width=1.05cm, xshift=-3cm]
                          [identifier, nt, text width=1.3cm, xshift=-3cm
                            ["journal\_mode", termval, text width=1.9cm, xshift=-3cm]
                          ]
                          ["\text{=}", term, text width=0.35cm, xshift=-2cm]
                          [literal, nt, text width=0.9cm, xshift=-2cm
                            ["WAL", termval, text width=0.7cm, xshift=-2cm]
                          ]
                        ]
                      ]
                    ]
                    ["\}", term, text width=0.35cm, xshift=-3cm]
                  ]
                ]
              ]
            ]
          ]
          ["\}", term, text width=0.35cm, xshift=-2cm] 
          [expected\_result, nt, text width=1.95cm
            [tcl\_list, nt, text width=1.2cm
              ["\{", term, text width=0.35cm]
              [tcl\_value, nt, text width=1.15cm
                [literal, nt, text width=0.9cm
                  [numeric\_literal, nt, text width=1.95cm
                    ["0", termval, text width=0.45cm]
                  ]
                ]
              ]
              [tcl\_value, nt,  text width=1.15cm
                [tcl\_list, nt,  text width=1.10cm
                  ["\{", term, text width=0.35cm]
                  ["\}", term, text width=0.35cm]
                ]
              ]
              ["\}", term, text width=0.35cm]
            ]
          ]
        ]
      ]
    ]
    [finish, nt, text width=0.8cm
      ["finish\_test", term, text width=1.7cm]
    ]
  ]
\end{forest}}
  \caption{Condensed parse tree for the example SQLite
           \texttt{.test} script used in \Cref{fig:fuzzing-example}. \textcolor{nonterminal}{{Purple}} boxes denote non-terminals, \textcolor{terminal}{green} boxes denote grammar  terminals, and \textcolor{blue}{blue\,italics} show literal terminal values substituted during this derivation. Subtrees unrelated to the header and the first \texttt{do\_execsql\_test} block are elided for brevity.}
  \label{fig:parse-tree}
\end{figure}

\subsection{Prompts and Constraints}
\label{app:prompts}
For all benchmarks, we use a standard in-context learning format where the prompt consists of a single (specification, solution) pair, followed by a new specification for which the model must generate a solution. A representative prompt for the \texttt{XML} benchmark is shown in Figure~\ref{fig:fuzzing-xml-prompt}, along with its corresponding grammar in Figure~\ref{fig:fuzzing-grammar-xml}. 

\begin{figure}
    \centering
    \scriptsize
    \begin{subfigure}{0.48\textwidth}
        \centering
        \input{codes/xml-prompt}
        \vspace{-0.5em}
        \caption{Prompt}
        \label{fig:fuzzing-xml-prompt}
    \end{subfigure}
    \hspace{1.5em}\vrule\hspace{1em}
    \begin{subfigure}{0.44\textwidth}
        \centering
        \input{codes/xml-grammar}
        \vspace{-0.5em}
        \caption{Grammar}
        \label{fig:fuzzing-grammar-xml}
    \end{subfigure}
       
    \caption{(a) Prompt given to a LM to generate seed test cases for fuzzing the XML parser.
    (b) Simplified version of the XML grammar written in EBNF notation. The goal of the problem is to generate multiple diverse seeds that trigger different code paths in the library being tested.
} 
    \label{fig:fuzzing-example-xml}
\end{figure}

\subsection{Fuzzing Protocol and Environment}
\label{app:fuzzing-protocol}

All fuzzing experiments were conducted using AFL++ 4.00c on the hardware and software setup described in~\Cref{app:hardware}. Each \((\text{benchmark},\text{method})\) pair was evaluated in \(\mathbf{N = 5}\) independent, single-instance AFL++ runs of exactly \(3600\,\text{s}\) (one hour).  We set `AFL\_RANDOM\_SEED` to \(42+i, (i=1...5)\) for reproducibility and configure standard environment variables to ensure non-interactive execution. All other AFL++ parameters remained at defaults to isolate the impact of seed corpus quality. Complete build and execution scripts are provided in the supplementary materials.

\subsection{Coverage Measurement via LLVM Instrumentation}
\label{app:coverage}

We measured branch coverage using LLVM's instrumentation toolchain 
\texttt{(-fprofile-instr-generate -fcoverage-mapping)}, which adds $\leq$ 2\% runtime overhead. Raw profiles were collected during execution and aggregated post-trial using llvm-profdata and llvm-cov.

\paragraph{Rationale.}
We report branch coverage rather than crash counts because  
the experiment isolates \textit{seed quality} --- all methods receive the same fixed prompt per benchmark, so coverage is a good measure on how their seeds exercise the code.

\subsection{Rejection Sampling Acceptance Rates}
\label{app:rej}

We quantified the viability of rejection sampling under the same
grammar constraints used by our MCMC framework in~\Cref{sec:evaluation:fuzzing}. Across 500 attempted samples per benchmark, the proportion of syntactically and semantically valid outputs was consistently below $1\%$.

\subsection{Branch Coverage and KL Divergence for XML}
\label{app:branchcov}

Figure~\ref{fig:xml-main-results} shows the same analysis illustrated in the main text for SQLite, but for the XML benchmark.

\begin{figure}[t]
  \centering
  \subcaptionbox{\label{fig:xml-div-sql}}
                [0.48\linewidth]{%
     \includegraphics[width=\linewidth]{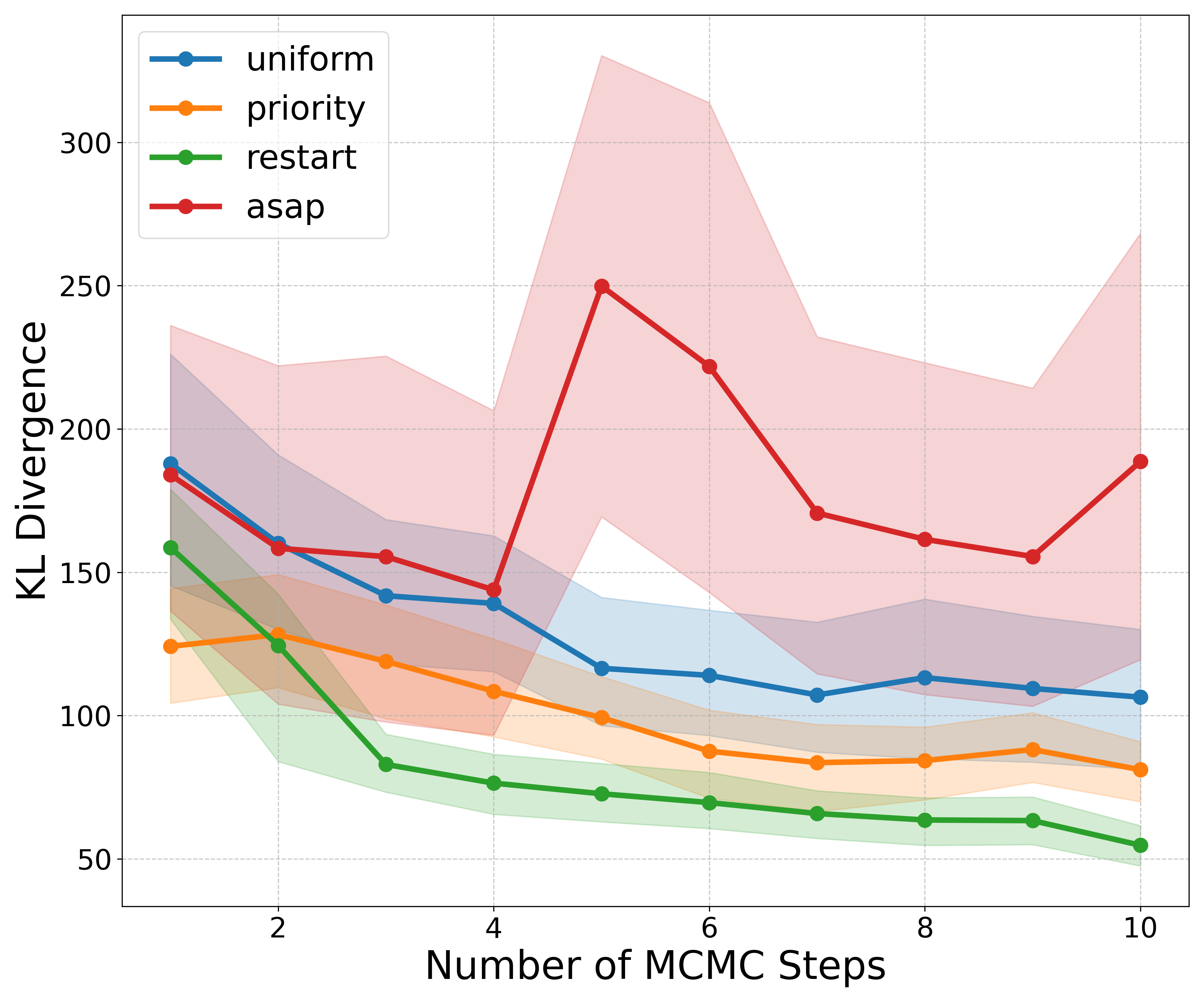}}
  \hfill
  \subcaptionbox{\label{fig:branch-xml}}
                [0.48\linewidth]{%
     \includegraphics[width=\linewidth]{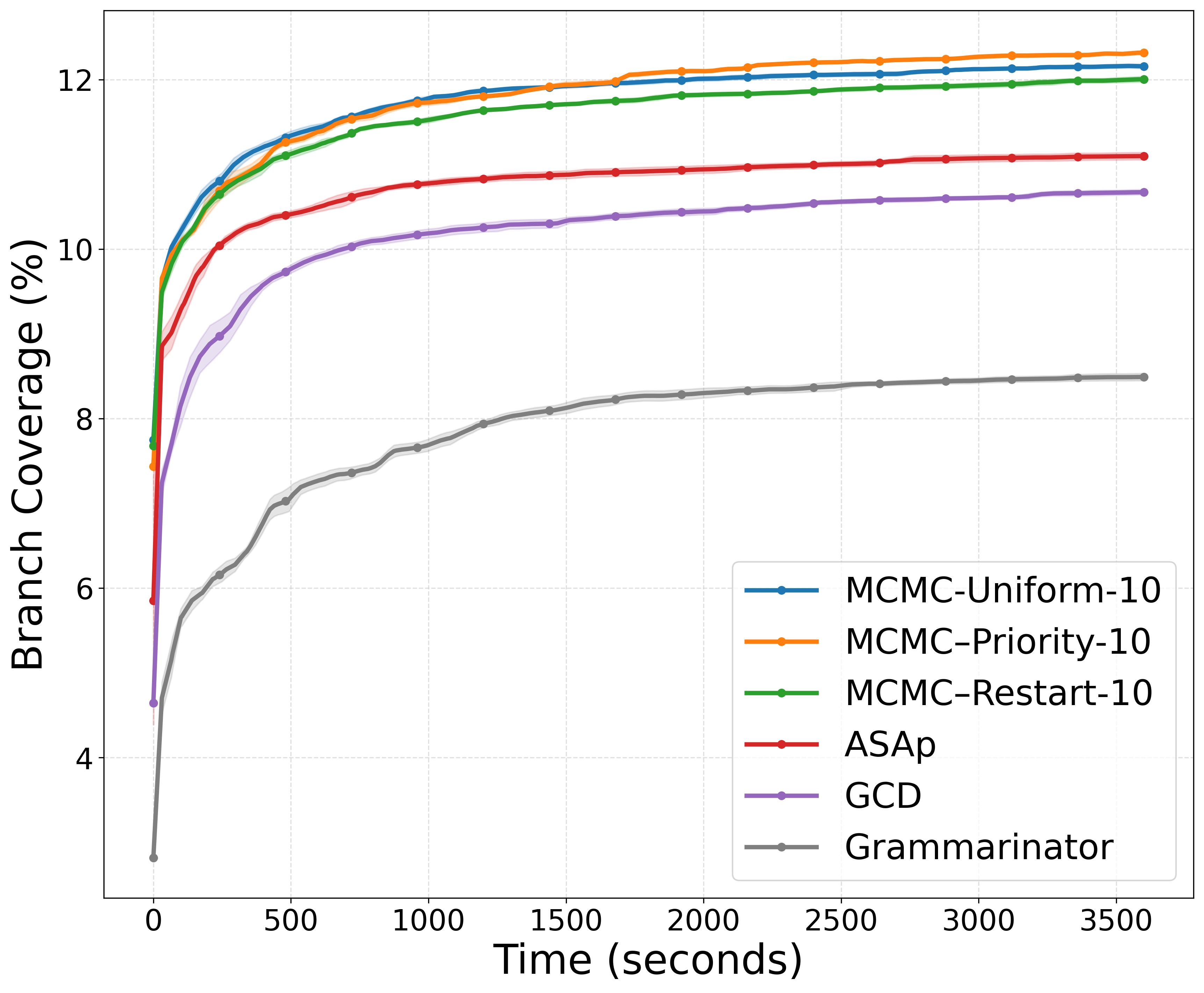}}
  \vspace{-0.6em}
  \caption{(a) KL divergence (mean ± 95 \% CI,
           hundred runs per approach) for MCMC with varying number of steps for the XML Benchmark. (b) Branch coverage over time (mean ± 95 \% CI,
           five trials per method) for the XML Benchmark.}
  \label{fig:xml-main-results}
\end{figure}

\subsection{Ablation Study}
\label{app:ablation}

To isolate the effect of (i) the proposal family ({Priority}, {Restart}, {Uniform}) and {(ii)} the number of steps \(k\!\in\!\{2,5,10\}\), we plot each family separately against the heuristic baseline (\textbf{GCD}).\footnote{Grammarinator is omitted for clarity; its curve lies far below all others and does not alter the ordering.}
Figures~\ref{fig:abl-xml}–\ref{fig:abl-sql} reveal two consistent
patterns.

\begin{enumerate}
  \item \textbf{Number of steps matters, but saturates.} Coverage grows monotonically with~\(k\); however, \(k{=}5\) already captures $\ge$95\,\% of the gain realised by \(k{=}10\) on both benchmarks.
  \item \textbf{All MCMC variants beat GCD.} Even the weakest setting (\(k{=}2\)) surpasses GCD’s final coverage by 4-5\%, demonstrating that MCMC proposals yield coverage gains over heuristic constrained decoding --- even with very few sampling steps.
\end{enumerate}

\begin{figure*}[t]
  \centering

  \subcaptionbox{Priority\label{fig:xml-prio}}
                [0.31\linewidth]{\includegraphics[width=\linewidth]{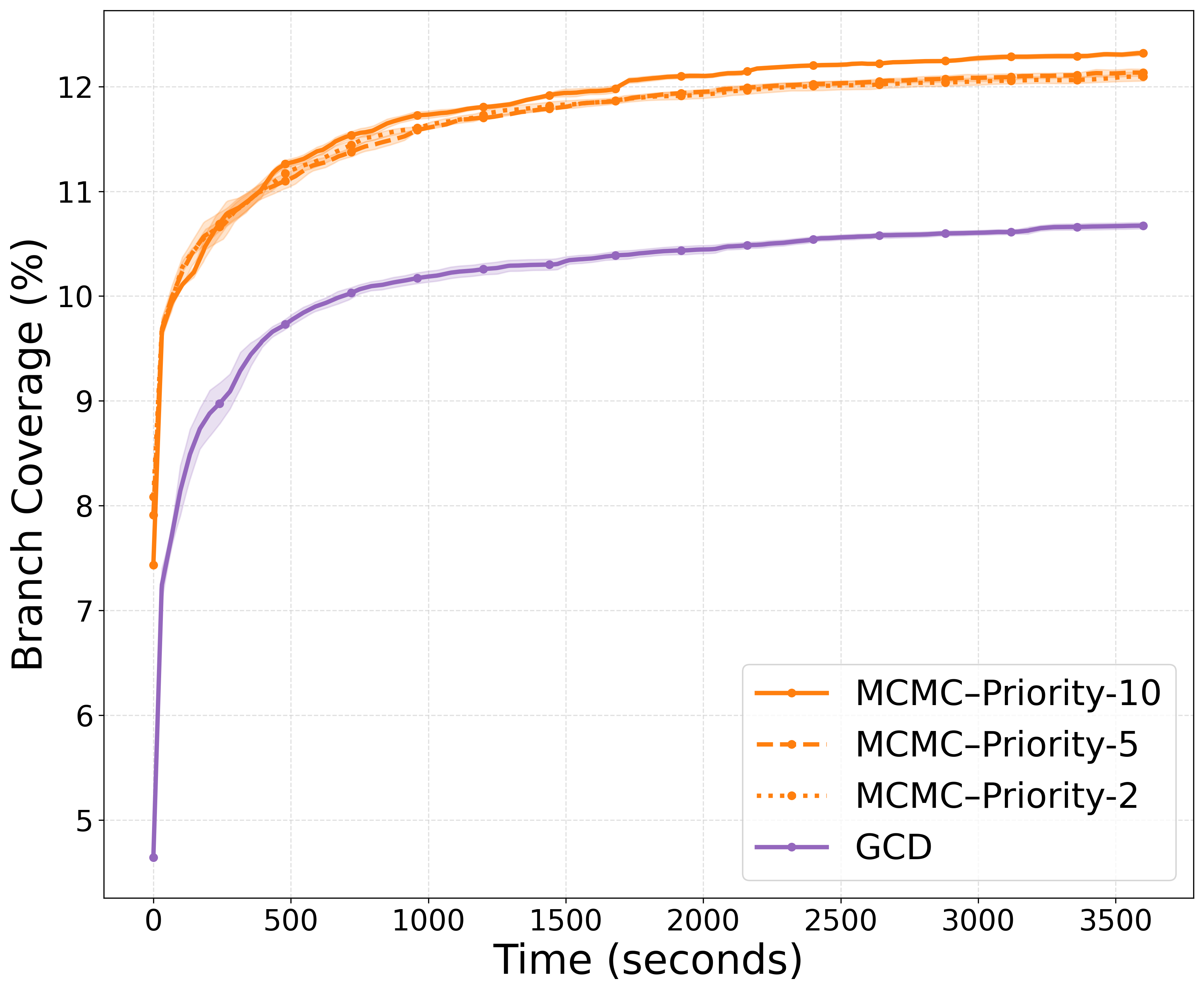}}
  \hfill
  \subcaptionbox{Restart\label{fig:xml-rest}}
                [0.31\linewidth]{\includegraphics[width=\linewidth]{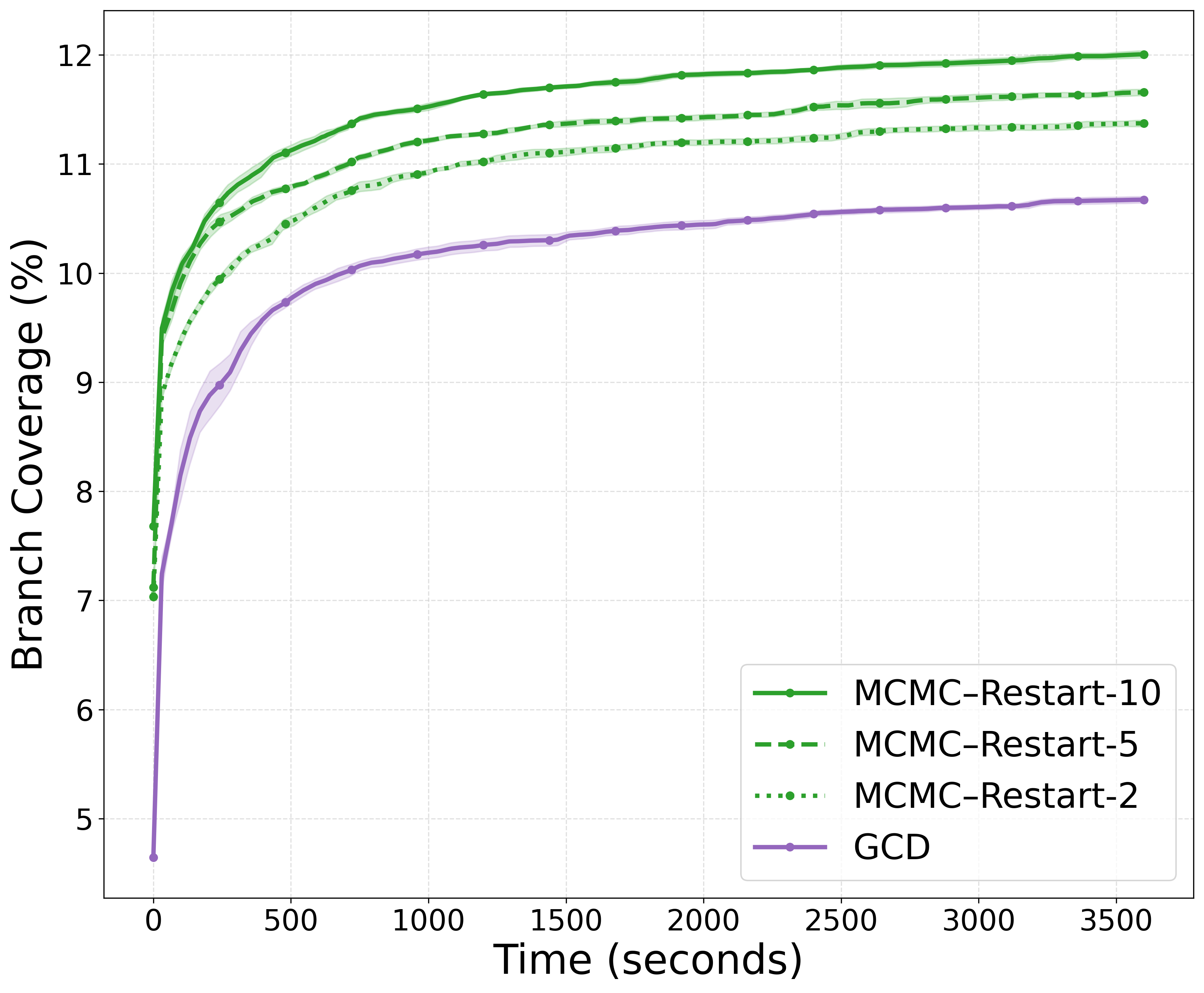}}
  \hfill
  \subcaptionbox{Uniform\label{fig:xml-pref}}
                [0.31\linewidth]{\includegraphics[width=\linewidth]{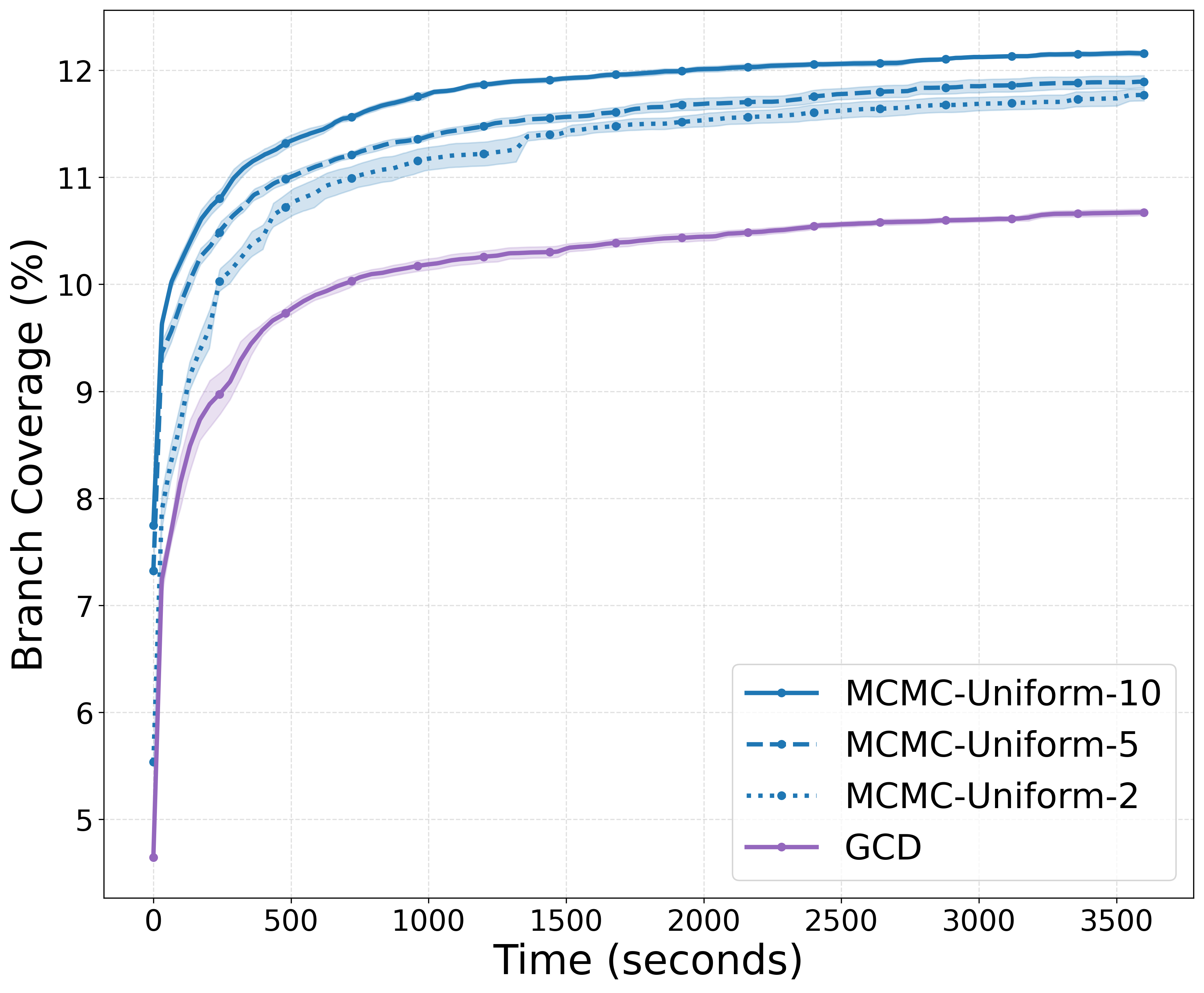}}
  \hfill
  \vspace{-0.4em}
  \caption{\textbf{XML}: branch-coverage ablation. Line style encodes the number of steps \(k\!\in\!\{2,5,10\}\) (dotted, dashed, solid).}
  \label{fig:abl-xml}
\end{figure*}

\begin{figure*}[t]
  \centering
  \subcaptionbox{Priority\label{fig:sql-prio}}
                [0.31\linewidth]{\includegraphics[width=\linewidth]{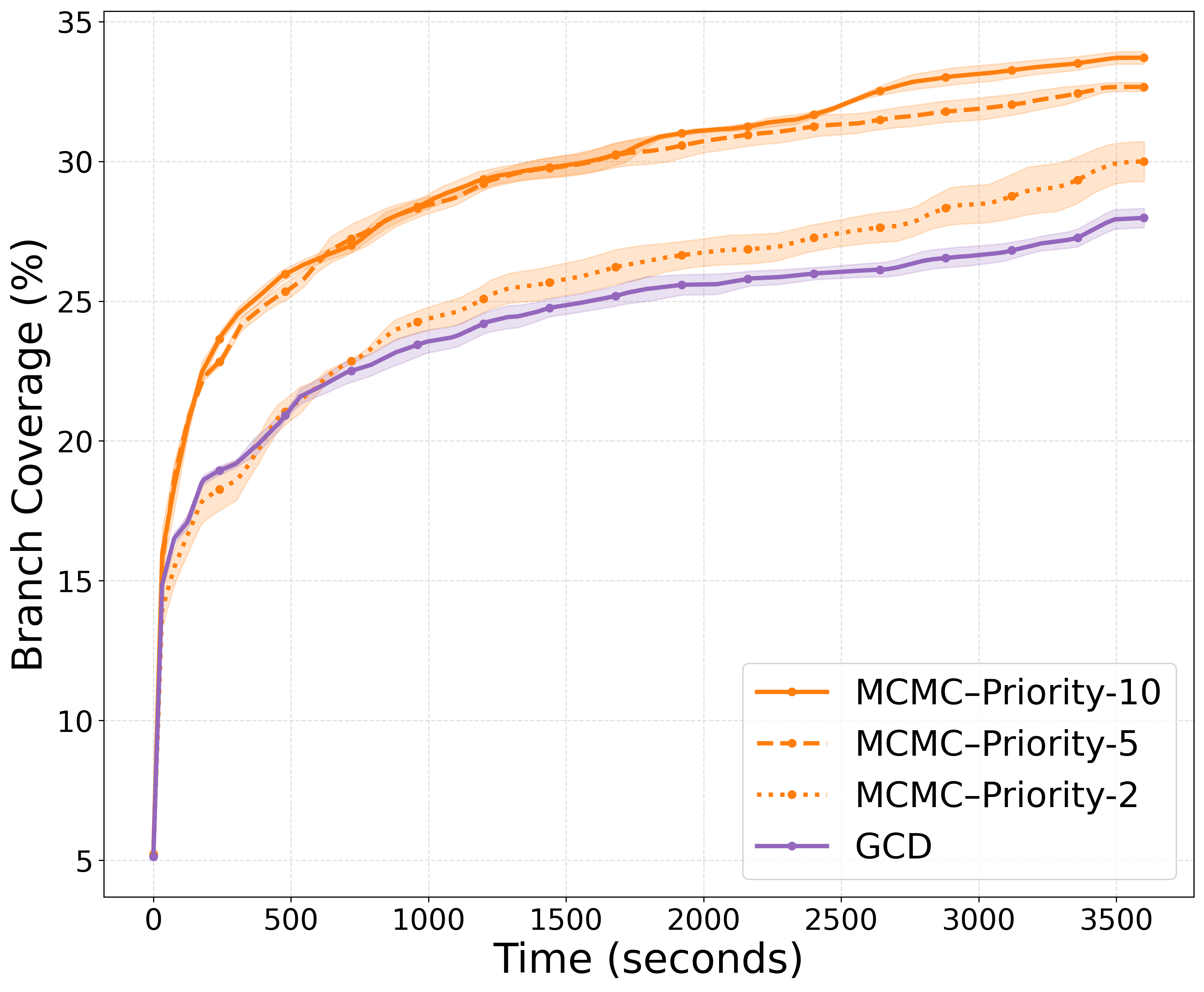}}
  \hfill
  \subcaptionbox{Restart\label{fig:sql-rest}}
                [0.31\linewidth]{\includegraphics[width=\linewidth]{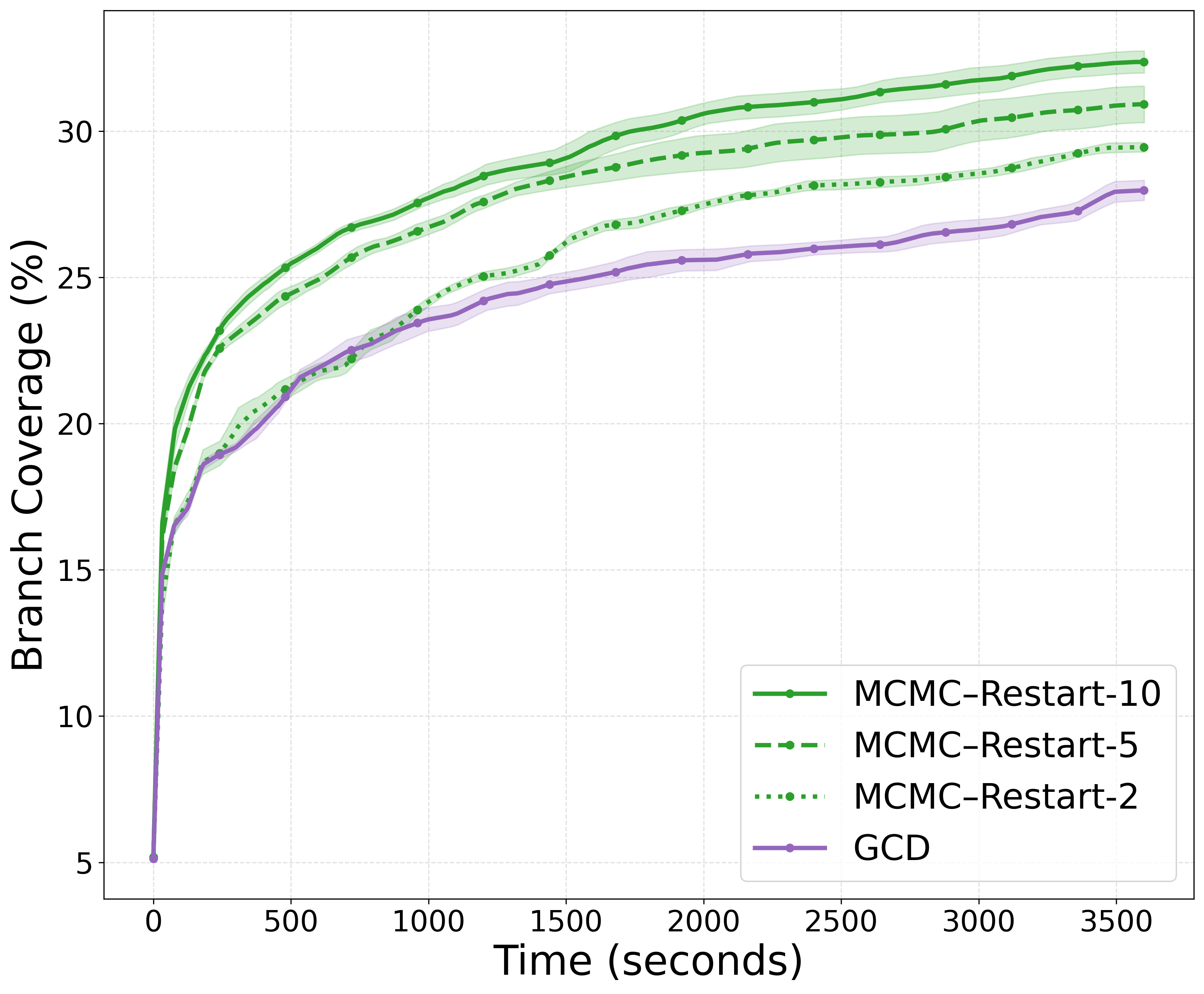}}
  \hfill
  \subcaptionbox{Uniform\label{fig:sql-pref}}
                [0.31\linewidth]{\includegraphics[width=\linewidth]{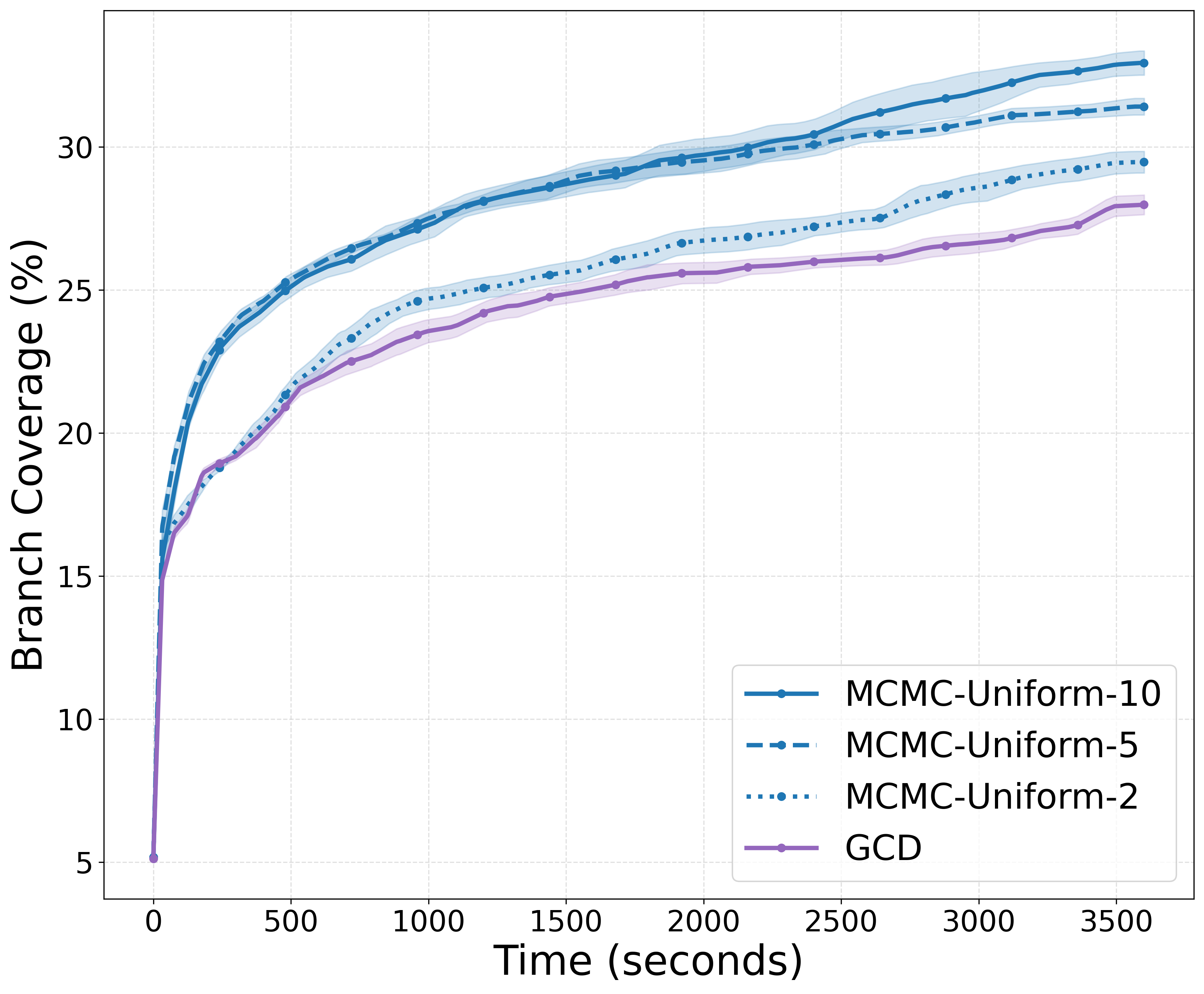}}
  \hfill
  \vspace{-0.4em}
  \caption{\textbf{SQL}: branch-coverage ablation.  Line style encodes the number of steps \(k\!\in\!\{2,5,10\}\) (dotted, dashed, solid).}
  \label{fig:abl-sql}
\end{figure*}

\subsection{Function and Line Coverage}
\label{app:addmetrics}

To corroborate the branch‐coverage trends reported in the main text, we additionally measure both

\begin{itemize}
  \item \textbf{Line coverage}: the fraction of source lines executed (\Cref{fig:line-cov-sql,fig:line-cov-xml}), and  
  \item \textbf{Function coverage}: the fraction of instrumented functions executed (\Cref{fig:function-cov-sql,fig:function-cov-xml}). 
\end{itemize}

using the same \texttt{llvm-cov} instrumentation described in Appendix~\ref{app:coverage}.

\begin{figure*}[t]
  \centering
  \subcaptionbox{Priority\label{fig:xml-prio-funccov}}
                [0.31\linewidth]{\includegraphics[width=\linewidth]{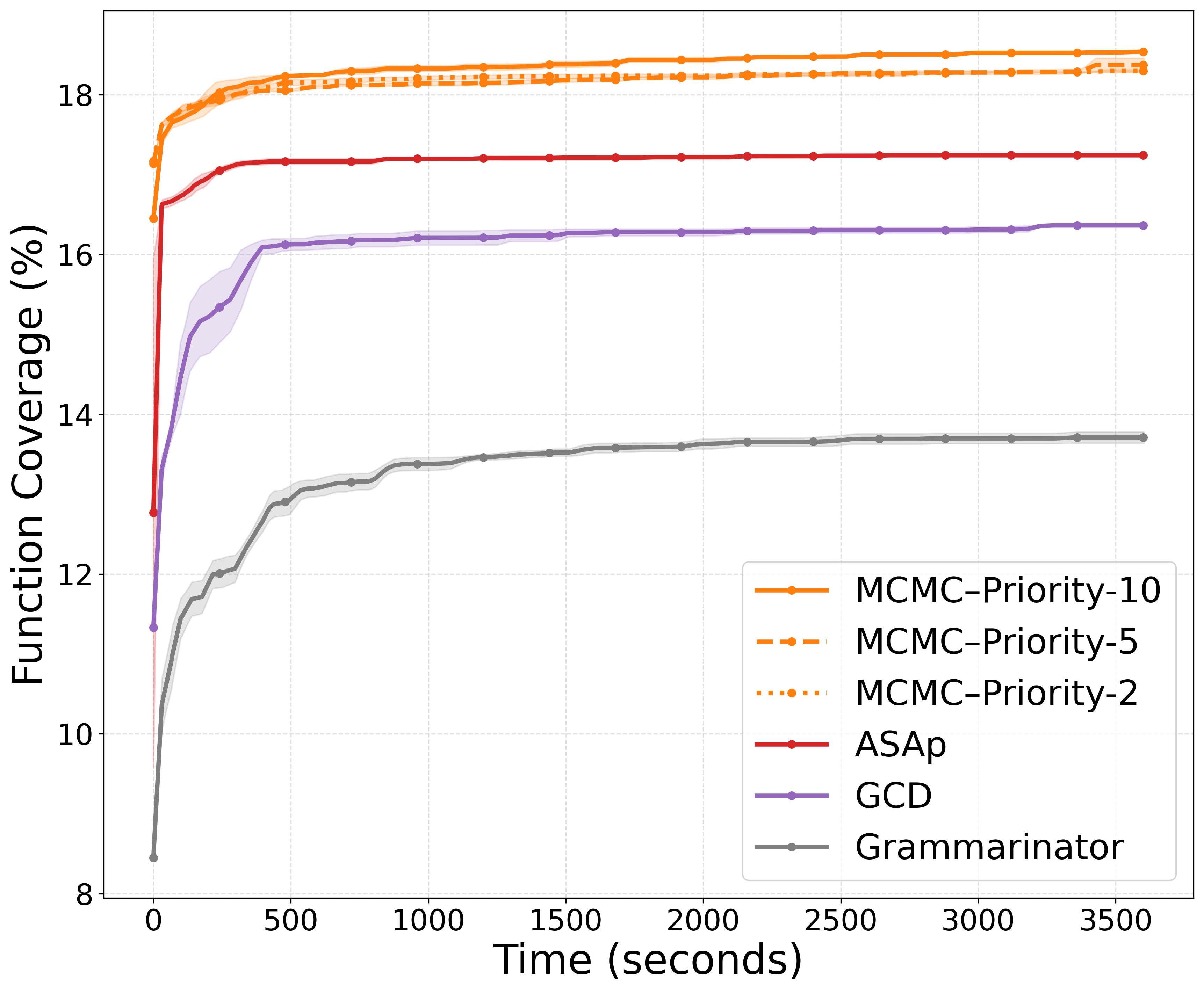}}
  \hfill
  \subcaptionbox{Restart\label{fig:xml-rest-funccov}}
                [0.31\linewidth]{\includegraphics[width=\linewidth]{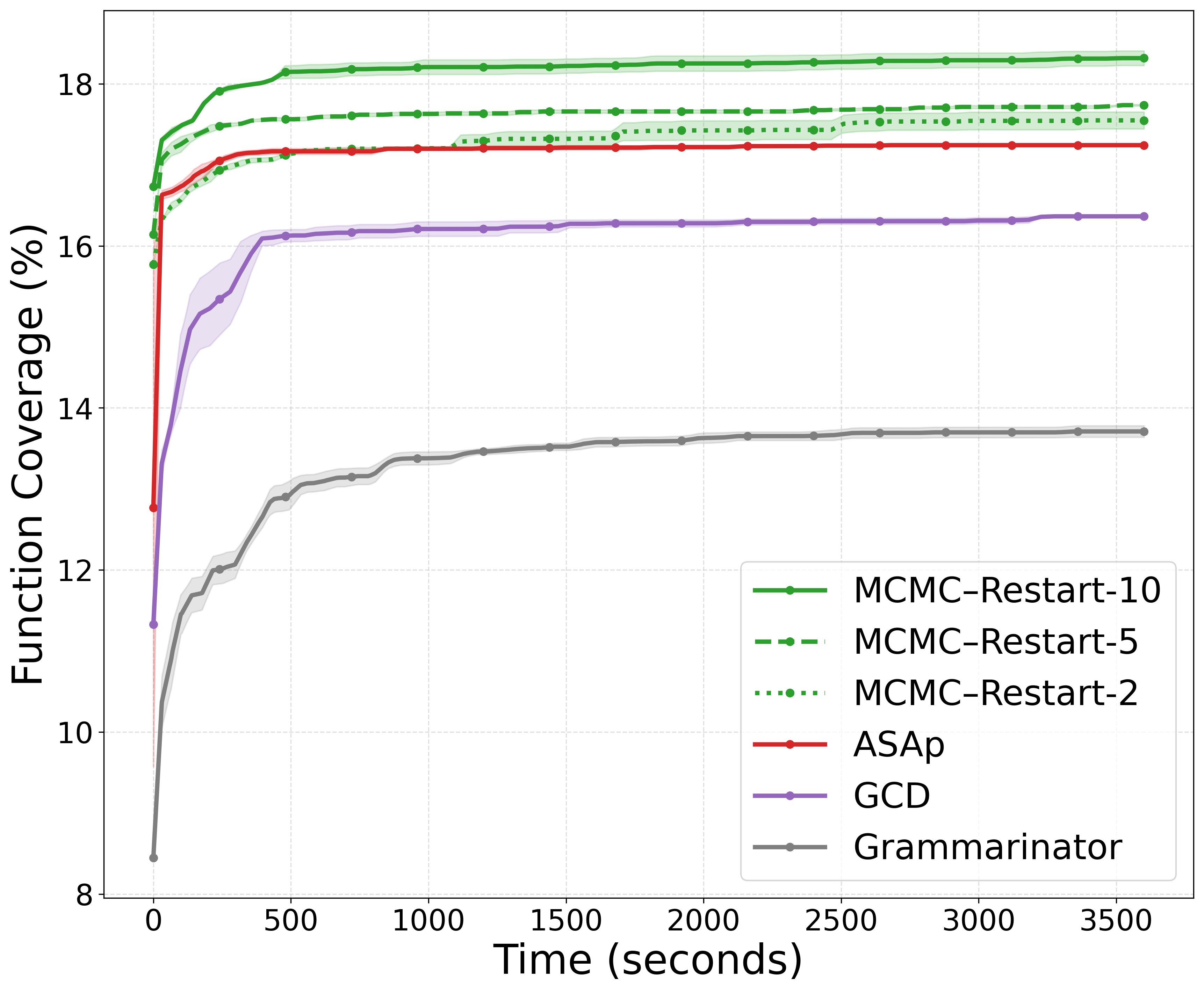}}
  \hfill
  \subcaptionbox{Uniform\label{fig:xml-pref-funccov}}
                [0.31\linewidth]{\includegraphics[width=\linewidth]{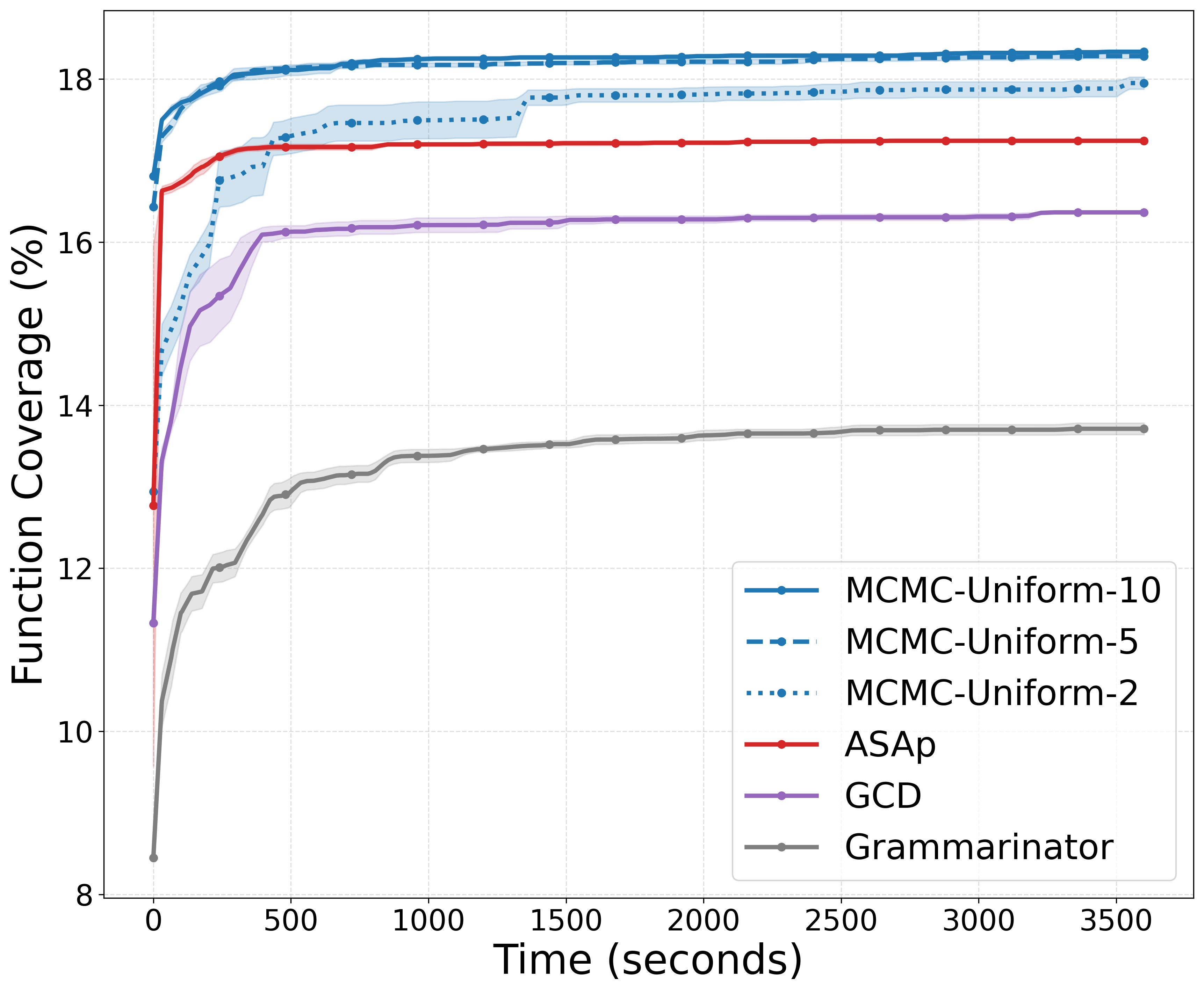}}
  \vspace{-0.4em}
  \caption{\textbf{XML}: function-coverage. Line style encodes the number of steps \(k\!\in\!\{2,5,10\}\) (dotted, dashed, solid).}
  \label{fig:function-cov-xml}
\end{figure*}

\begin{figure*}[t]
  \centering
  \subcaptionbox{Priority\label{fig:sql-prio-funccov}}
                [0.31\linewidth]{\includegraphics[width=\linewidth]{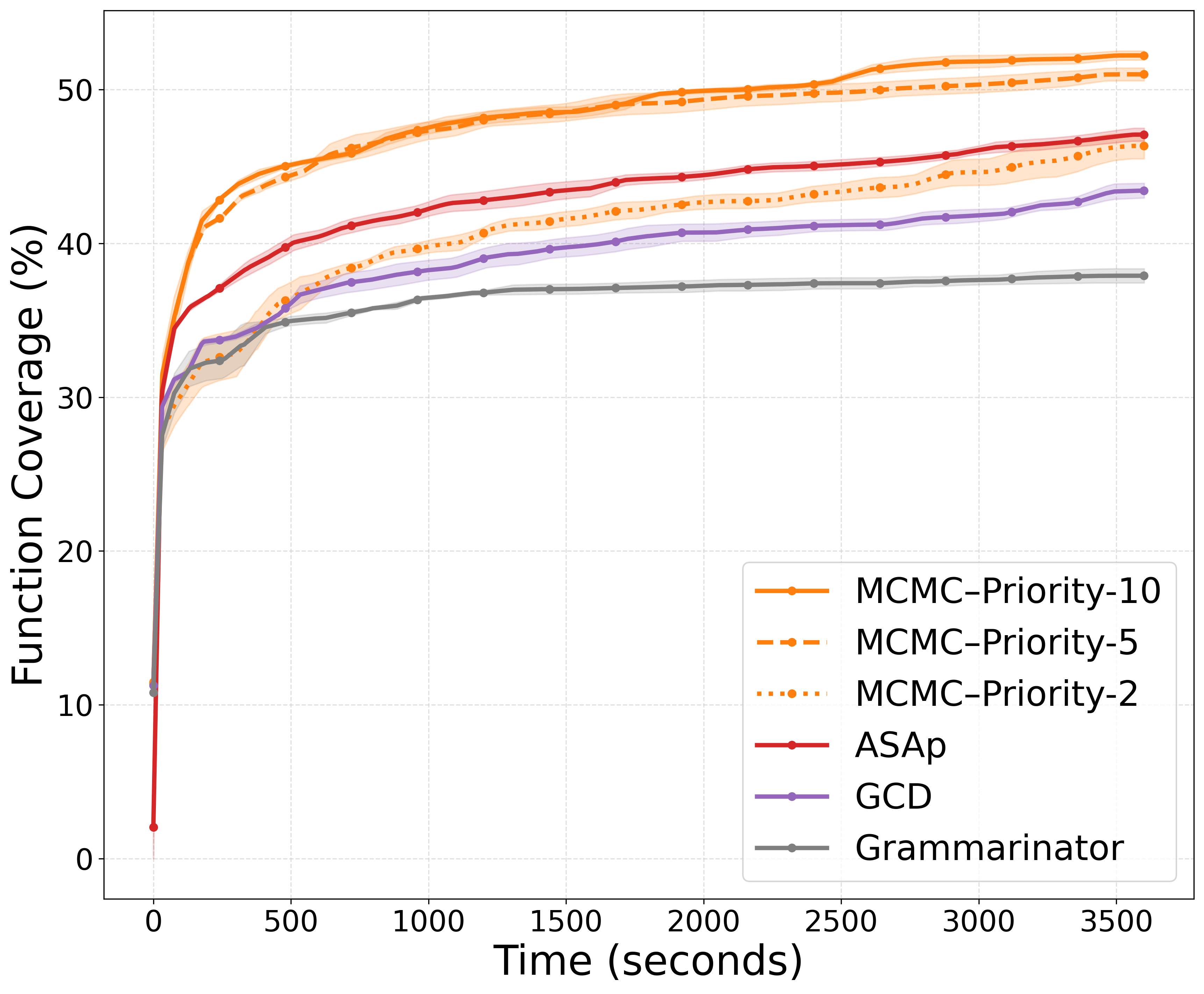}}
  \hfill
  \subcaptionbox{Restart\label{fig:sql-rest-funccov}}
                [0.31\linewidth]{\includegraphics[width=\linewidth]{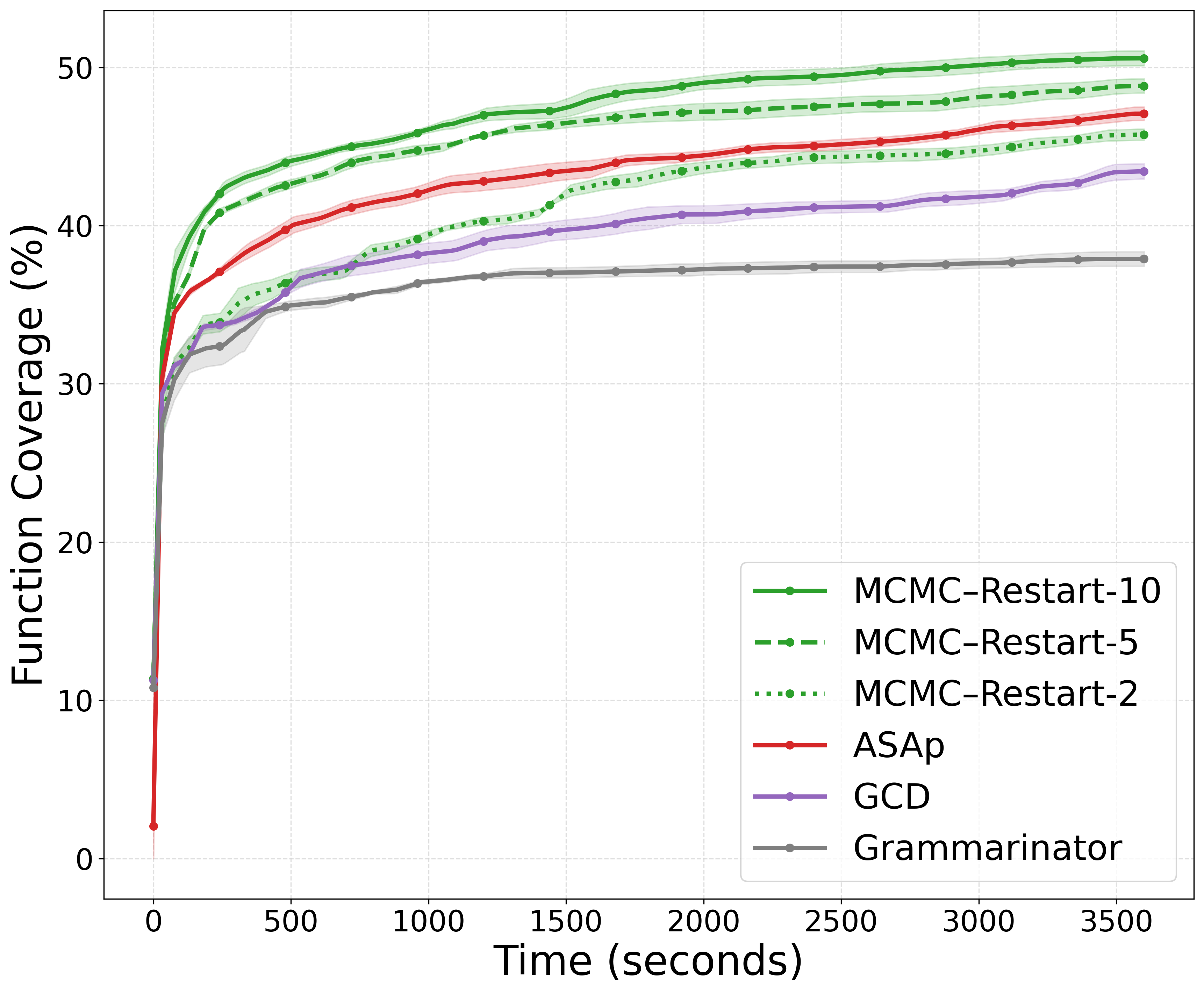}}
  \hfill
  \subcaptionbox{Uniform\label{fig:sql-pref-funccov}}
                [0.31\linewidth]{\includegraphics[width=\linewidth]{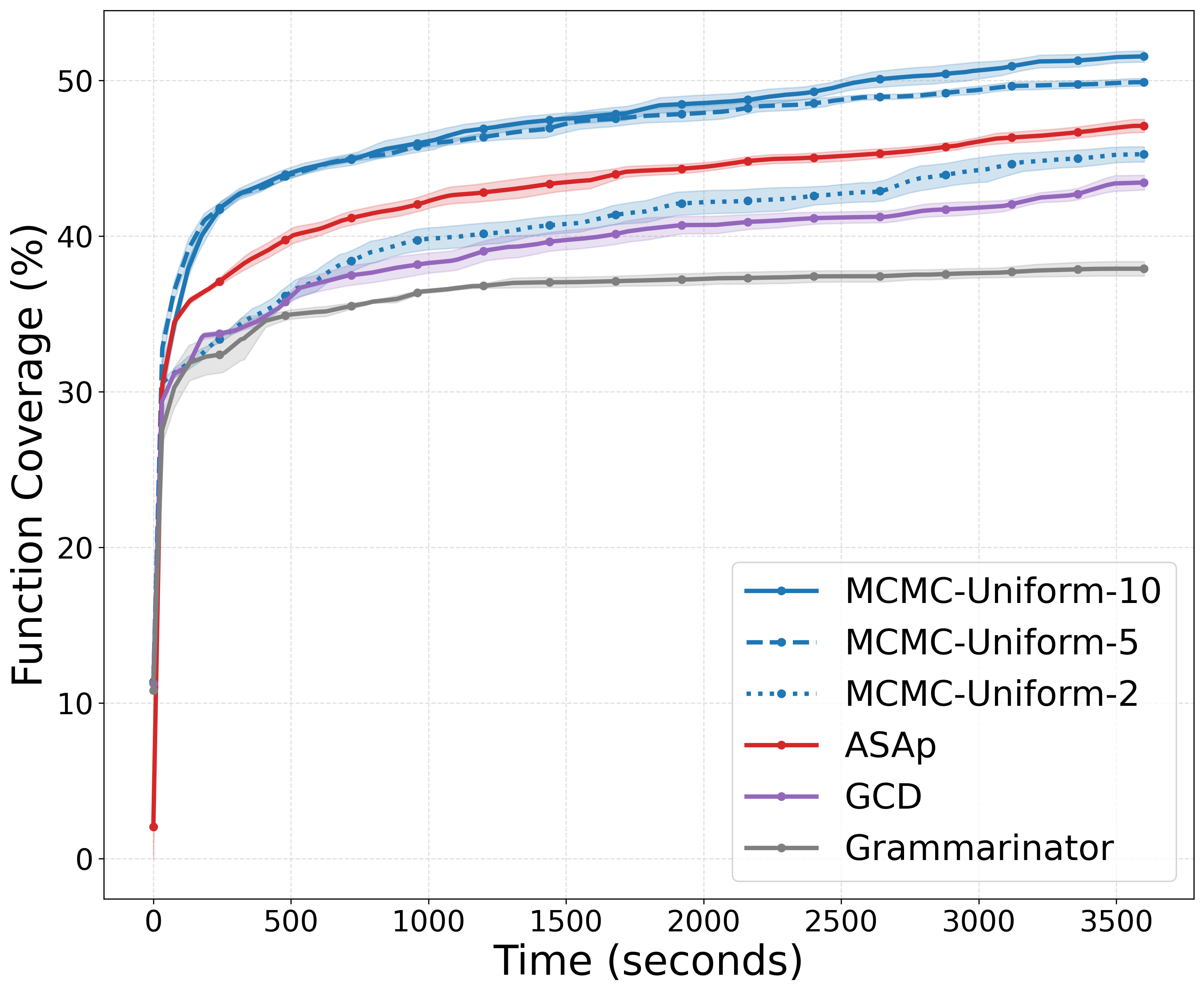}}
  \vspace{-0.4em}
  \caption{\textbf{SQL}: function-coverage.  Line style encodes the number of steps \(k\!\in\!\{2,5,10\}\) (dotted, dashed, solid).}
  \label{fig:function-cov-sql}
\end{figure*}

\begin{figure*}[t]
  \centering
  \subcaptionbox{Priority\label{fig:xml-prio-linecov}}
                [0.31\linewidth]{\includegraphics[width=\linewidth]{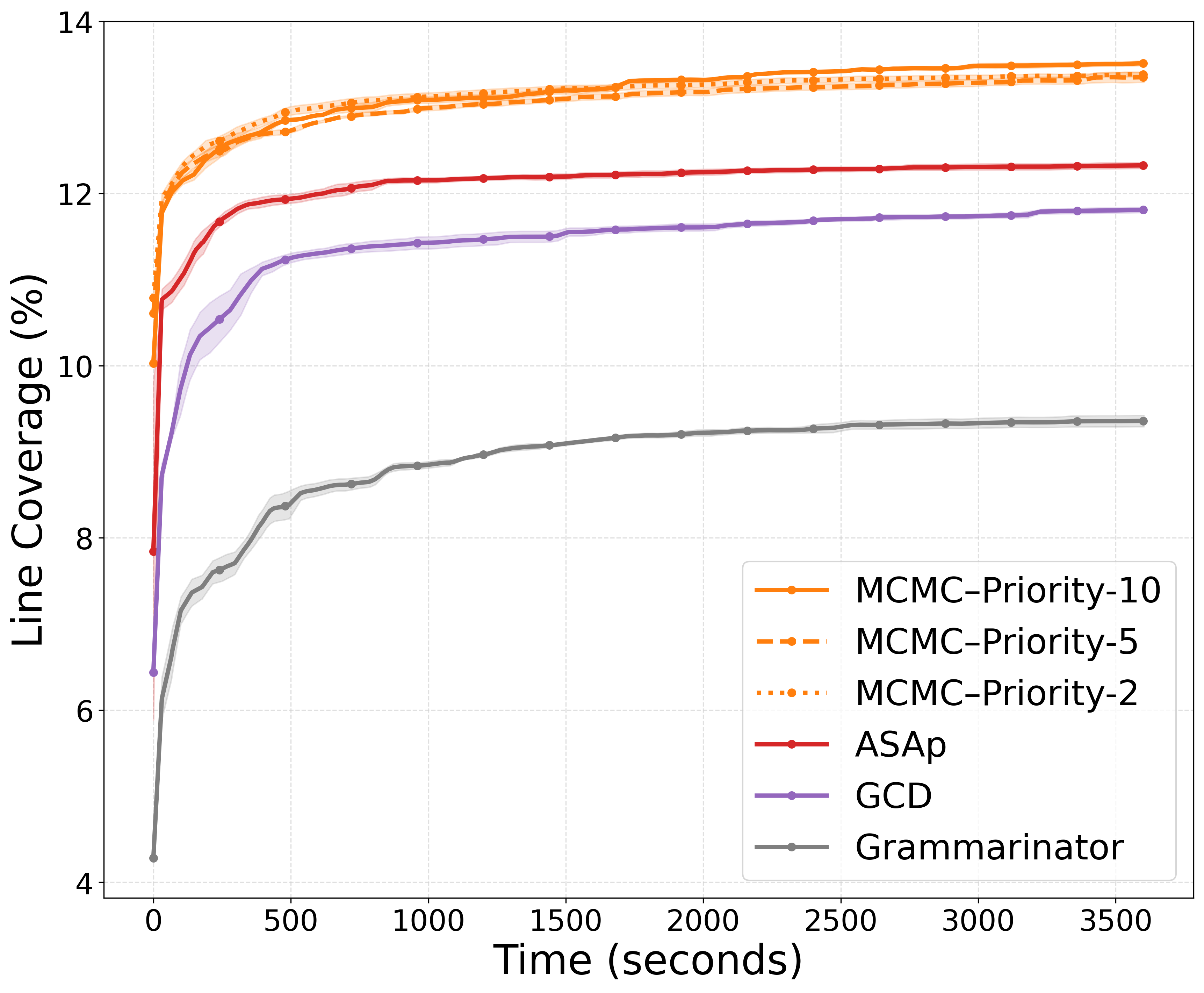}}
  \hfill
  \subcaptionbox{Restart\label{fig:xml-rest-linecov}}
                [0.31\linewidth]{\includegraphics[width=\linewidth]{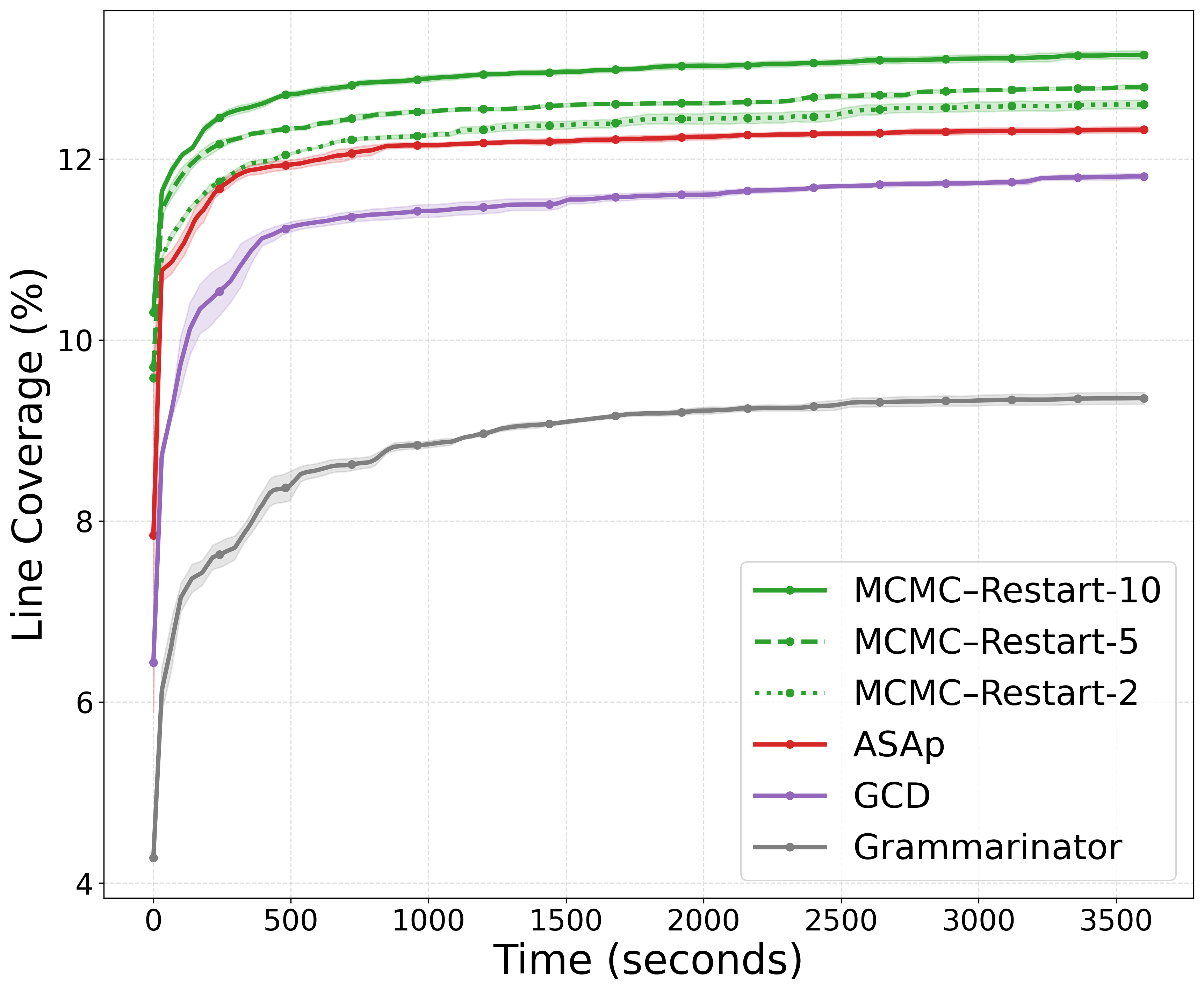}}
  \hfill
  \subcaptionbox{Uniform\label{fig:xml-pref-linecov}}
                [0.31\linewidth]{\includegraphics[width=\linewidth]{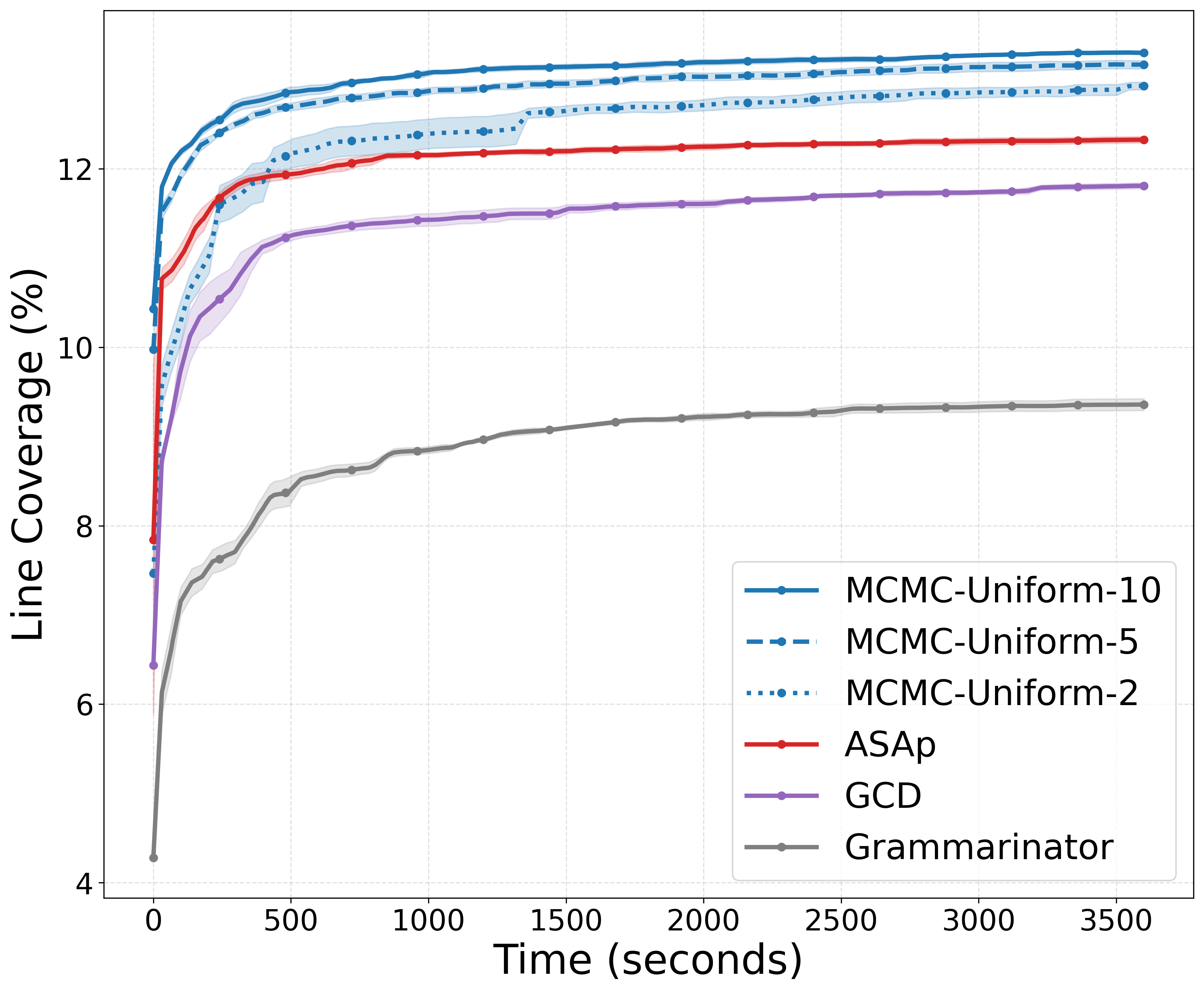}}
  \vspace{-0.4em}
  \caption{\textbf{XML}: line-coverage. Line style encodes the number of steps \(k\!\in\!\{2,5,10\}\) (dotted, dashed, solid).}
  \label{fig:line-cov-xml}
\end{figure*}

\begin{figure*}[t]
  \centering
  \subcaptionbox{Priority\label{fig:sql-prio-linecov}}
                [0.31\linewidth]{\includegraphics[width=\linewidth]{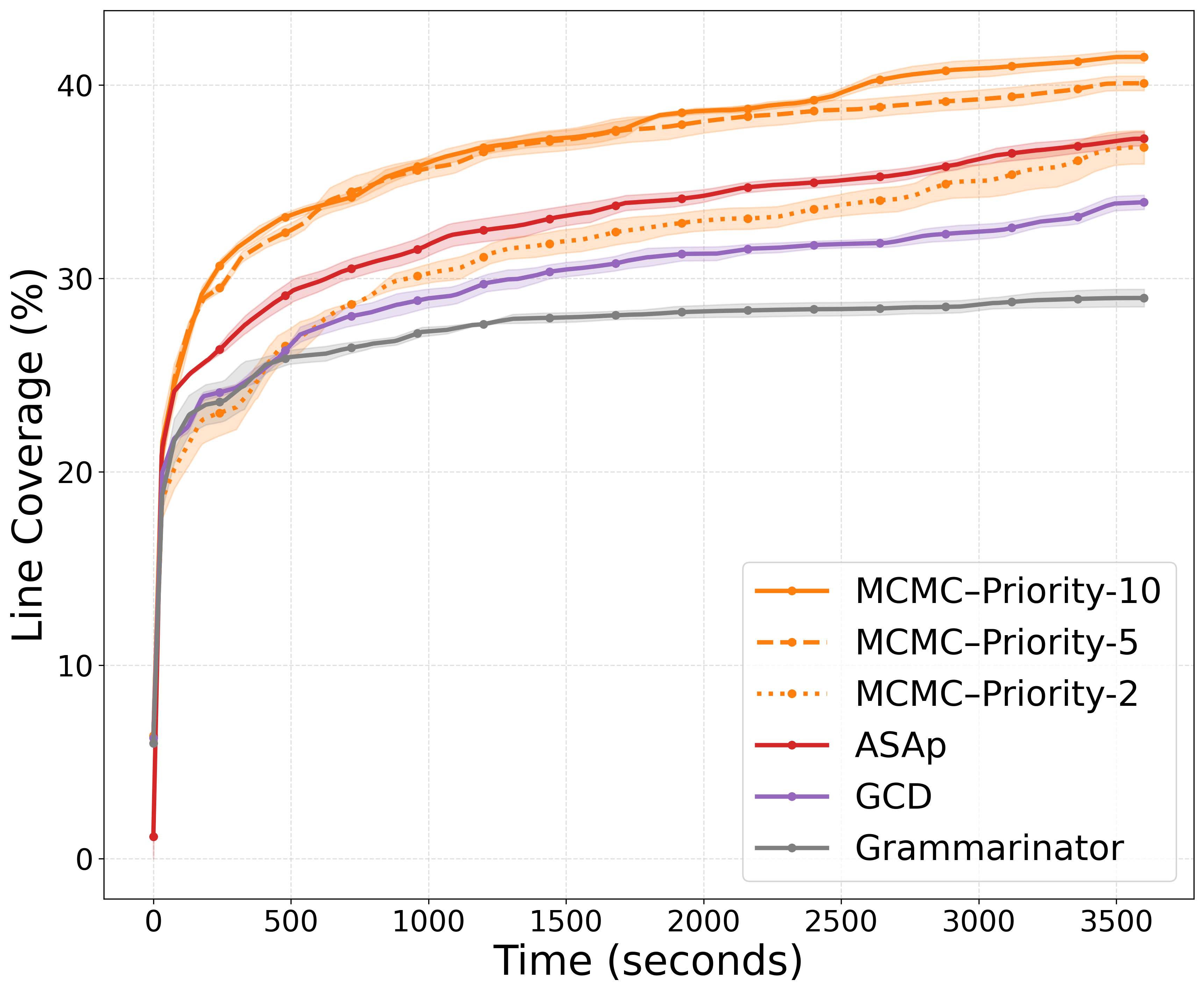}}
  \hfill
  \subcaptionbox{Restart\label{fig:sql-rest-linecov}}
                [0.31\linewidth]{\includegraphics[width=\linewidth]{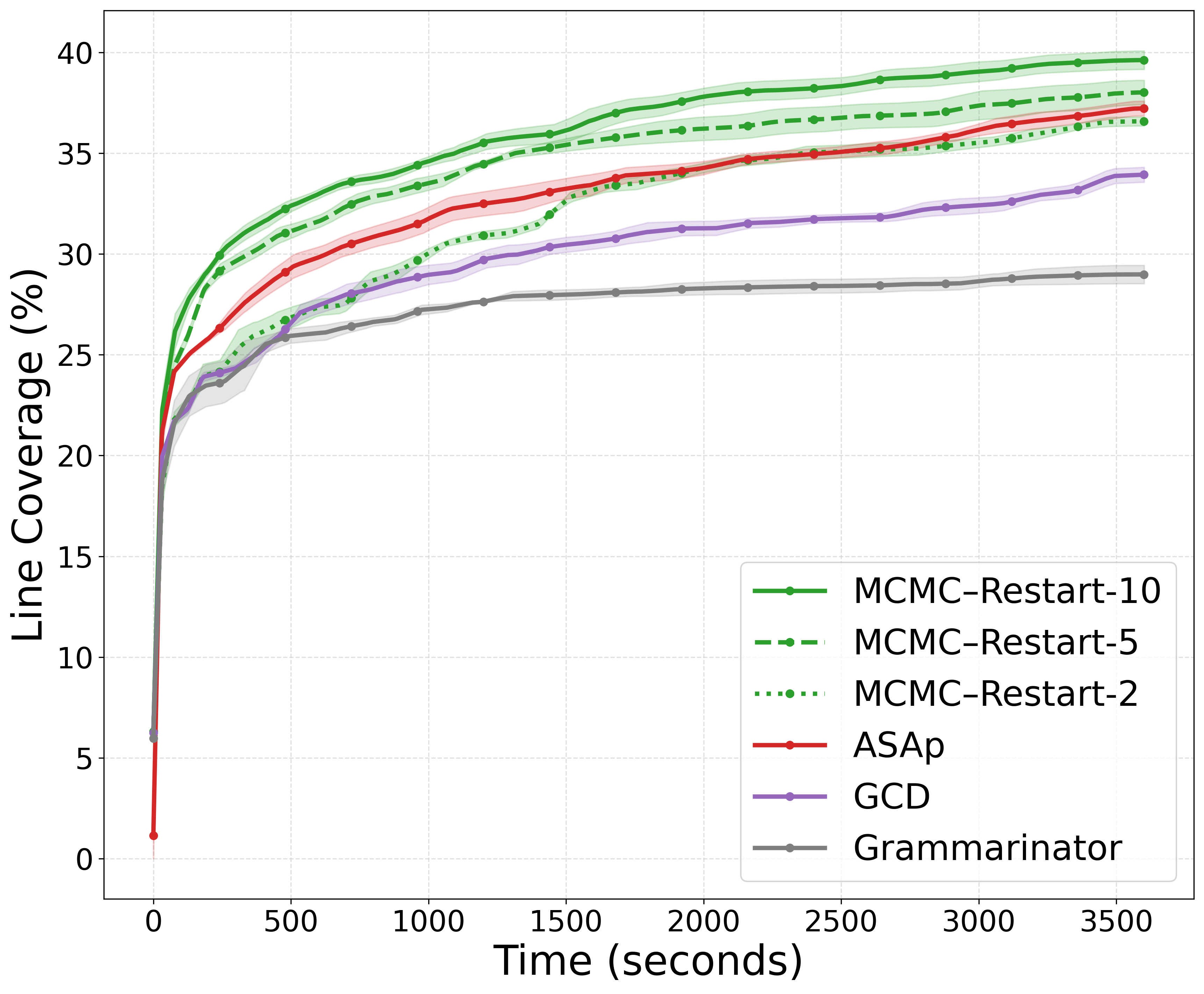}}
  \hfill
  \subcaptionbox{Uniform\label{fig:sql-pref-linecov}}
                [0.31\linewidth]{\includegraphics[width=\linewidth]{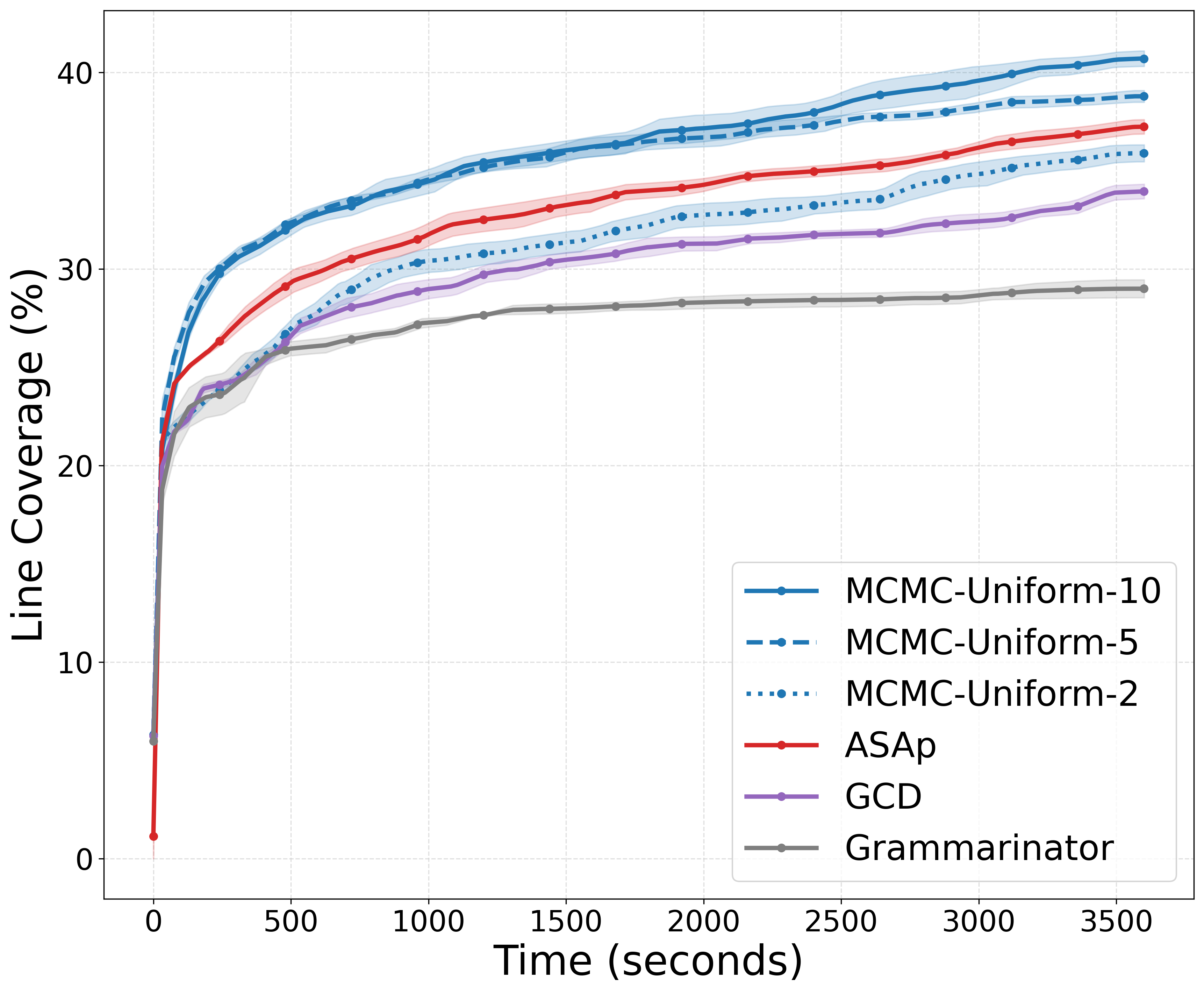}}
  \vspace{-0.4em}
  \caption{\textbf{SQL}: line-coverage.  Line style encodes the number of steps \(k\!\in\!\{2,5,10\}\) (dotted, dashed, solid).}
  \label{fig:line-cov-sql}
\end{figure*}

\subsection{Overall Coverage Results}
\label{app:overallcov}
Tables~\ref{tab:supp-metrics-sql} and~\ref{tab:supp-metrics-xml} show that the improvements achieved by our grammar-based MCMC sampler on branch coverage translate consistently to both function and line coverage across both the SQL and XML benchmarks.

\begin{table}[t]
  \centering
  \small
  \caption{SQL benchmark coverage (mean ± 95 \% CI) over five trials, and relative gain in branch coverage versus GCD, for all \(k\).  Best entries in green.}
  \label{tab:supp-metrics-sql}
  \begin{tabular}{
    l    
    c   
    S[table-format=2.2(2)]  
    S[table-format=2.2(2)]  
    S[table-format=2.2(2)]  
    S[table-format=2.1]     
  }
    \toprule
    Method         & \(\boldsymbol{k}\)
                   & {Branch (\%)} 
                   & {Function (\%)} 
                   & {Line (\%)} 
                   & {$\Delta$ vs GCD (\%)} \\
    \midrule
    \textbf{ASAP}           & 10  & 28.71(0.48) & 45.19(0.80) & 35.74(0.69) & 5.6 \\
    \textbf{GCD}            & —  & 28.26(0.68) & 43.87(0.96) & 34.28(0.73) &  0.0 \\
    \textbf{Grammarinator}  & —  & 25.04(0.91) & 39.89(0.96) & 30.52(0.94) & -11.4 \\
    \addlinespace
    \textbf{MCMC-Prefix} \\
                   & 2  & 29.47(0.75) & 45.25(0.96) & 35.88(0.84) &  4.3 \\
                   & 5  & 31.41(0.57) & 49.88(0.48) & 38.78(0.59) & 11.2 \\
                   & 10 & 32.93(0.83) & 51.53(0.71) & 40.69(0.77) & 16.5 \\
    \addlinespace
    \textbf{MCMC-Priority} \\
                   & 2  & 29.99(1.40) & 46.34(1.64) & 36.78(1.70) &  6.1 \\
                   & 5  & 32.66(0.32) & 50.99(0.82) & 40.09(0.74) & 15.6 \\
                   & 10 & \cellcolor{green!15}\bfseries33.70(0.44)
                                & \cellcolor{green!15}52.22(0.58)
                                & \cellcolor{green!15}41.45(0.62)
                                & 19.3 \\
    \addlinespace
    \textbf{MCMC-Restart} \\
                   & 2  & 29.45(0.32) & 45.75(0.70) & 36.59(0.45) &  4.2 \\
                   & 5  & 30.91(1.23) & 48.83(0.89) & 38.02(1.20) &  9.4 \\
                   & 10 & 32.36(0.74) & 50.57(0.92) & 39.62(0.90) & 14.5 \\
    \bottomrule
  \end{tabular}
\end{table}

\begin{table}[t]
  \centering
  \small
  \caption{XML coverage (mean ± 95 \% CI) over five trials, and relative gain in branch coverage versus GCD, for all \(k\in\{2,5,10\}\).  Best branch performer in green.}
  \label{tab:supp-metrics-xml}
  \begin{tabular}{
    l    
    c    
    S[table-format=2.2(2)]  
    S[table-format=2.2(2)]  
    S[table-format=2.2(2)]  
    S[table-format=2.1]     
  }
    \toprule
    Method         & \(\boldsymbol{k}\)
                   & {Branch (\%)} 
                   & {Function (\%)} 
                   & {Line (\%)} 
                   & {$\Delta$ vs GCD (\%)} \\
    \midrule
    \textbf{ASAP}           & 10 & 10.66(0.09) & 16.55(0.03) & 11.84(0.07) & 3.8 \\
    \textbf{GCD}            & — & 10.67(0.05) & 16.36(0.02) & 11.81(0.06) &  0.0 \\
    \textbf{Grammarinator}  & — &  8.49(0.08) & 13.71(0.14) &  9.36(0.13) & -20.4 \\
    \addlinespace
    \textbf{MCMC-Prefix} \\
                   & 2  & 11.77(0.18) & 17.95(0.14) & 12.93(0.12) & 10.3 \\
                   & 5  & 11.89(0.12) & 18.28(0.06) & 13.17(0.09) & 11.4 \\
                   & 10 & 12.16(0.05) & 18.33(0.03) & 13.30(0.04) & 14.0 \\
    \addlinespace
    \textbf{MCMC-Priority} \\
                   & 2  & 12.10(0.20) & 18.30(0.13) & 13.38(0.10) & 13.4 \\
                   & 5  & 12.62(0.19) & 19.11(0.17) & 13.88(0.12) & 18.2 \\
                   & 10 & \cellcolor{green!15}\bfseries12.81(0.05)
                                & \cellcolor{green!15}19.28(0.05)
                                & \cellcolor{green!15}14.05(0.06)
                                & 20.0 \\
    \addlinespace
    \textbf{MCMC-Restart} \\
                   & 2  & 11.37(0.16) & 17.55(0.12) & 12.61(0.11) &  6.6 \\
                   & 5  & 11.66(0.14) & 17.74(0.07) & 12.80(0.10) &  9.4 \\
                   & 10 & 12.00(0.06) & 18.32(0.09) & 13.15(0.09) & 12.5 \\
    \bottomrule
  \end{tabular}
\end{table}

\section{Properties and Proofs}
\label{appendix:proofs}

In this section, we formalize and prove the two key properties of our sampler (\autoref{alg:MH}), constraint satisfying (\autoref{appendix-theorem:completeness}) and monotonically converging (\autoref{appendix-theorem:converge}).

The first property follows directly from the procedure in \autoref{alg:MH}.
\begin{theorem}[Constraint Satisfying] \label{appendix-theorem:completeness} For any LM $P$, any grammar $G$, any chain length $k$, and any truncation distribution $p_\textsc{pos}$, the result of \autoref{alg:MH} is always inside $\mathcal L(G)$. 
\end{theorem}
\begin{proof}
The output of \autoref{alg:MH} can only be generated from either Line 1 or Line 11, both of which call the GCD procedure.
Since GCD samples only from the constrained language \( \mathcal{L}(G) \), the result of \autoref{alg:MH} must also fall within \( \mathcal{L}(G) \).
\end{proof}

As for monotonically converging, we prove it by applying the following theorem for Markov chains.
\begin{theorem}[Thm. 5.6.6 in \cite{DBLP:books/cu/10/D2010}] \label{theorem:aux} 
Let \( p \) be a Markov chain with countable states, and let $\tvd{q}{q'}$ denotes the total variance distance of two distributions, i.e., $\frac{1}{2}\sum_x|q(x) - q'(x)|$.

When \( p \) is irreducible, aperiodic, and has stationary distribution \( \pi \), then for any state state $x$ \( \tvd{p^k(\cdot\mid x)}{\pi}\) will converge to $0$ as \( k\) approaches to $\infty$.
\end{theorem}

\begin{theorem}[Monotonically Converging] \label{appendix-theorem:converge} For any LM $P$, any grammar $G$, and any truncation distribution $p_\textsc{pos}$, if $p_{\textsc{pos}}^{\sent}(0) > 0$ for all sequences $\sent \in \mathcal L(G)$, then the output distribution of $\autoref{alg:MH}$ will monotonically converge to $\probgrammarpg{\prob}{\grammar}$ as the chain length $k$ approaches to infinite, as shown below.
\begin{gather}
\lim_{k \rightarrow \infty} \tvd{\outprob{k}}{\probgrammarpg{\prob}{\grammar}} = 0 \label{formula:converge}\\ 
\forall k, \tvd{\outprob{k}}{\probgrammarpg{\prob}{\grammar}} \geq \tvd{\outprob{k+1}}{\probgrammarpg{\prob}{\grammar}} \label{formula:mono}
\end{gather}
where $\outprob{k}$ denotes the output distribution of \autoref{alg:MH} when the chain length is $k$.
\end{theorem}
\begin{proof}
We prove the \textbf{convergence} (\autoref{formula:converge}) by applying \autoref{theorem:aux} to our case, by verifying that the Markov chain constructed in \autoref{alg:MH} satisfies all prerequisites of \autoref{theorem:aux}.
\begin{enumerate}
\item (Countable states) In our Markov chain, the state set comprises all sequences with non-zero probability in \( \probgrammarpg{\prob}{\grammar} \), denoted as \( S \).
This set is countable since the set of all sequences is countable.

\item (Irreducibility) Let \( q \) be the proposal distribution of our Markov chain.
We start by showing that \( q(y \mid x) > 0 \) for all \( x, y \in S \).
Consider the event where the empty prefix is selected in Line 10, and \( y \) is selected by GCD in Line 11.
The probability of this event is \( p_\textsc{pos}^x(0) \cdot P_{\text{GCD}}(y) \), which is non-zero. On the other hand, \( q(y \mid x) \) is no smaller than this probability, hence it must also be non-zero.

Then, by the definition of the transition probability \( p \) in the Metropolis-Hastings algorithm\footnote{The greater-than part of the first inequality captures the special case where \( x = y \):}
\[
\forall x, y \in S, \quad p(y \mid x) \geq q(y \mid x) \cdot \alpha(x, y) = q(y \mid x) \cdot \max \left\{1, \frac{P(y)q(x \mid y)}{P(x) q(y \mid x)}\right\}.
\]
All values on the right-hand side are positive, so \( p(y \mid x) \) must also be positive, implying the irreducibility of the Markov chain.

\item (Aperiodicity) By the above analysis, \( p(x \mid x) > 0 \) for any state \( x \), implying aperiodicity.

\item (Stationary distribution) The Metropolis-Hastings algorithm ensures that the target distribution \( \probgrammarpg{\prob}{\grammar} \) is a stationary distribution of the constructed Markov chain.
\end{enumerate}

Hence, all prerequisites of \autoref{theorem:aux} are satisfied, then \autoref{formula:converge} follows directly from it.

Then, we prove the \textbf{monotocity} by the following derivation, where $p(\sent \mid \sent')$ denotes the probability for our Markov chain to move from $\sent'$ to $\sent$.
\begin{align*}
\tvd{\outprob{k+1}}{\probgrammarpg{\prob}{\grammar}} &= \frac{1}{2} \sum_{\sent} \left| \outprob{k+1}(\sent) - \probgrammarpg{\prob}{\grammar}(\sent)\right| \\ 
& = \frac{1}{2} \sum_{\sent}\left| \sum_{\sent'} \big( \outprob{k}(\sent') - \probgrammarpg{\prob}{\grammar}(\sent') \big) \cdot p(\sent\mid \sent') \right| \\ 
& \leq \frac{1}{2} \sum_{\sent} \sum_{\sent'} \big| \outprob{k}(\sent') - \probgrammarpg{\prob}{\grammar}(\sent') \big| \cdot p(\sent\mid \sent') \\
& = \frac{1}{2} \sum_{\sent'} \big| \outprob{k}(\sent') - \probgrammarpg{\prob}{\grammar}(\sent') \big| \left(\sum_{\sent} p(\sent \mid \sent')\right) \\ 
& = \frac{1}{2} \sum_{\sent'} \big| \outprob{k}(\sent') - \probgrammarpg{\prob}{\grammar}(\sent') \big| = \tvd{\outprob{k}}{\probgrammarpg{\prob}{\grammar}}
\end{align*}
\end{proof}

\section{Benchmarks by~\citet{park2024grammaraligned}}
\label{app:sygus_eval}

We evaluate the convergence properties of Grammar-Aligned MCMC Sampling empirically on the benchmark tasks proposed by \citet{park2024grammaraligned}.
\autoref{fig:gcd_scatters} relates the $KL(\probgrammar \Vert GCD)$ and $KL(\probgrammar \Vert \operatorname{MCMC-T}(k))$ for $k=10$, 
for $T\in$\{Uniform, Priority, Restart\}.
Each point represents a single task.
Points below the diagonal indicate tasks where MCMC approximates $\probgrammar$ better than GCD.
\autoref{fig:asap_scatters} displays the same information, but for ASAp$(10)$ instead of GCD.

\autoref{fig:slia_kl}, \autoref{fig:bv4_kl}, \autoref{fig:cp_kl} compare the distance to $\probgrammar$ for ASAp$(k)$ and MCMC-T$(k)$, as $k$ increases, across all the benchmark tasks.
Lower KL-Divergence indicates a better approximation to $\probgrammar$.

GCD, MCMC$(k)$ and ASAp$(k)$ are all approximated using 100 samples, and $\probgrammar$ is approximated using all the samples acquired during the runs of MCMC and ASAp.
A $95\%$ confidence band is shown for convergence plots, computed via Bootstrapping.

\begin{figure}[H]
    \centering
    \begin{tabular}{ccc}
        \includegraphics[width=0.3\textwidth]{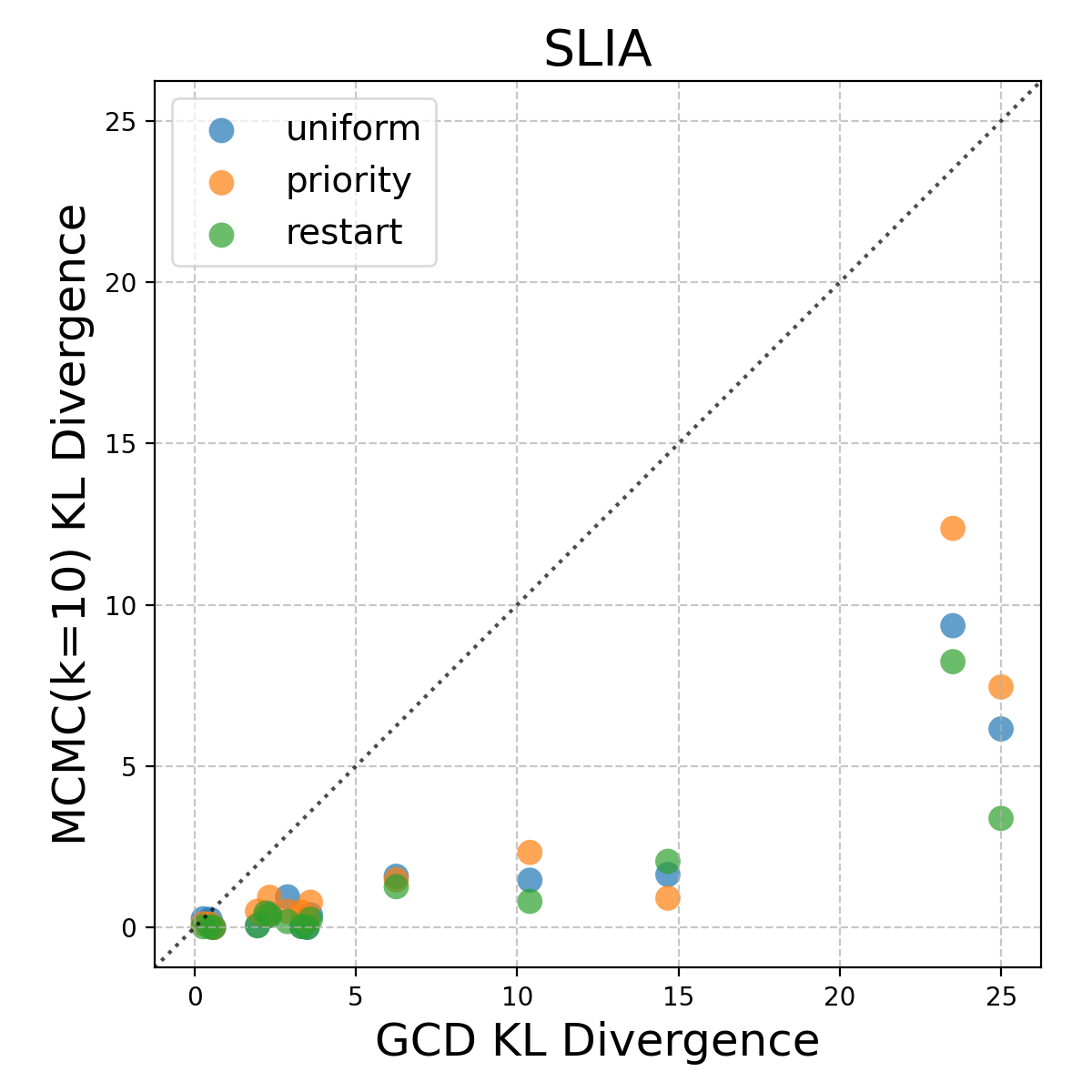} &
        \includegraphics[width=0.3\textwidth]{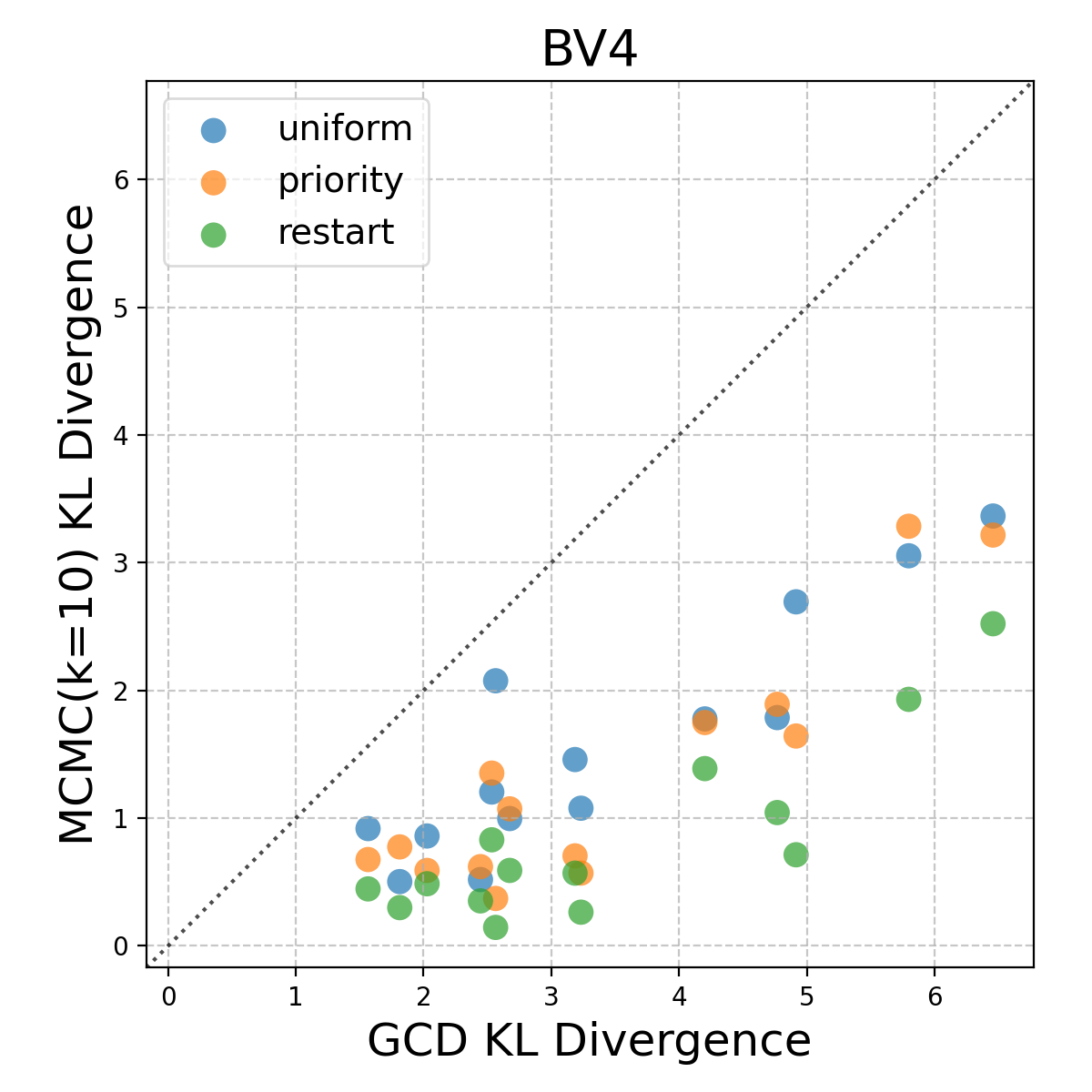} &
        \includegraphics[width=0.3\textwidth]{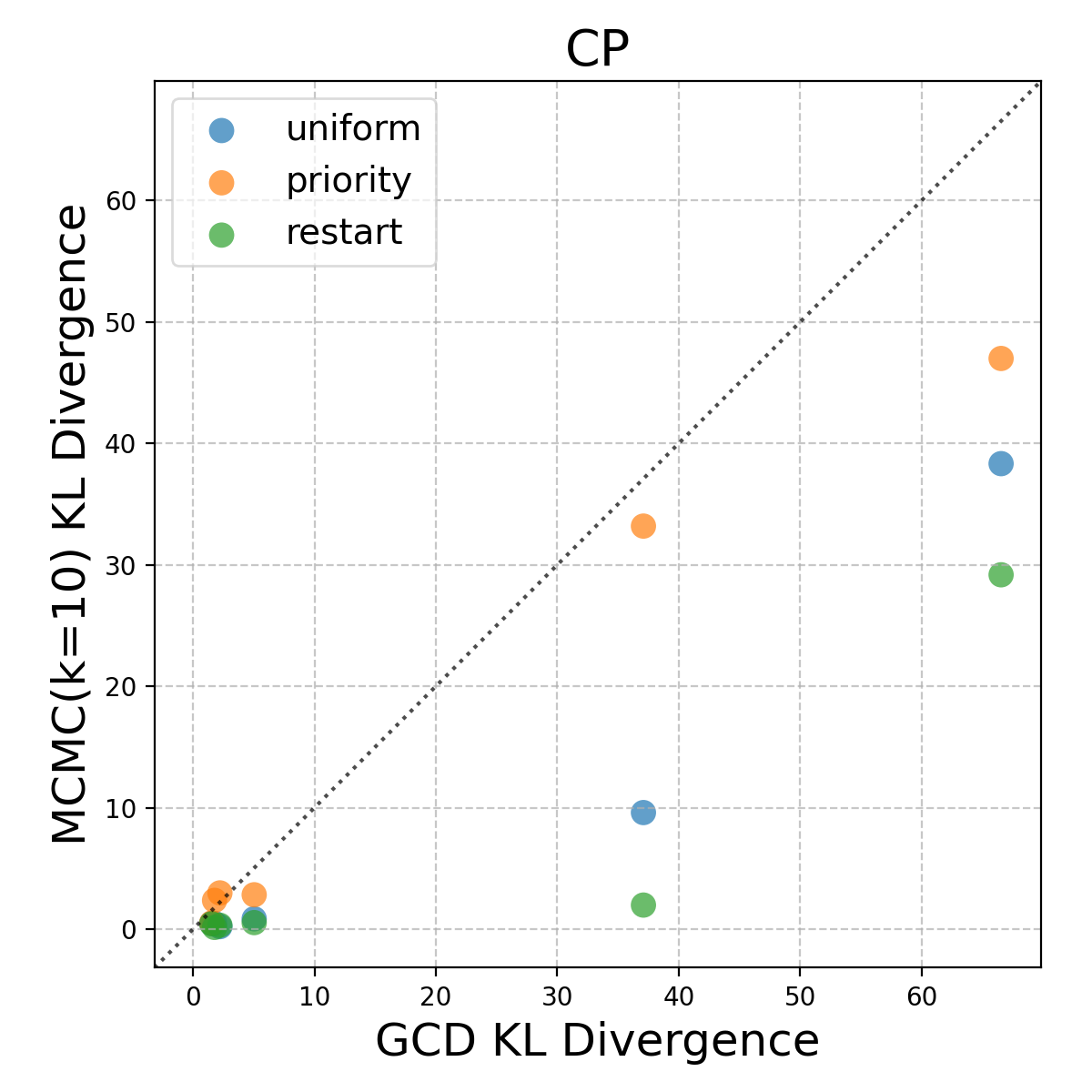} \\
    \end{tabular}
    \caption{KL-Divergence for GCD vs MCMC($k=10$) by subset}
    \label{fig:gcd_scatters}
\end{figure}

\begin{figure}[H]
    \centering
    \begin{tabular}{ccc}
        \includegraphics[width=0.3\textwidth]{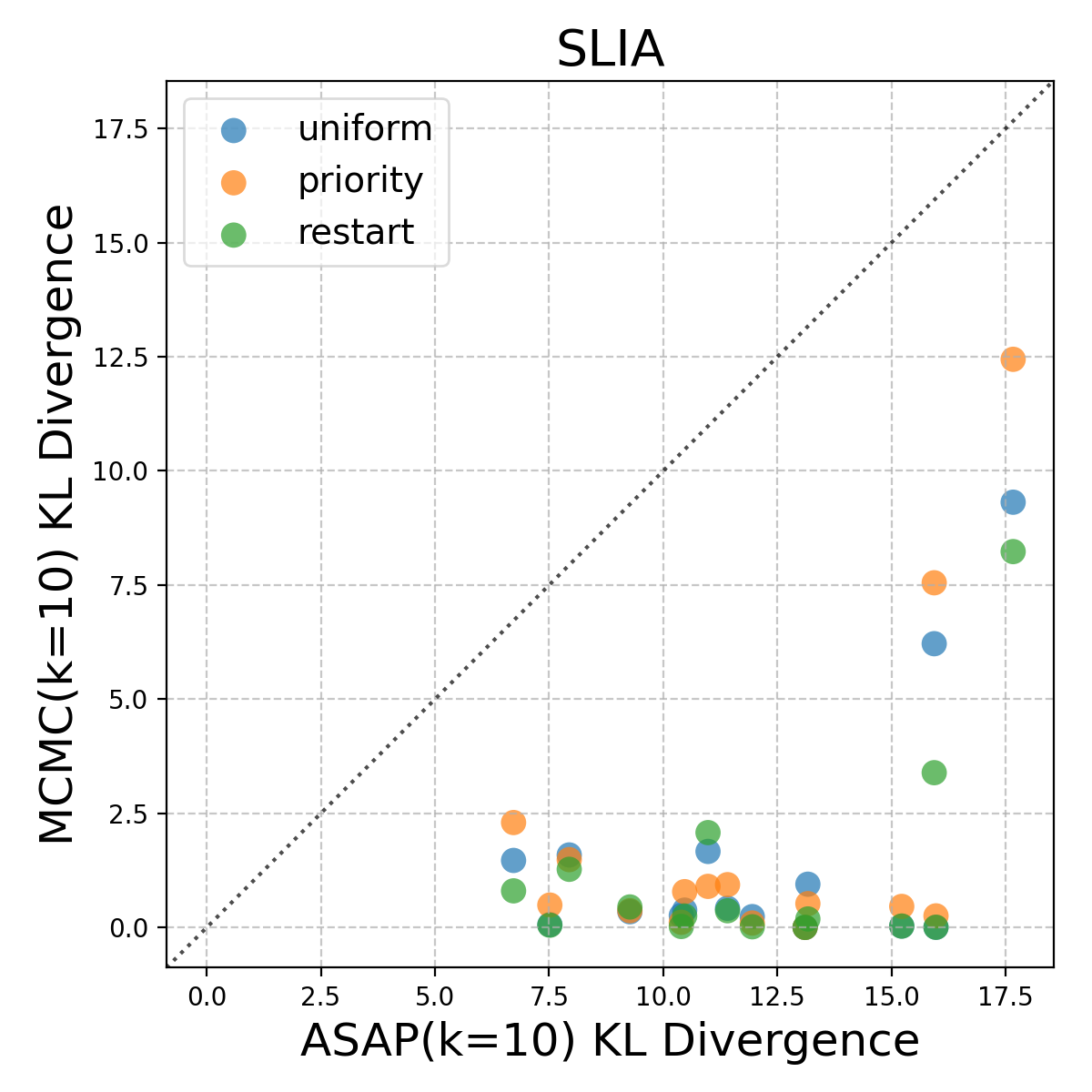} &
        \includegraphics[width=0.3\textwidth]{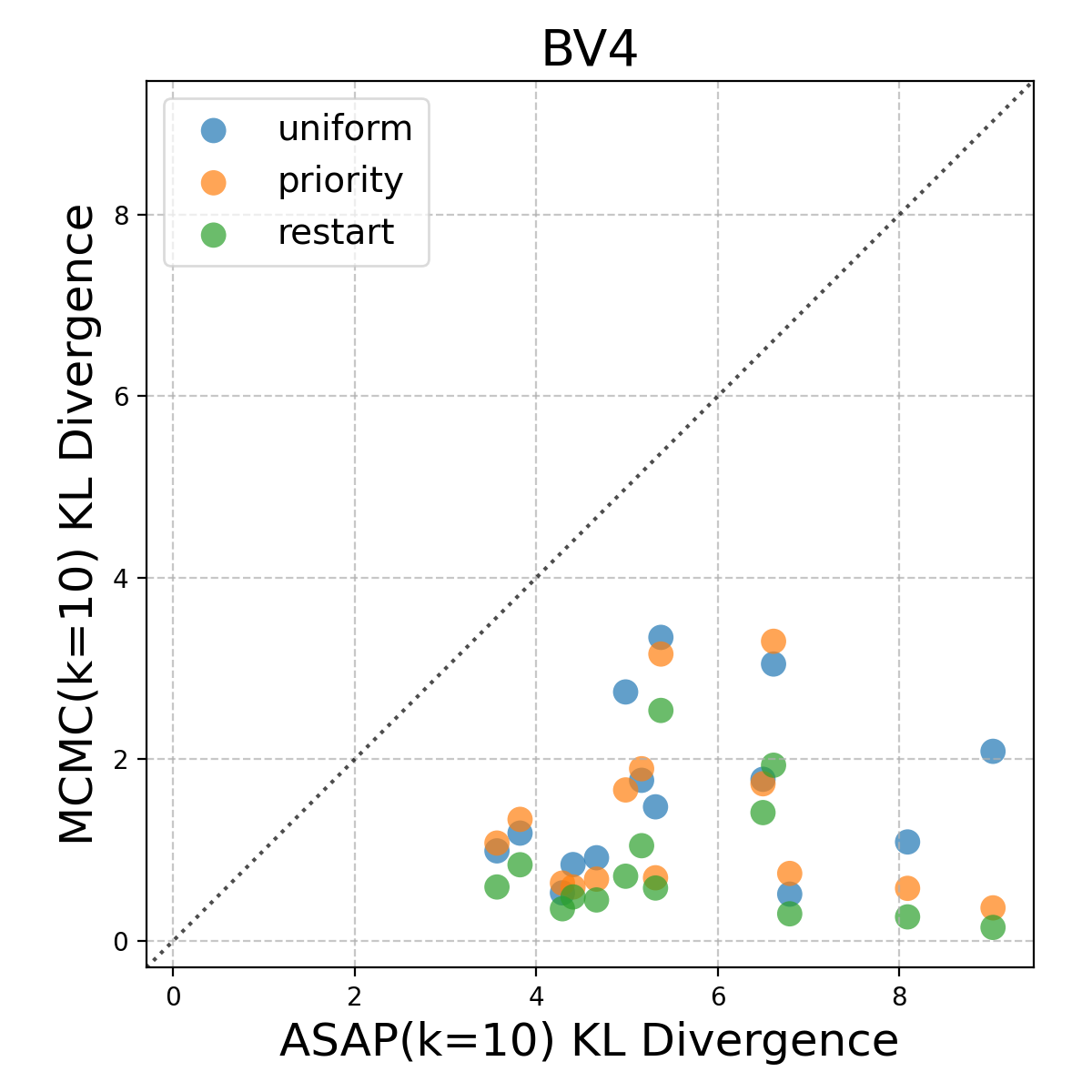} &
        \includegraphics[width=0.3\textwidth]{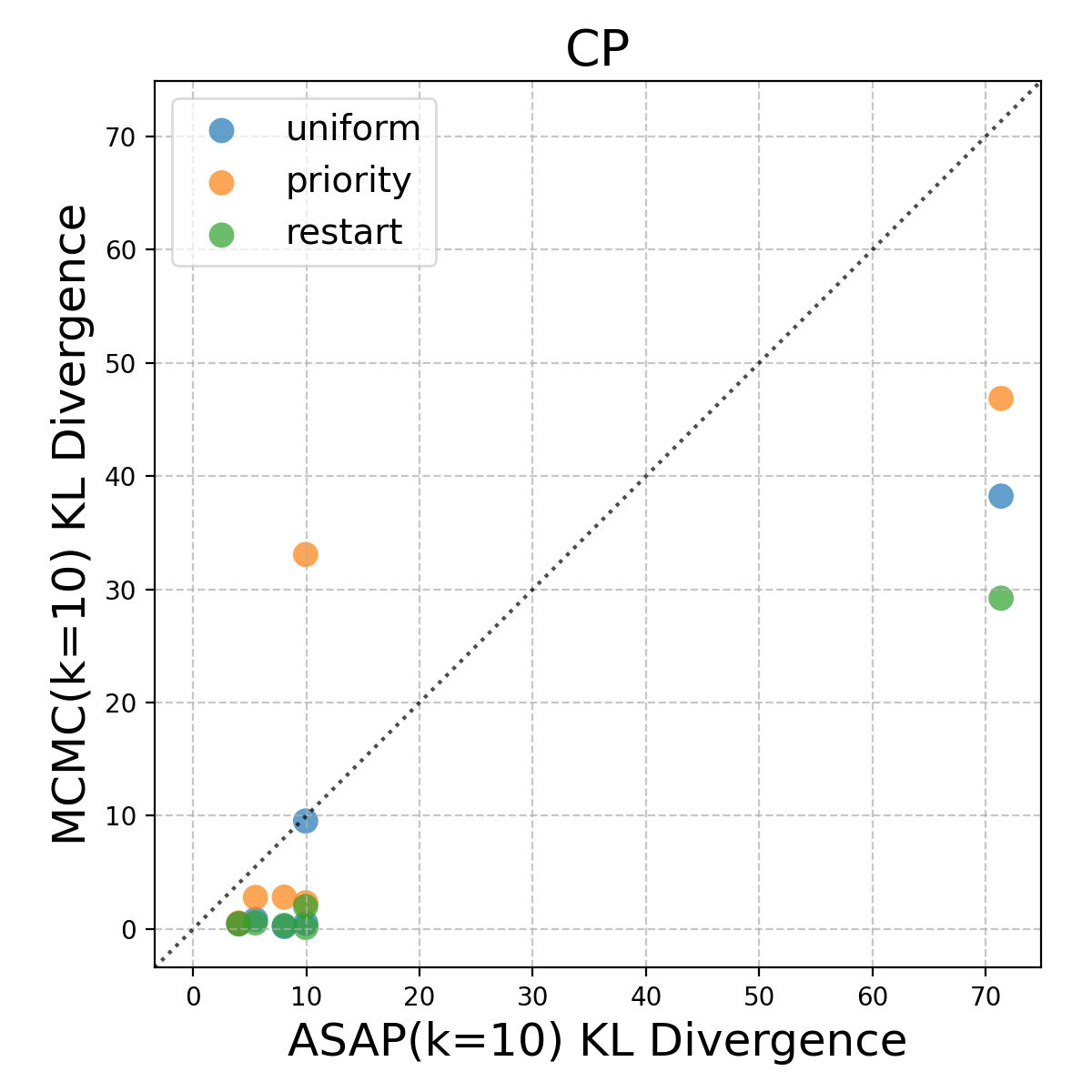} \\
    \end{tabular}
    \caption{KL-Divergence for ASAp($k=10$) vs MCMC($k=10$) by subset}
    \label{fig:asap_scatters}
\end{figure}

\begin{figure}[H]
    \centering
    \begin{tabular}{cccc}
        \includegraphics[width=0.22\textwidth]{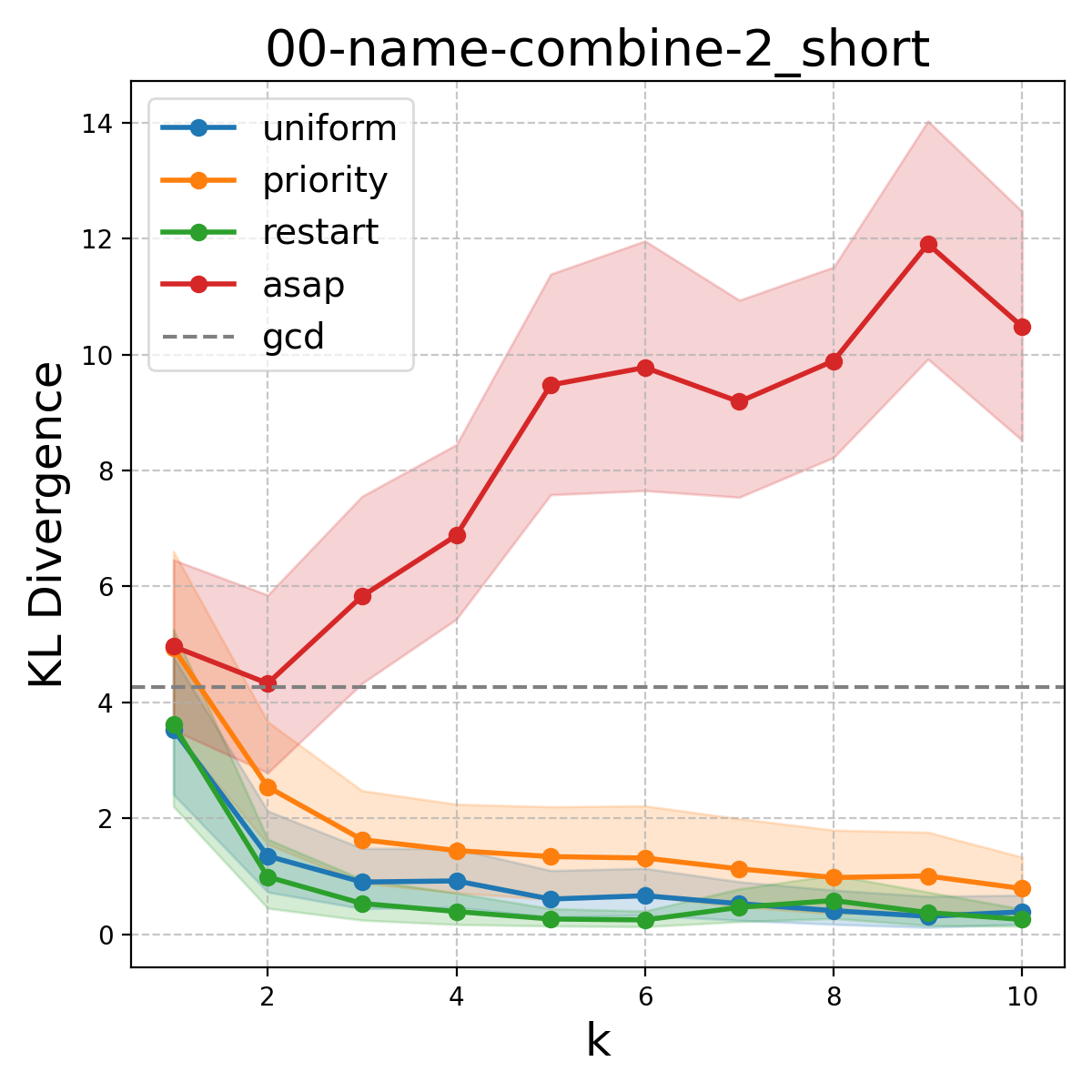} &
        \includegraphics[width=0.22\textwidth]{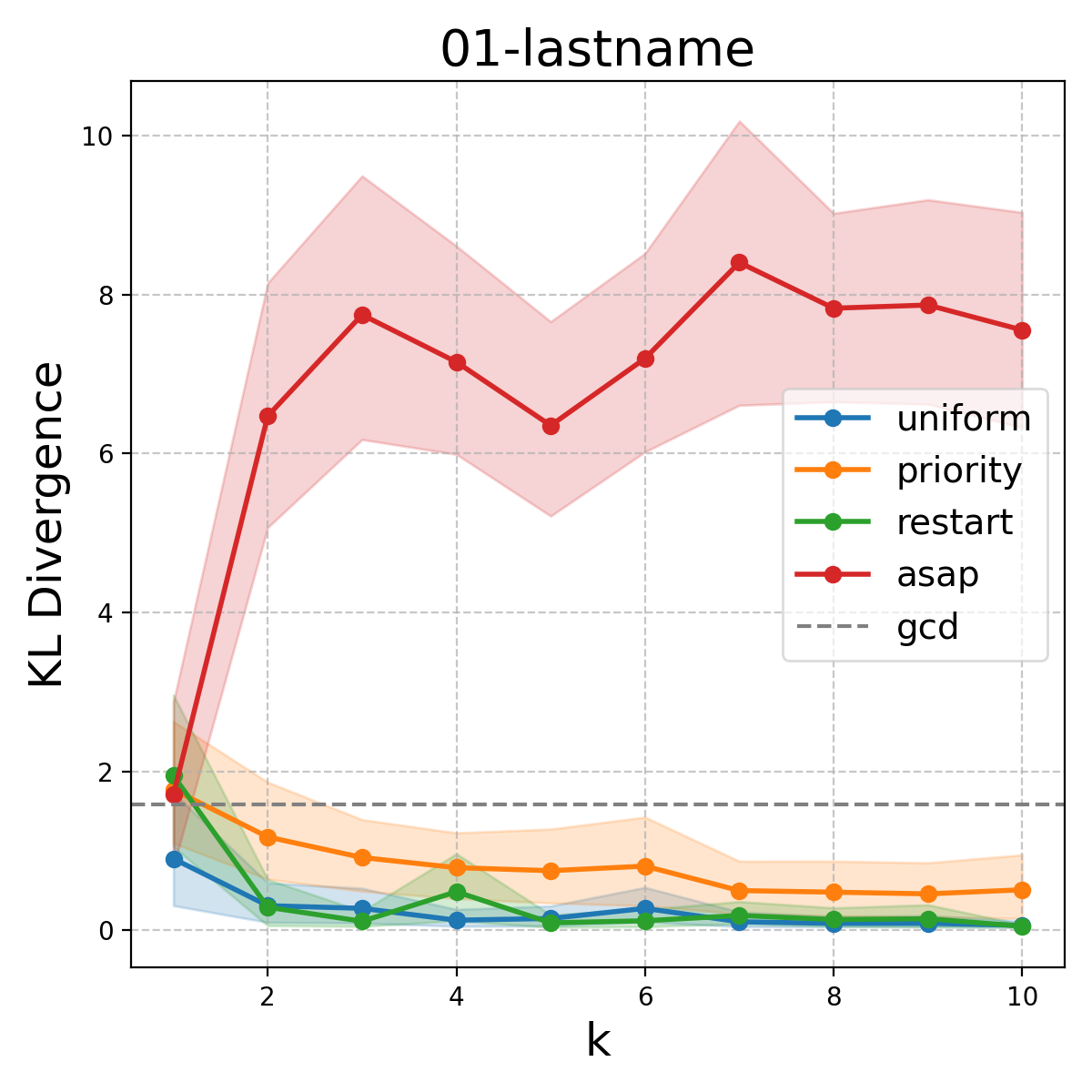} &
        \includegraphics[width=0.22\textwidth]{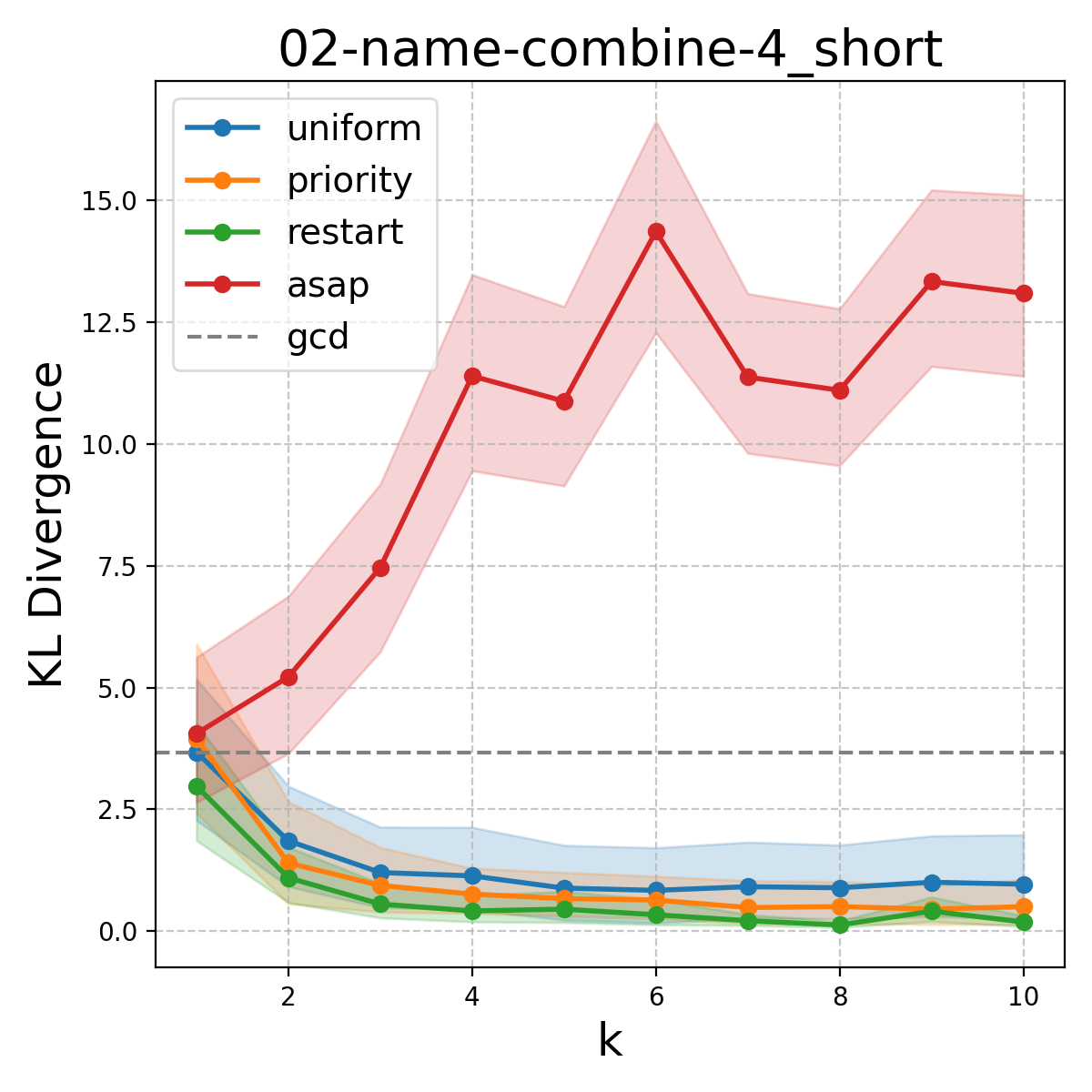} &
        \includegraphics[width=0.22\textwidth]{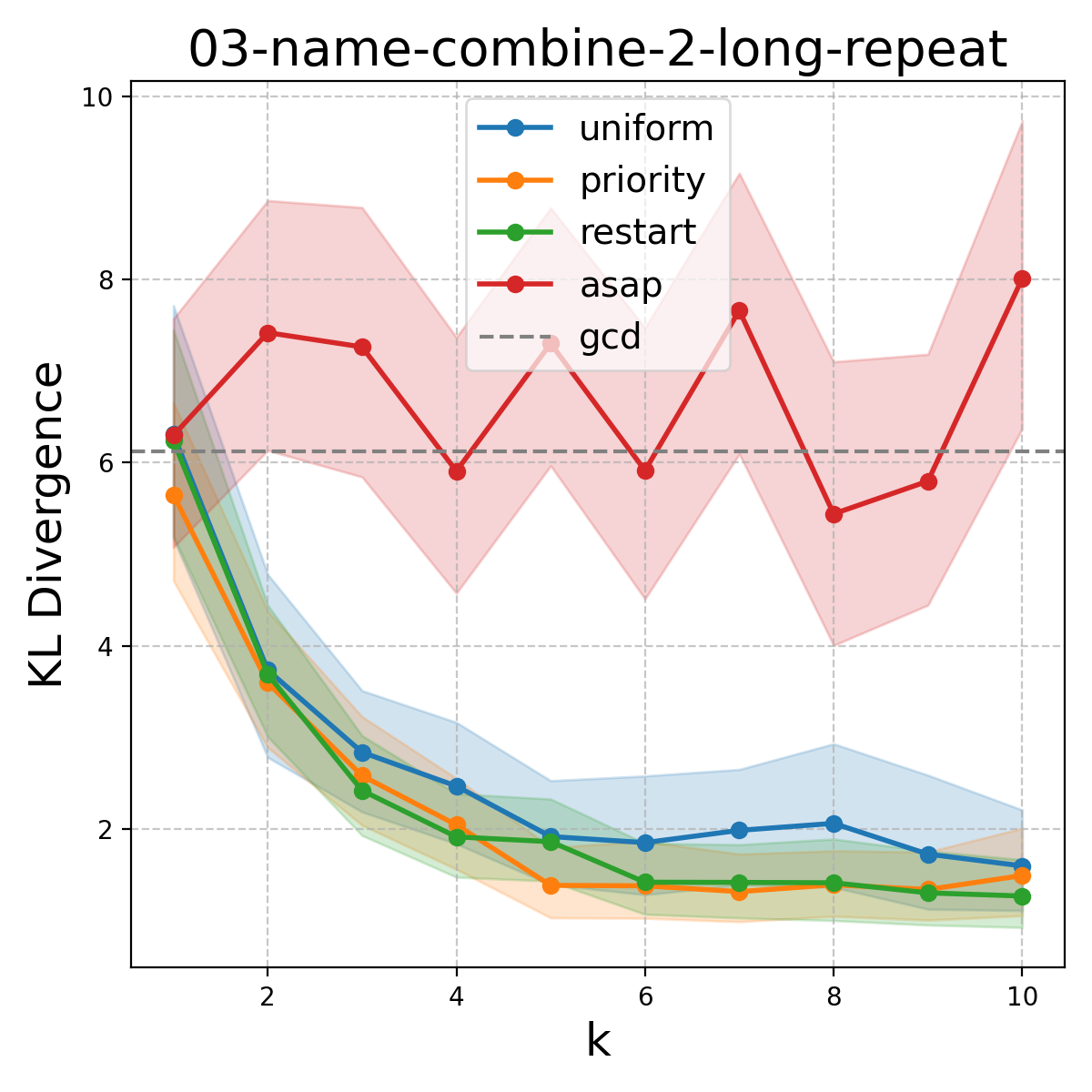} \\
        \includegraphics[width=0.22\textwidth]{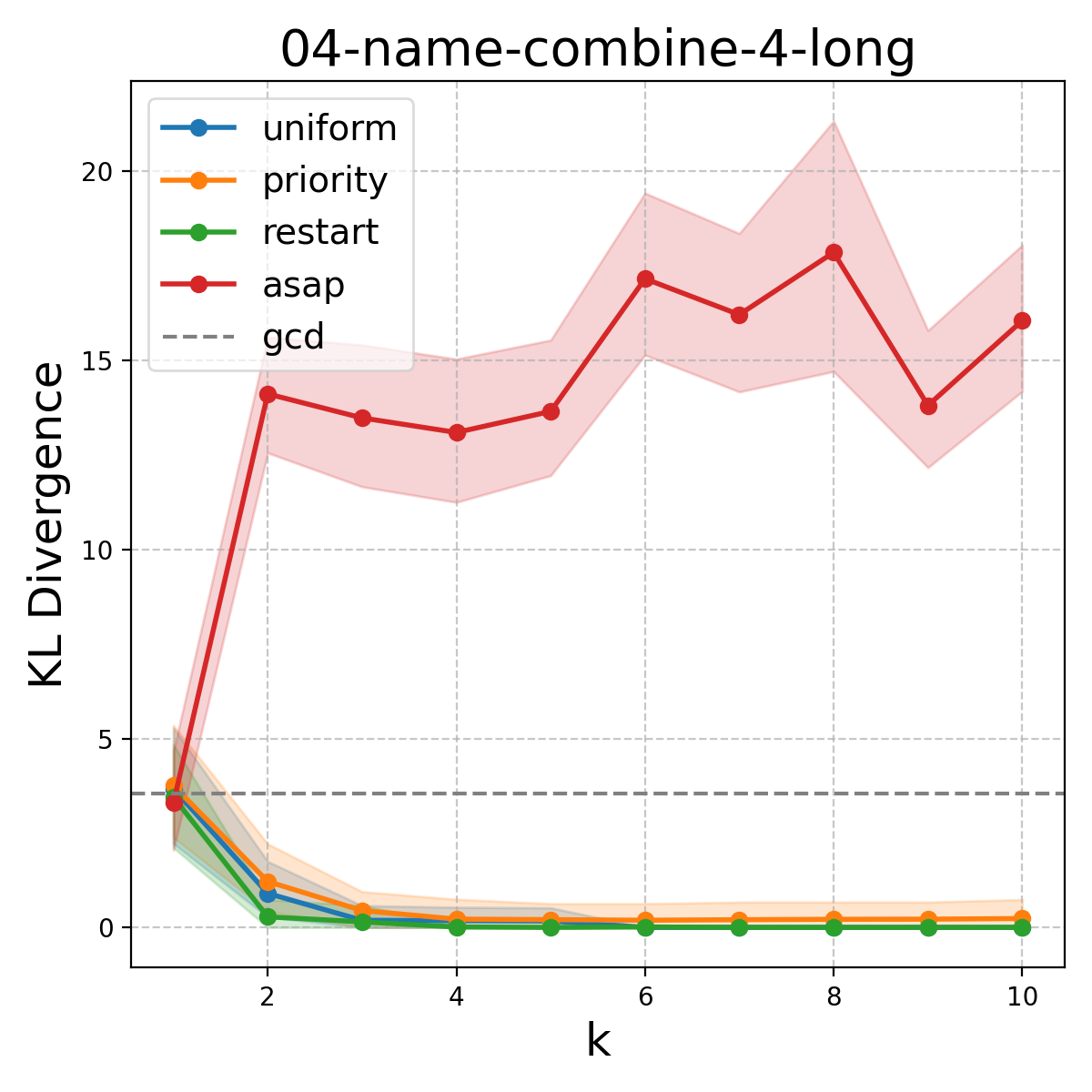} &
        \includegraphics[width=0.22\textwidth]{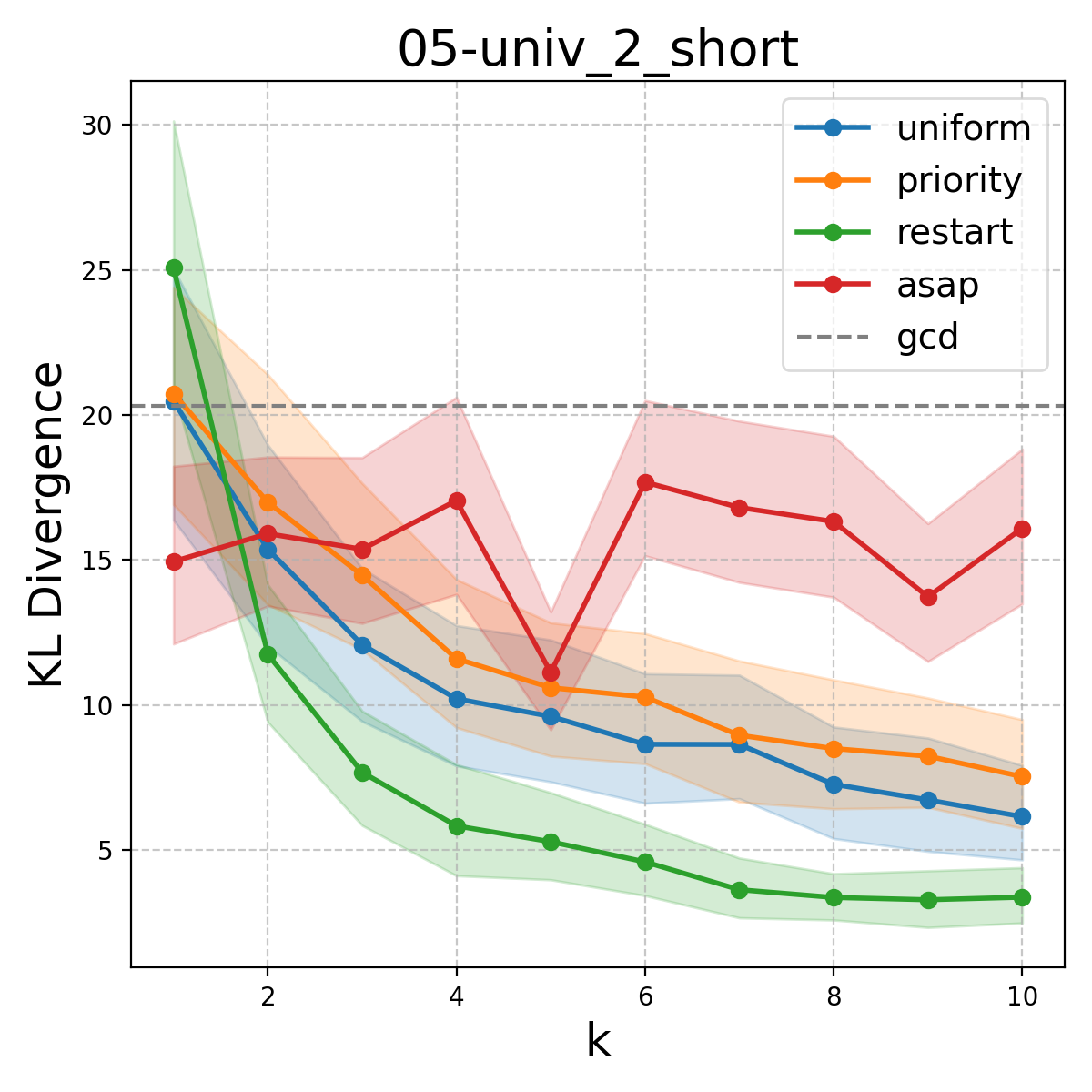} &
        \includegraphics[width=0.22\textwidth]{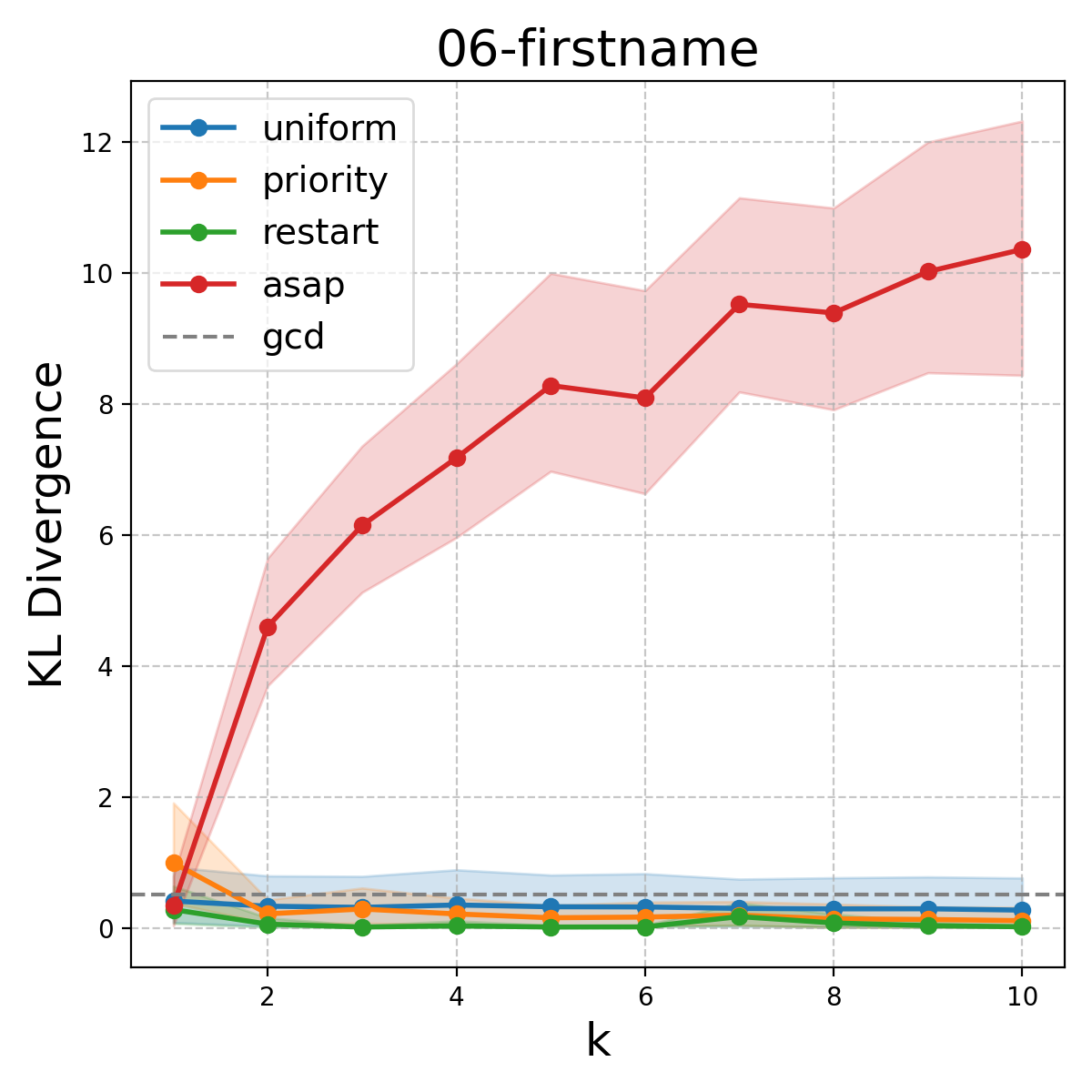} &
        \includegraphics[width=0.22\textwidth]{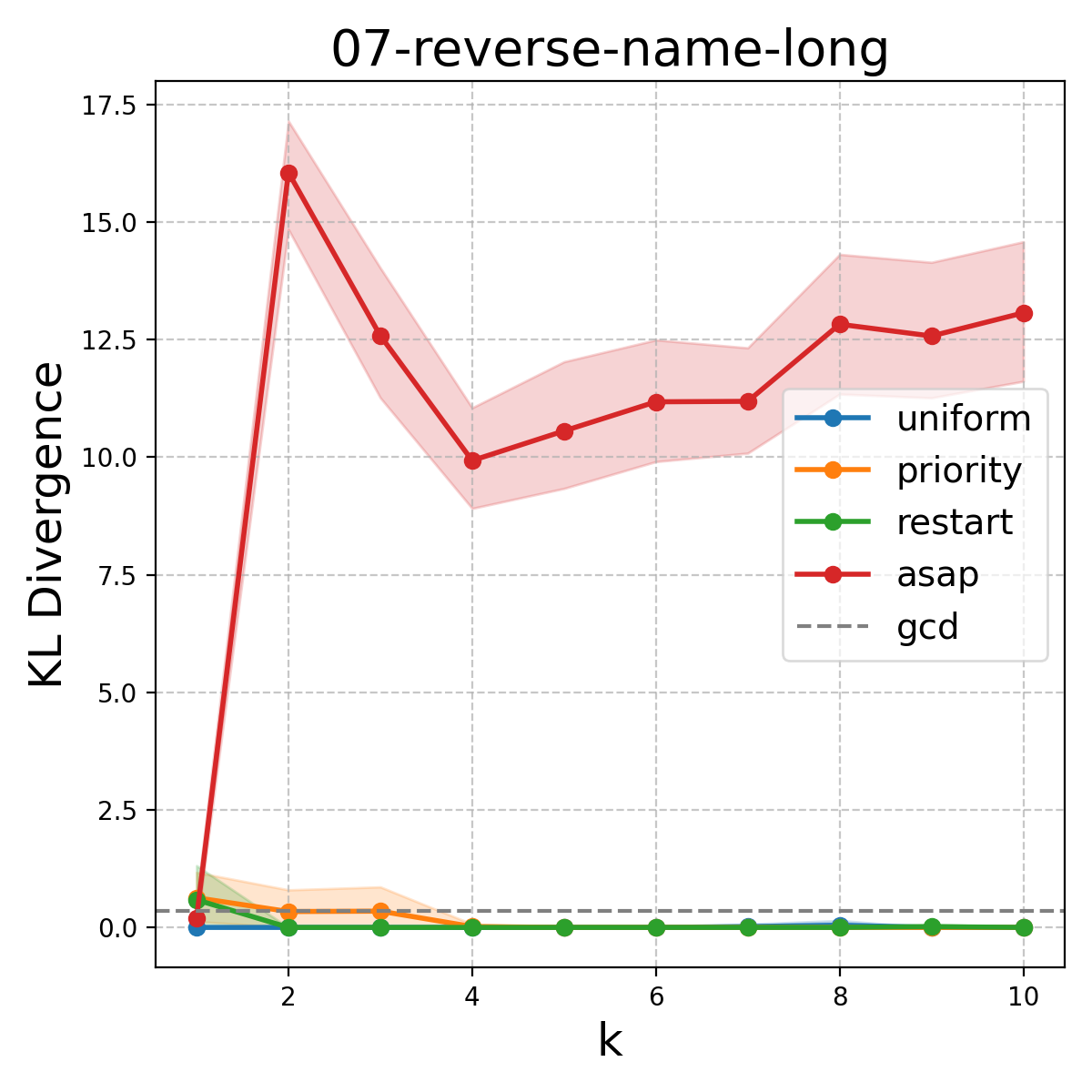} \\
        \includegraphics[width=0.22\textwidth]{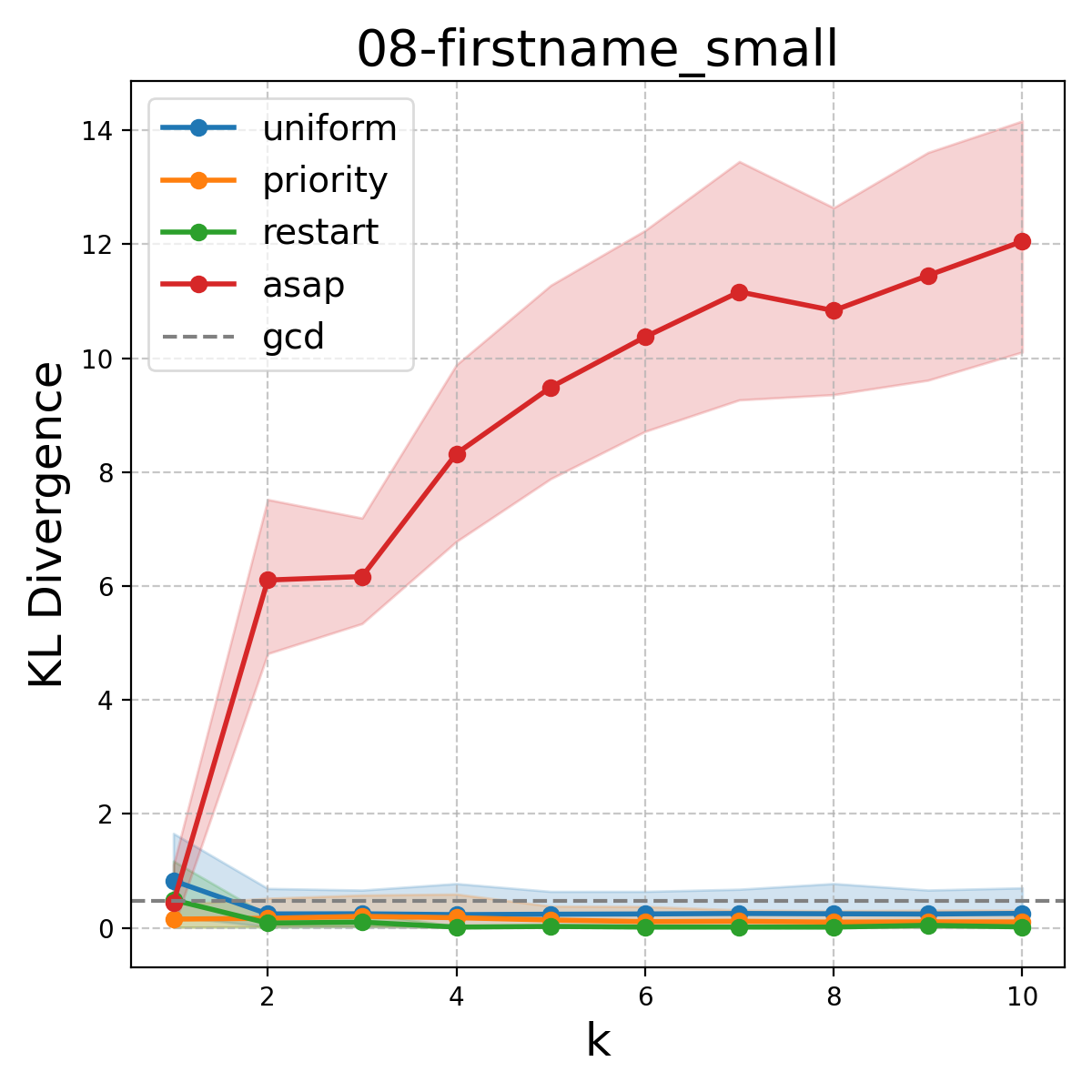} &
        \includegraphics[width=0.22\textwidth]{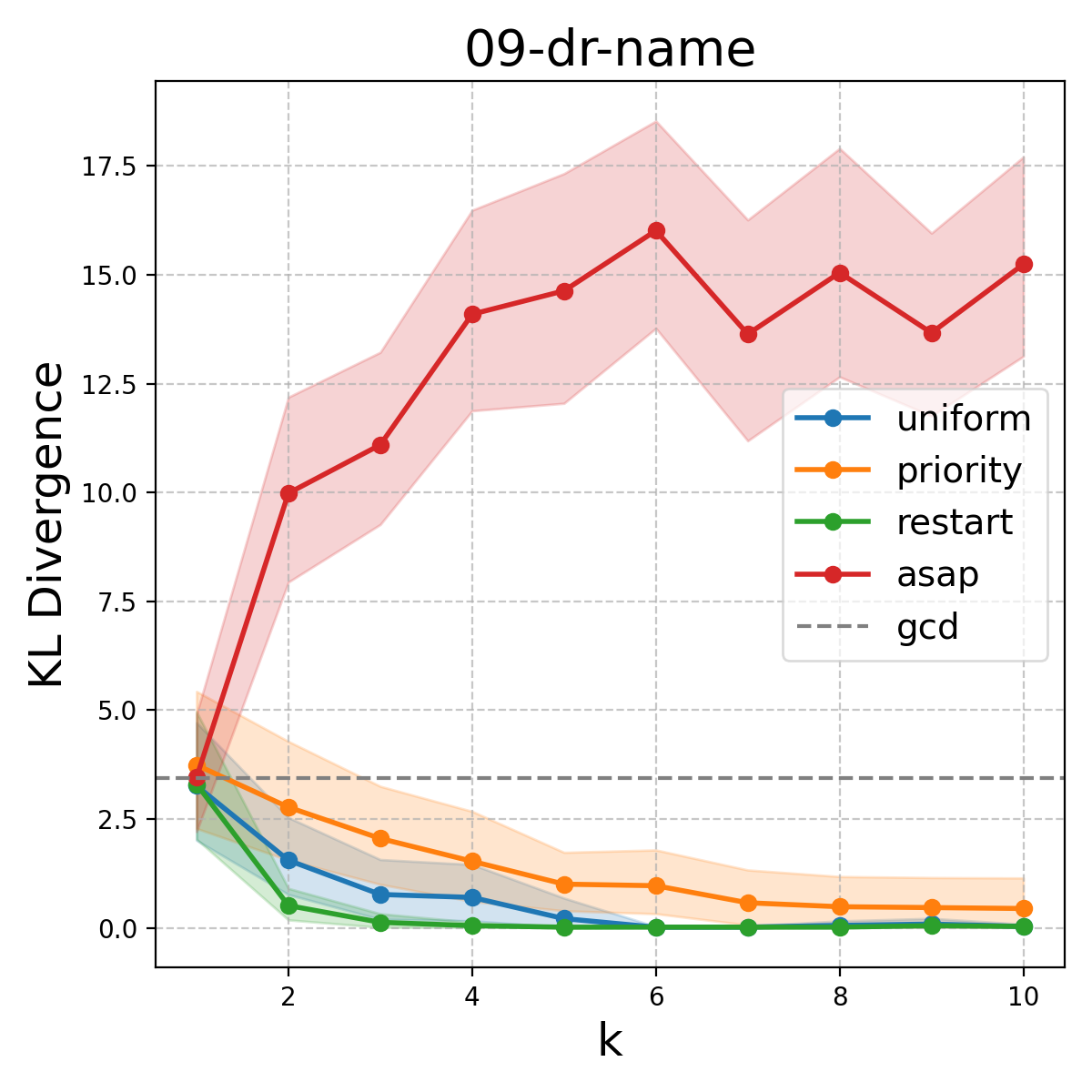} &
        \includegraphics[width=0.22\textwidth]{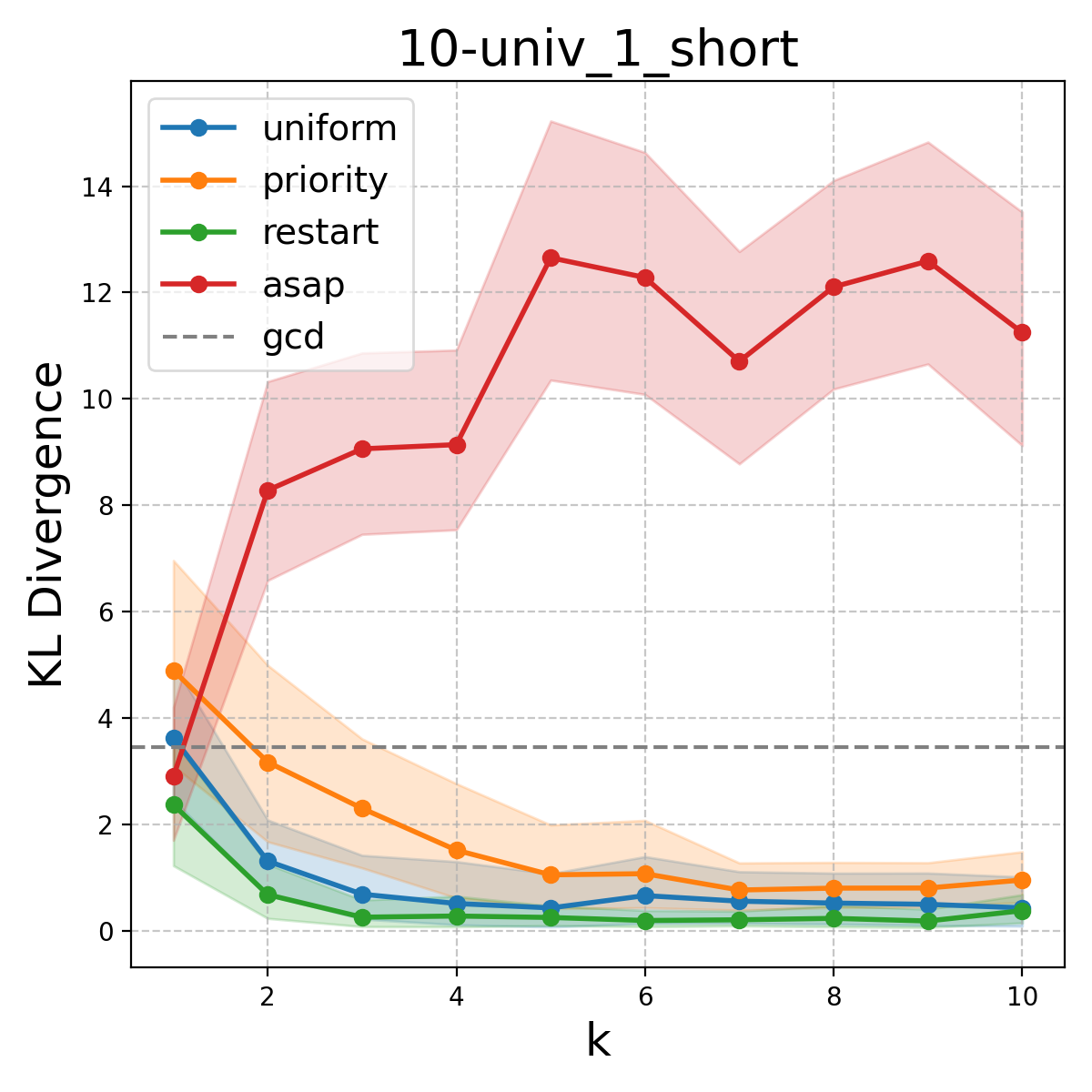} &
        \includegraphics[width=0.22\textwidth]{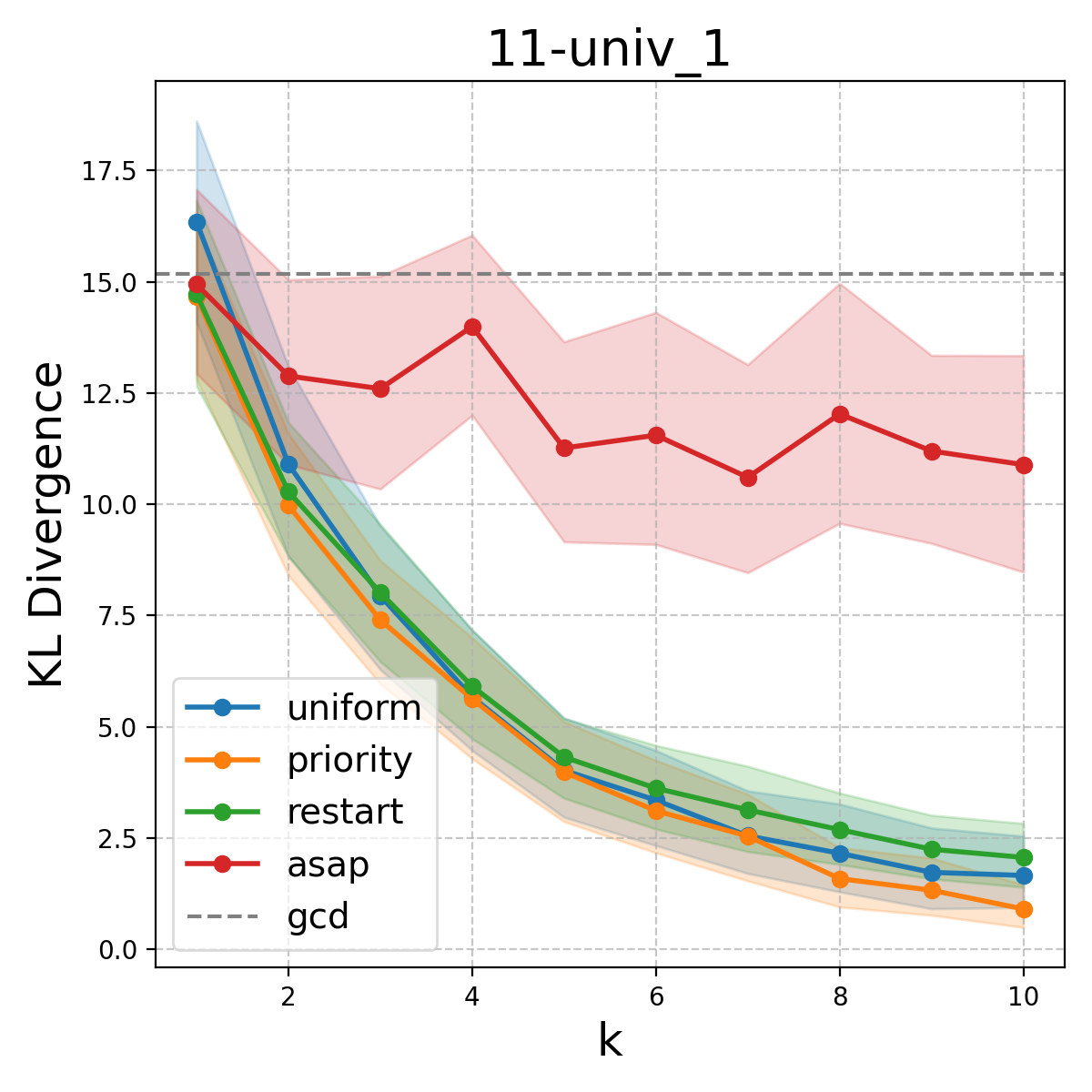} \\
        \includegraphics[width=0.22\textwidth]{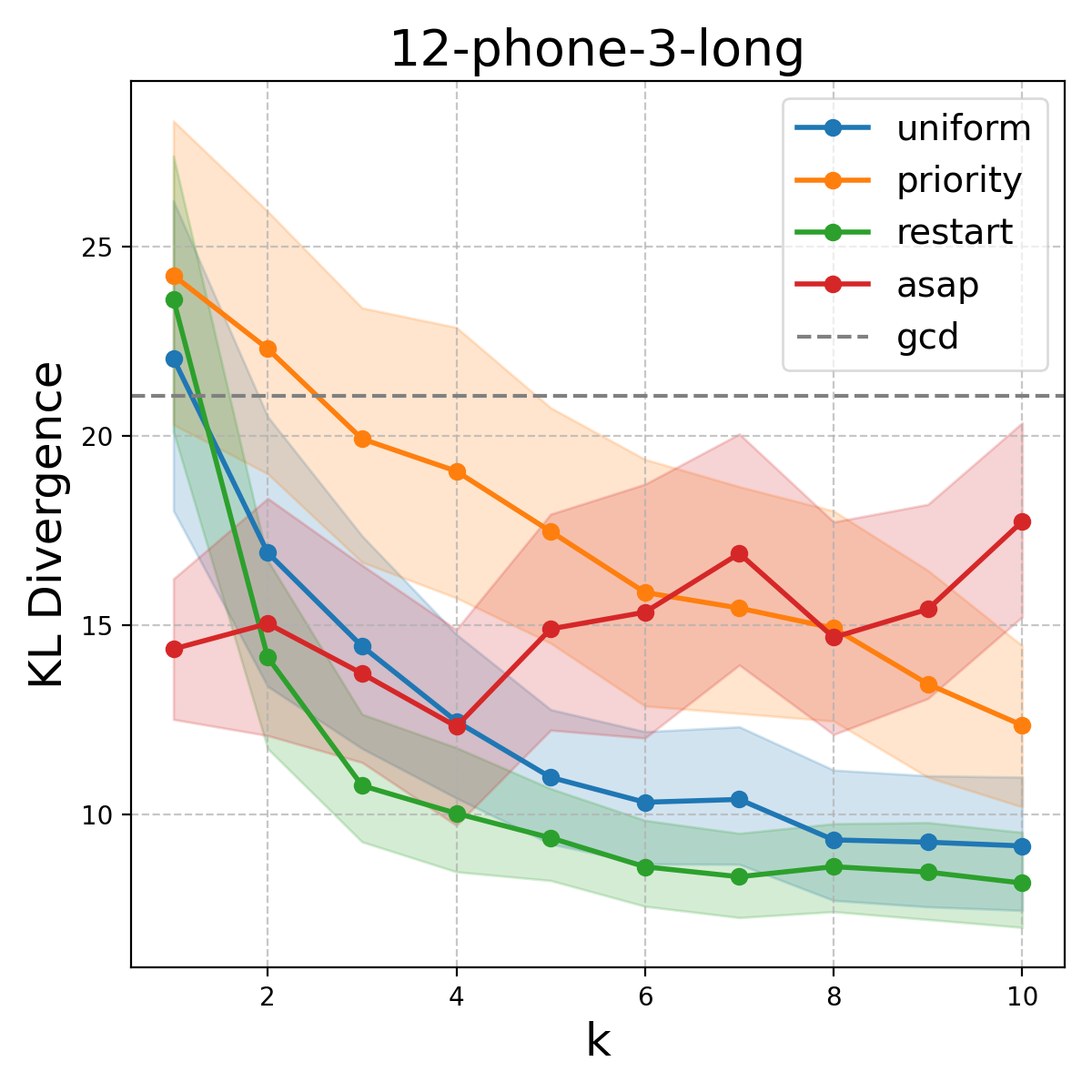} &
        \includegraphics[width=0.22\textwidth]{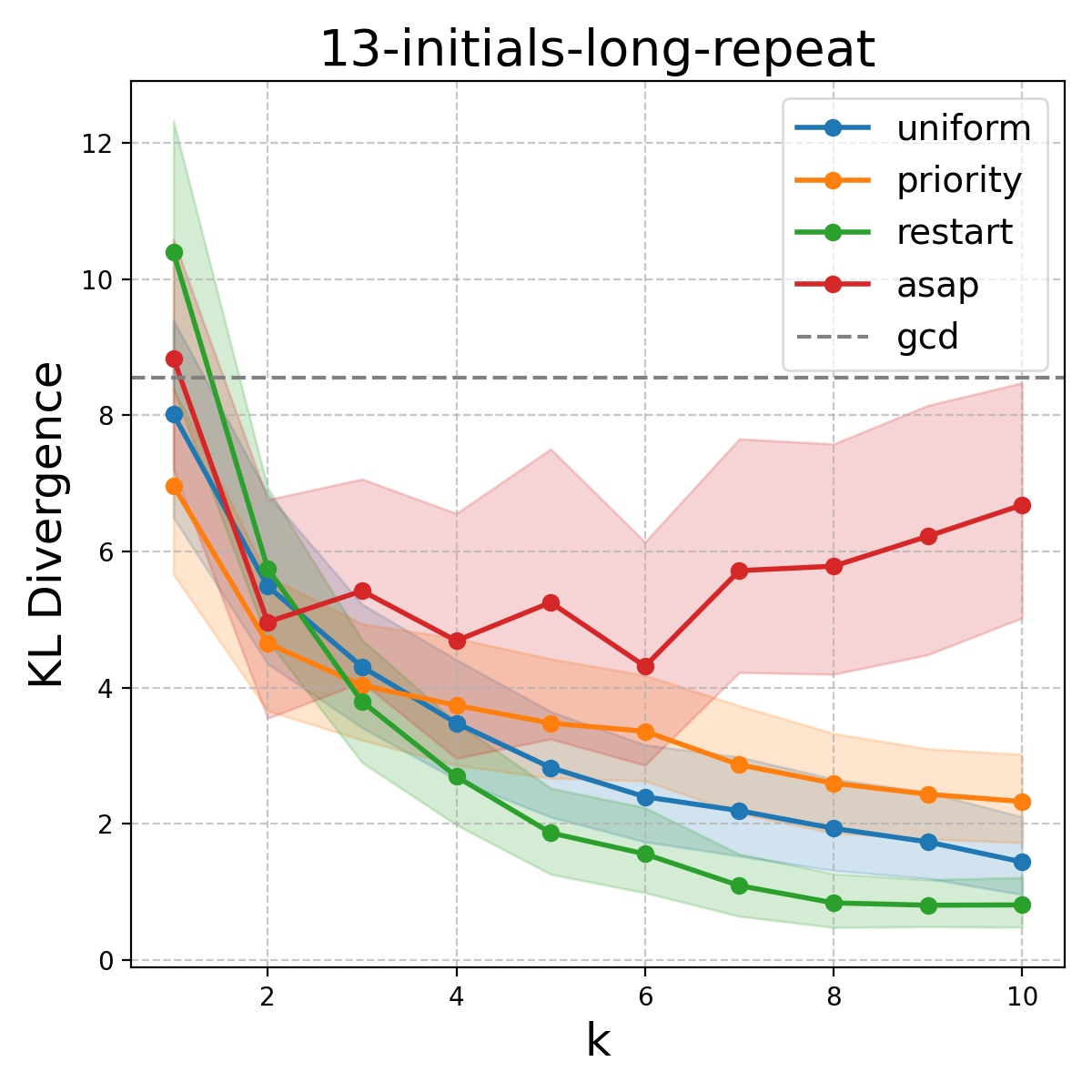} &
        \includegraphics[width=0.22\textwidth]{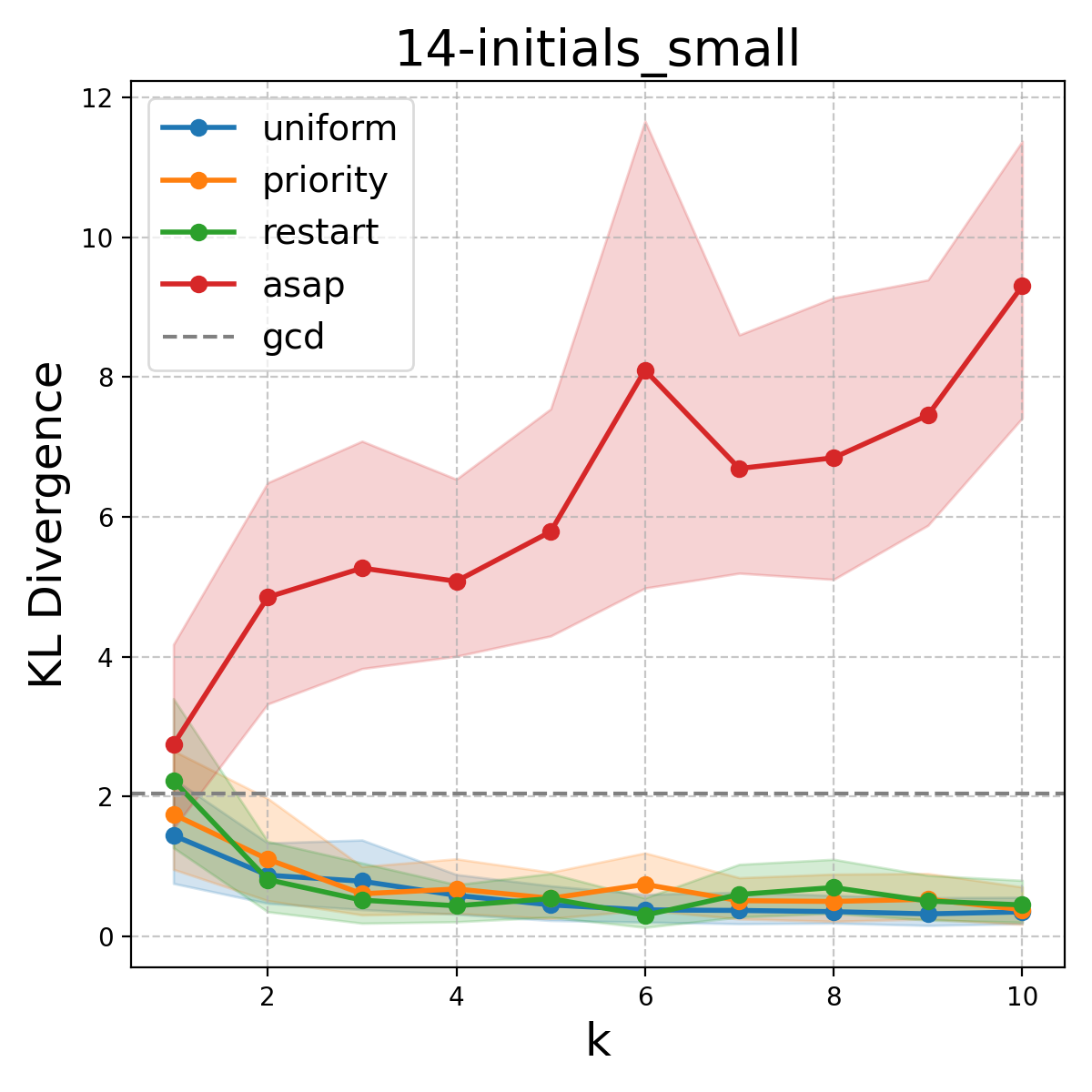} \\
    \end{tabular}
    \caption{KL-Divergence for ASAp($k$) and MCMC($k$) in SLIA subset}
    \label{fig:slia_kl}
\end{figure}

\begin{figure}[H]
    \centering
    \begin{tabular}{cccc}
        \includegraphics[width=0.22\textwidth]{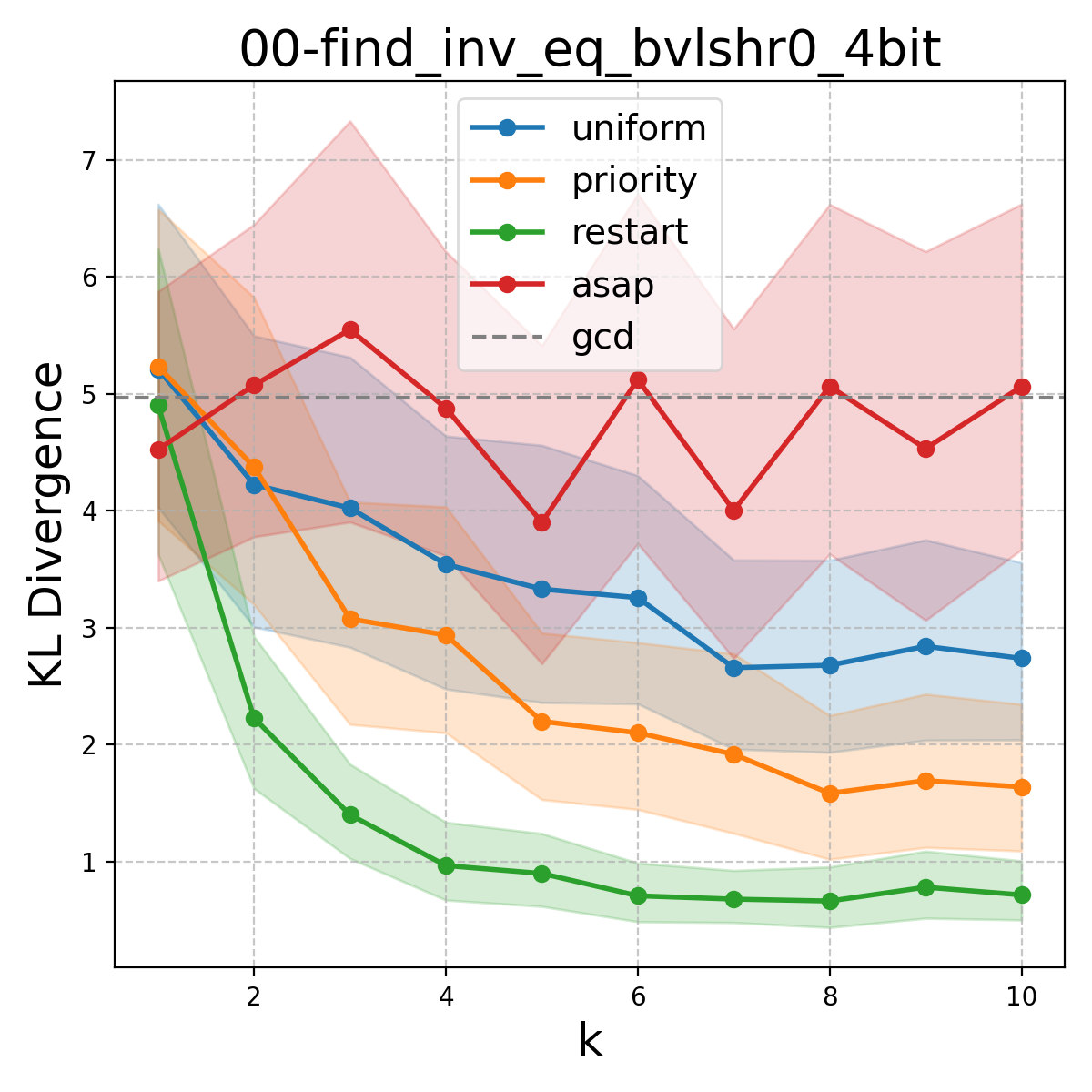} &
        \includegraphics[width=0.22\textwidth]{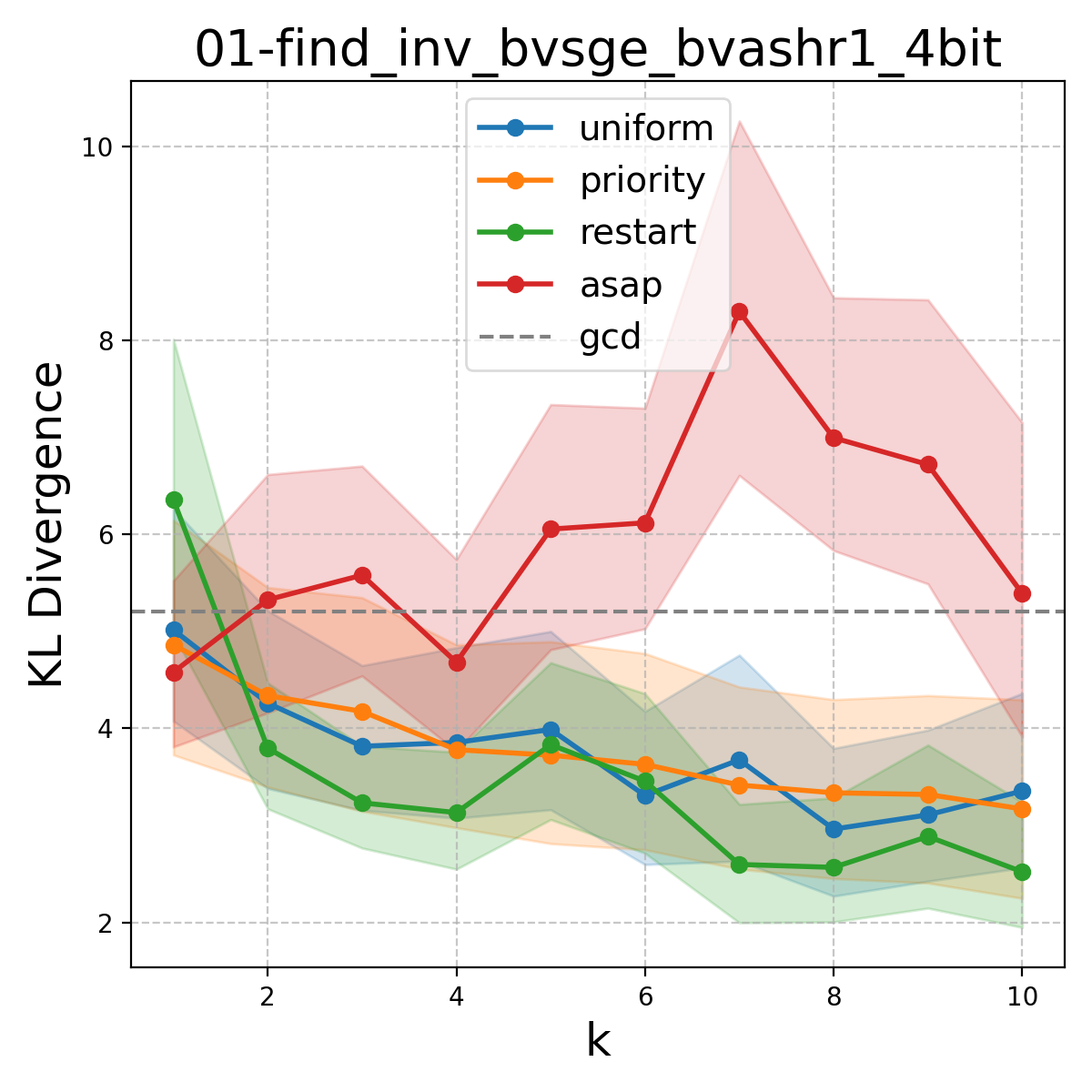} &
        \includegraphics[width=0.22\textwidth]{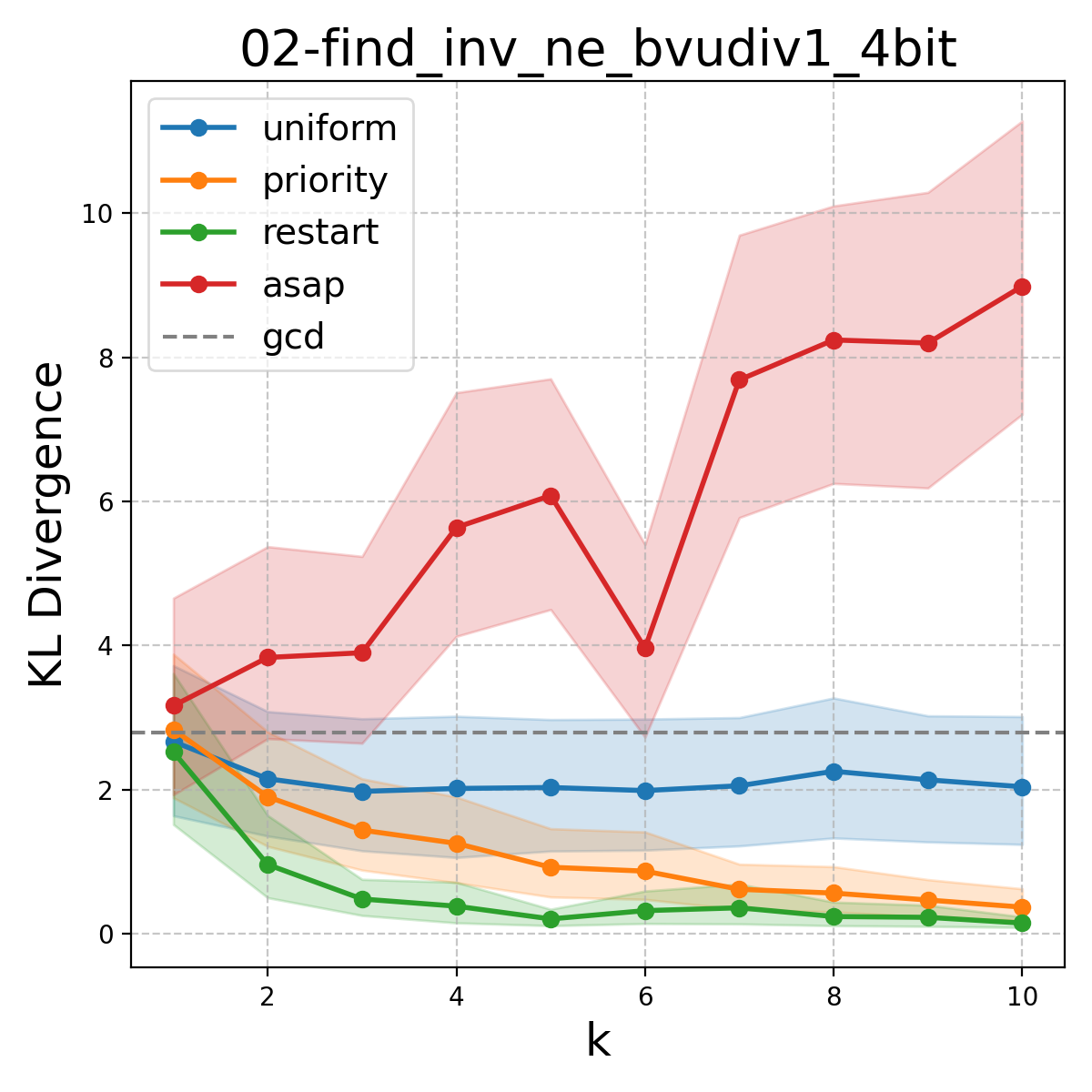} &
        \includegraphics[width=0.22\textwidth]{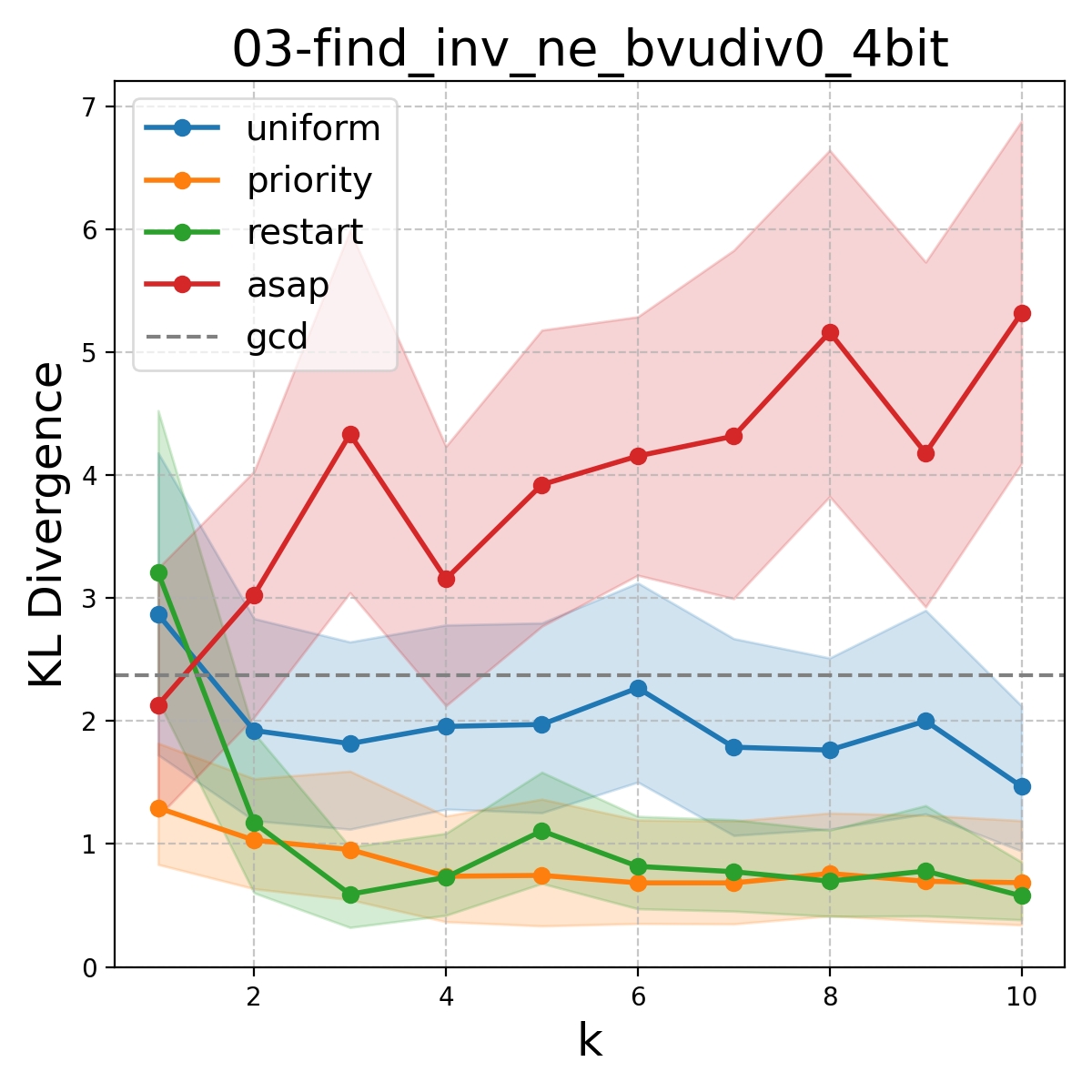} \\
        \includegraphics[width=0.22\textwidth]{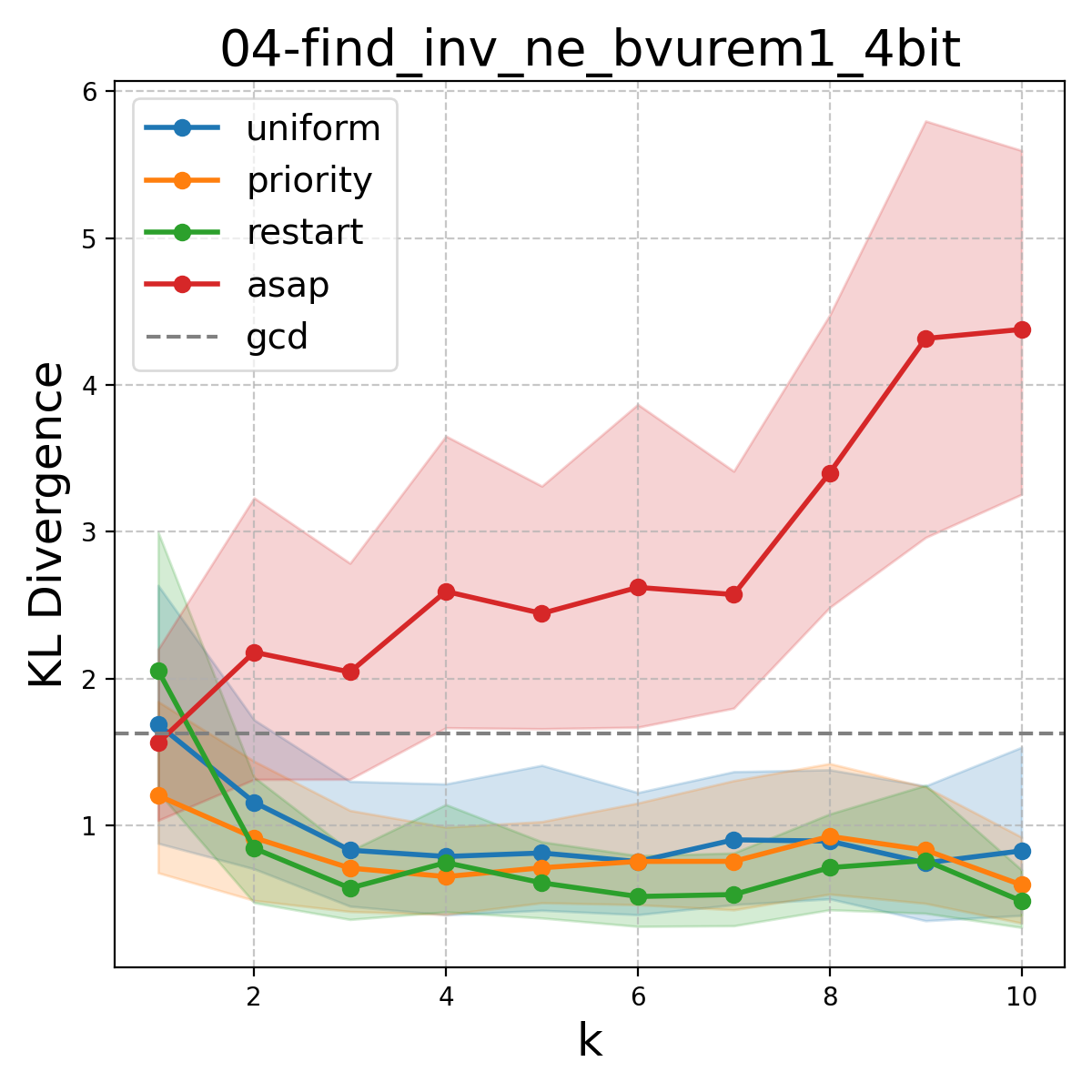} &
        \includegraphics[width=0.22\textwidth]{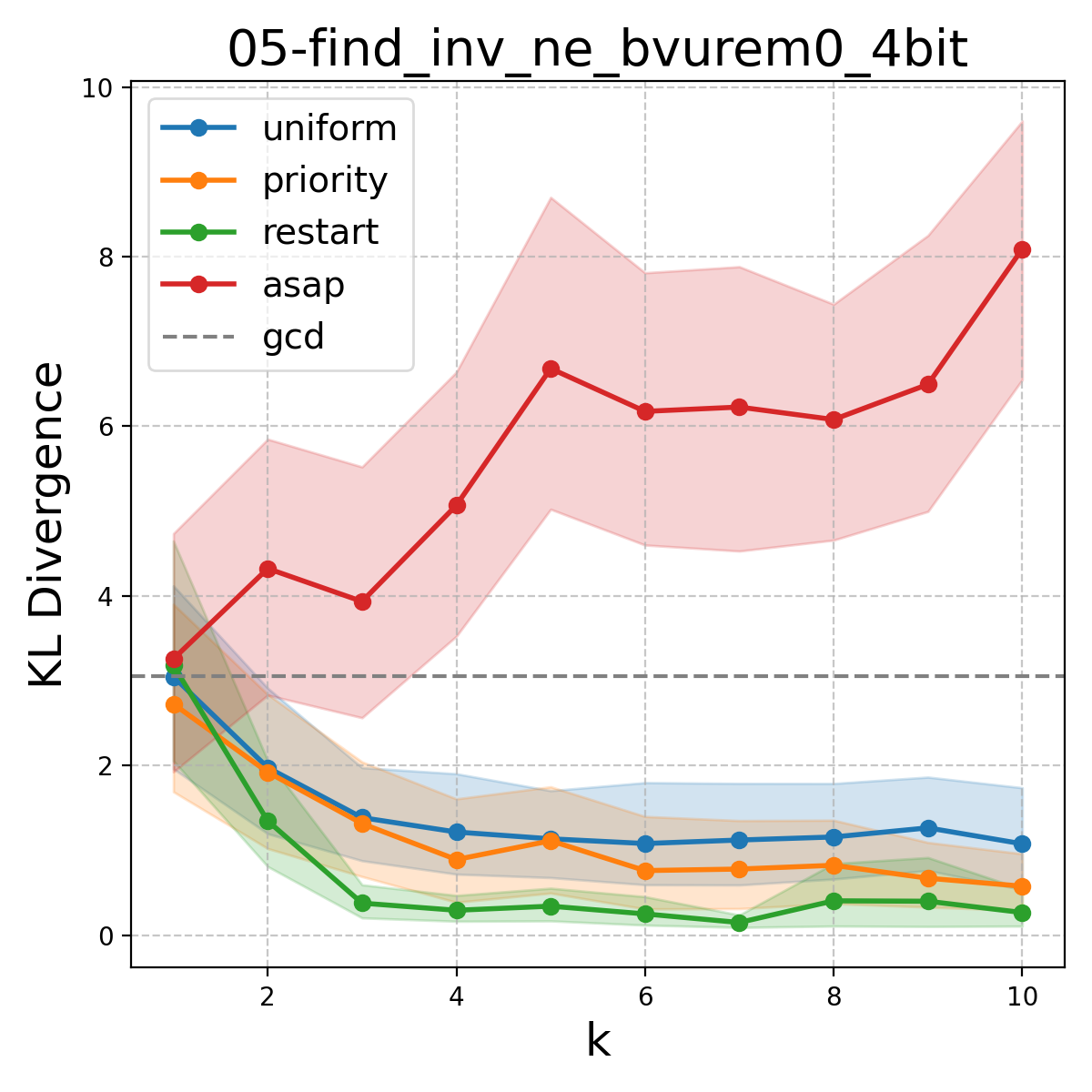} &
        \includegraphics[width=0.22\textwidth]{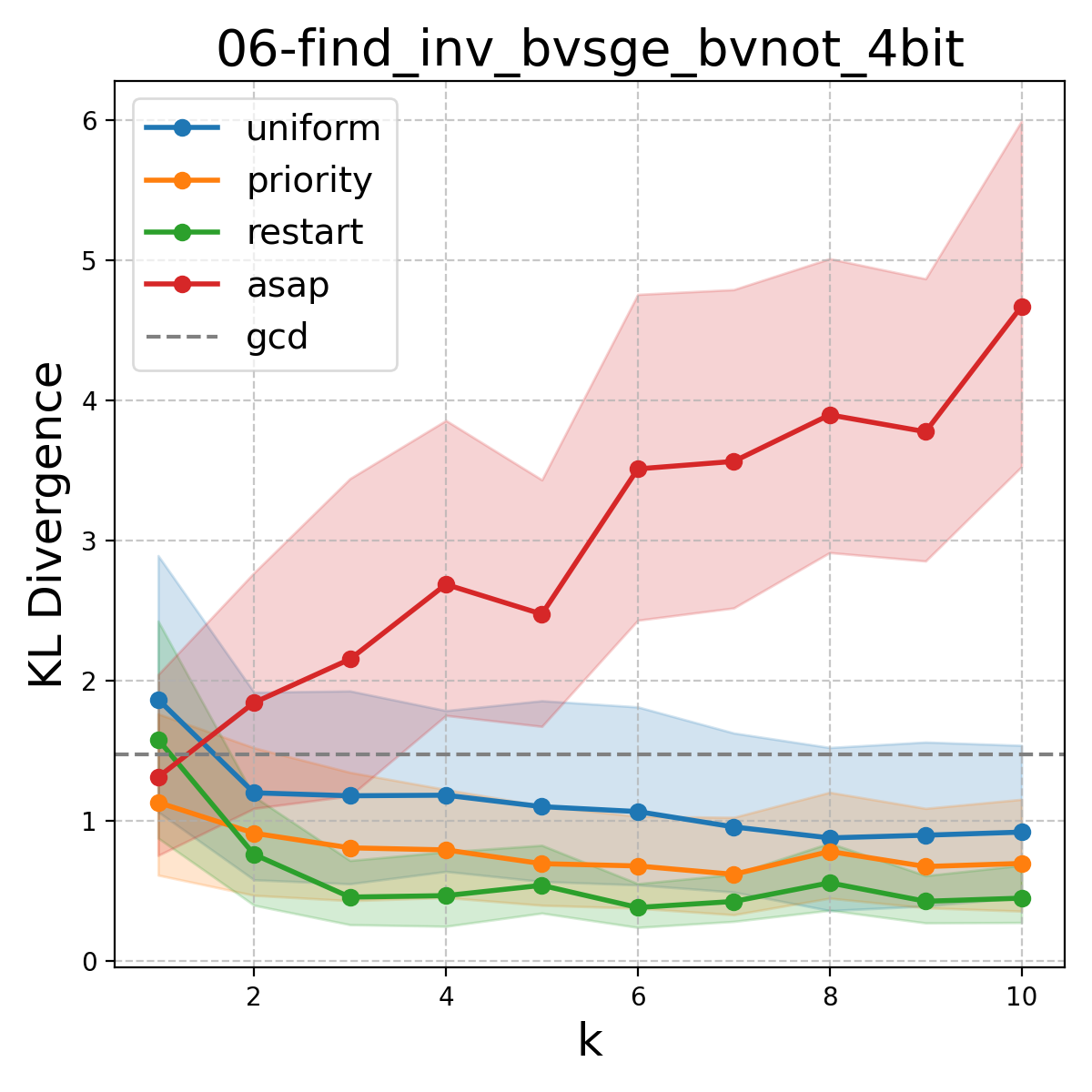} &
        \includegraphics[width=0.22\textwidth]{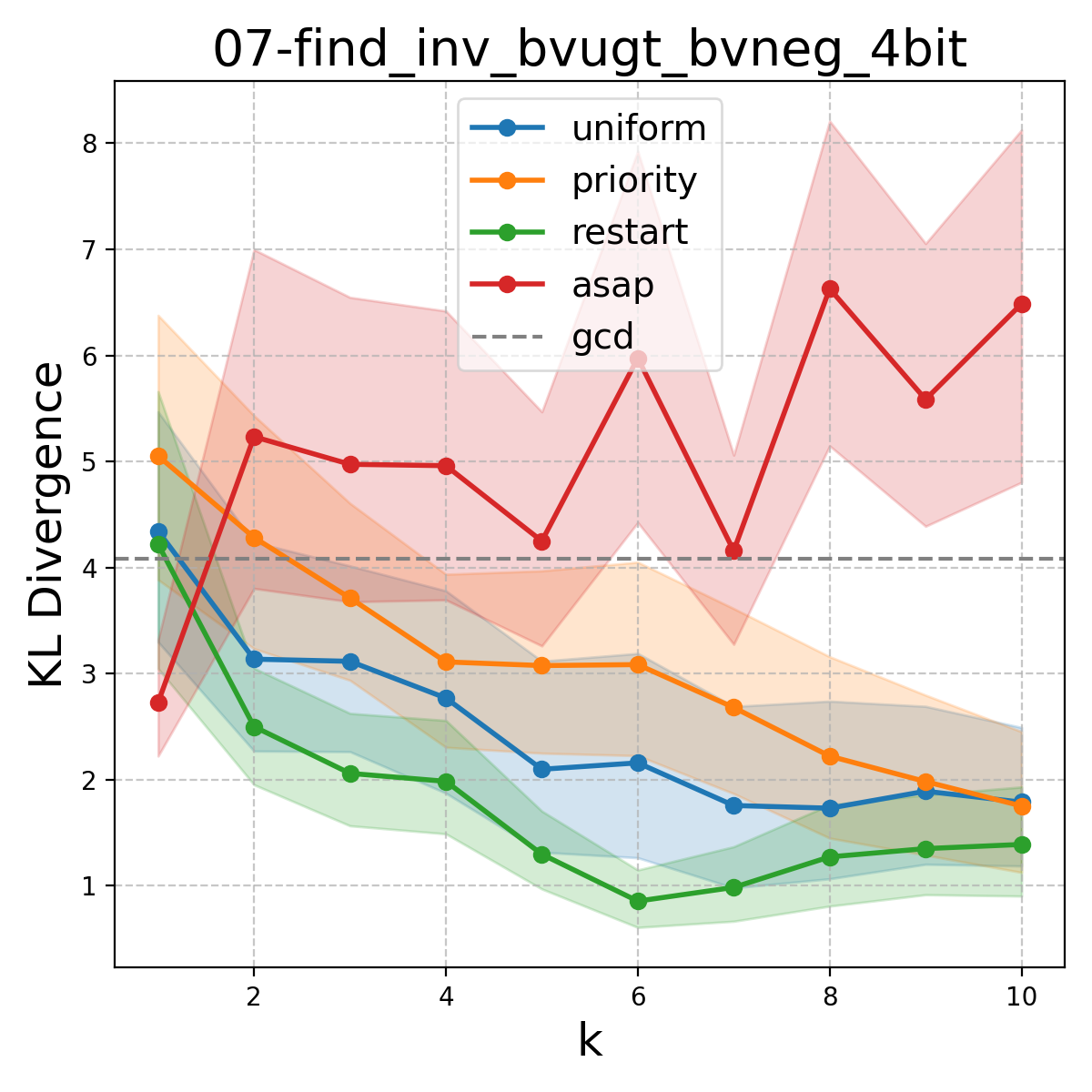} \\
        \includegraphics[width=0.22\textwidth]{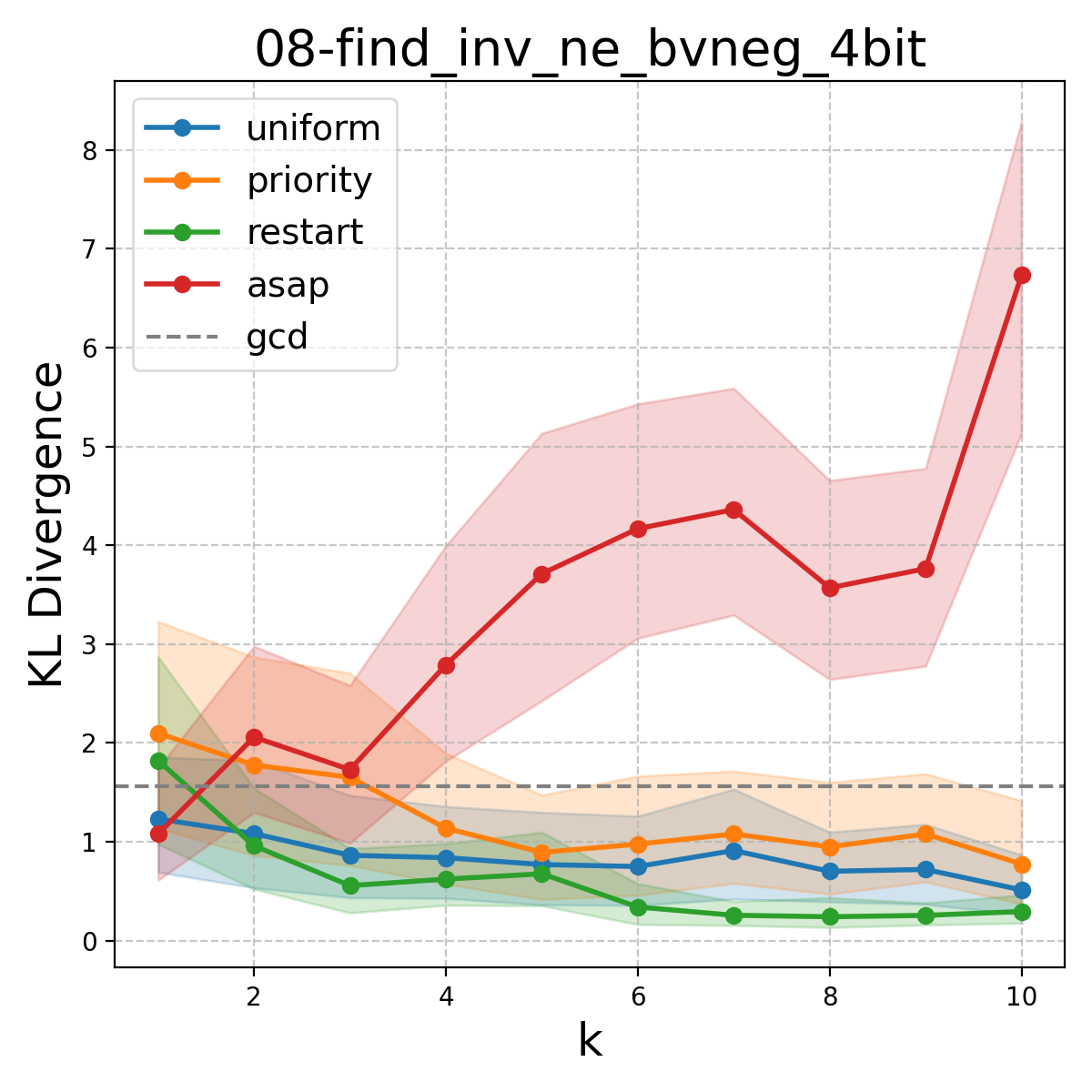} &
        \includegraphics[width=0.22\textwidth]{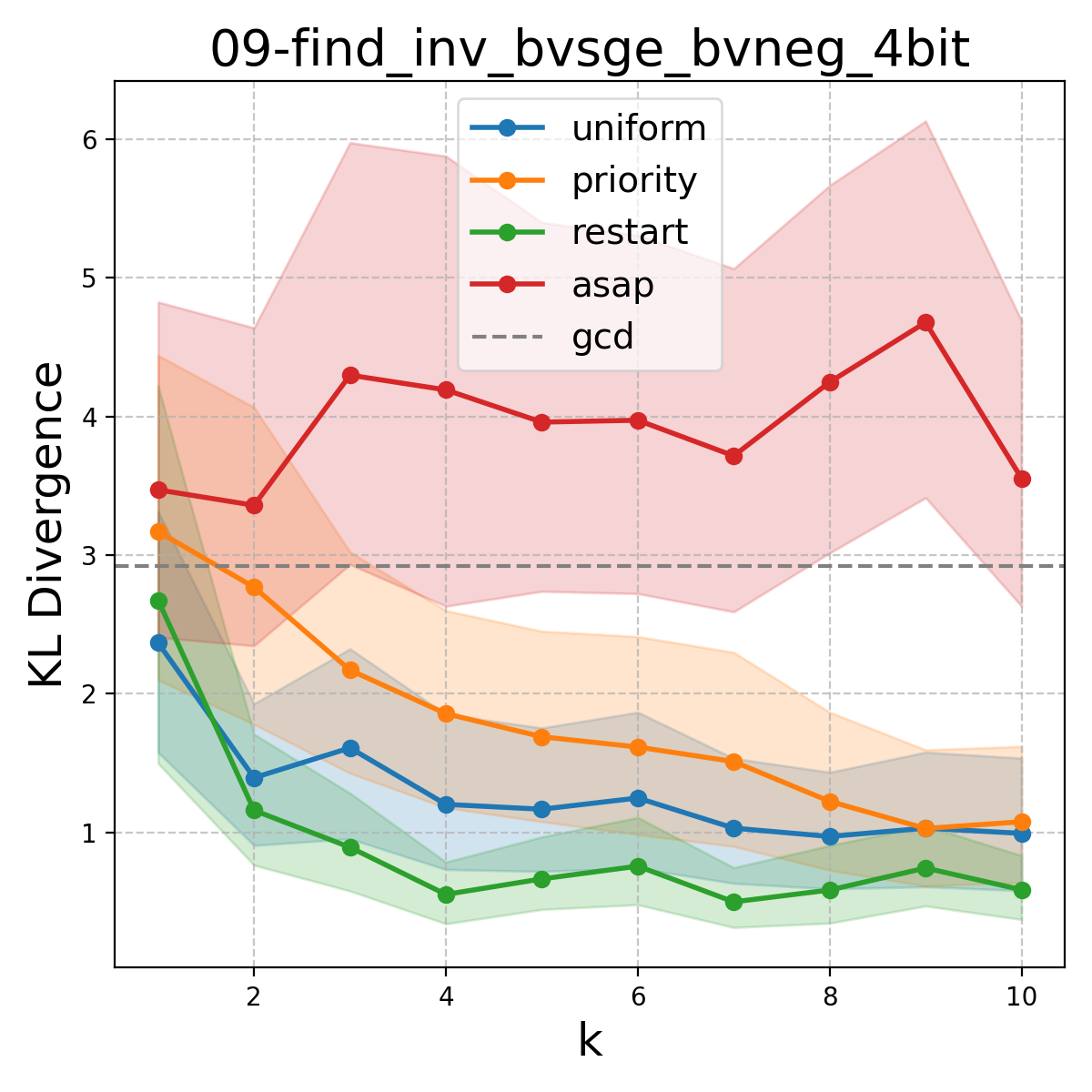} &
        \includegraphics[width=0.22\textwidth]{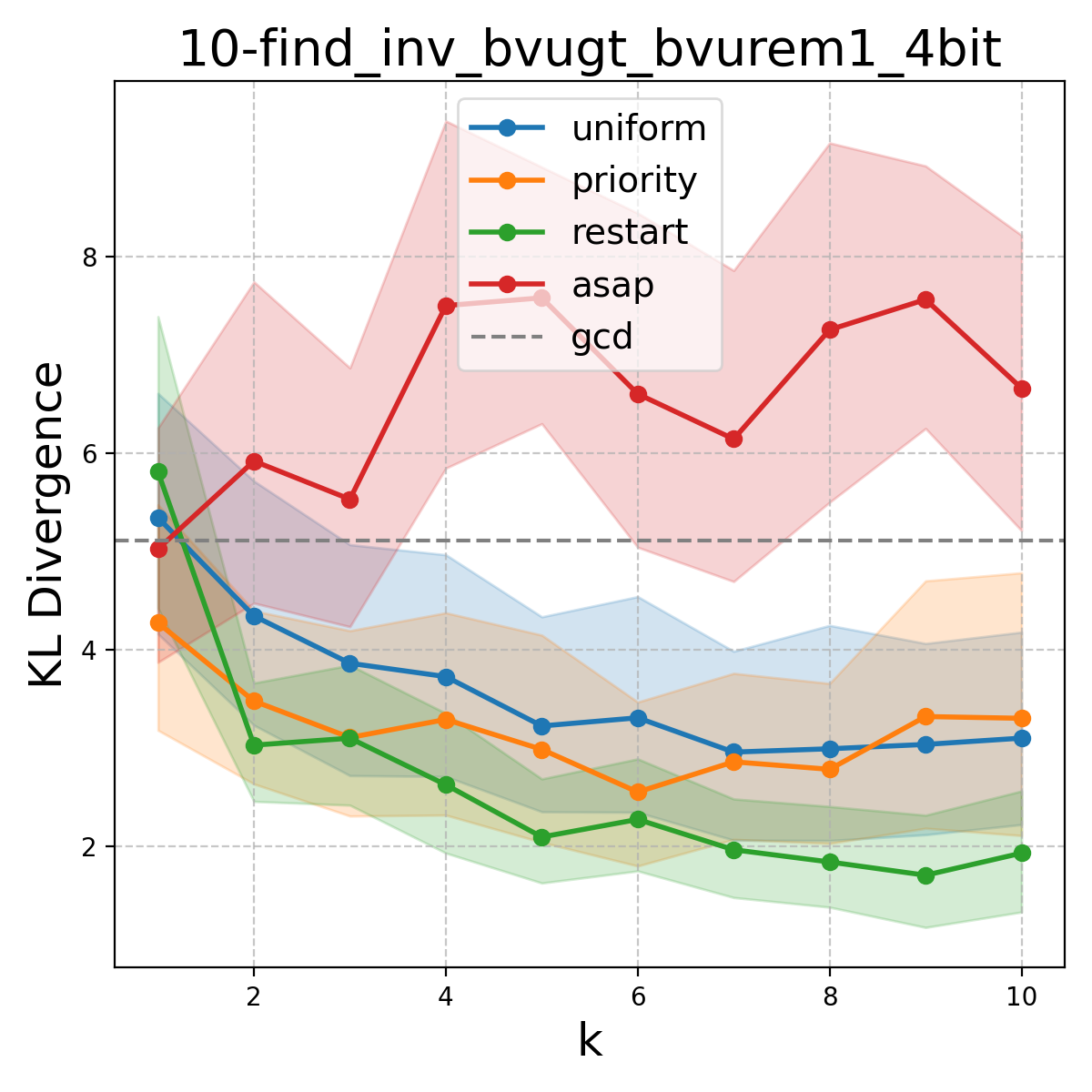} &
        \includegraphics[width=0.22\textwidth]{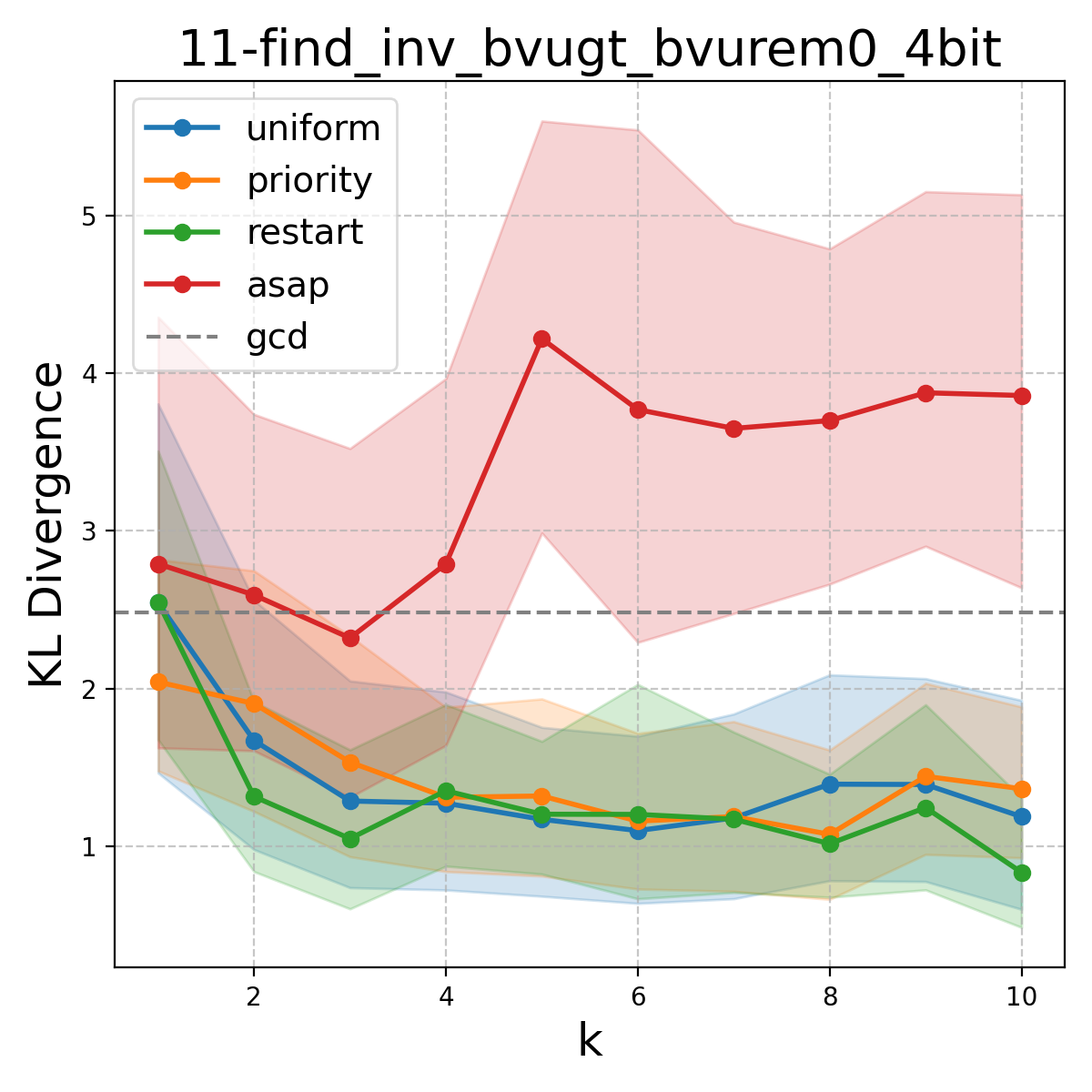} \\
        \includegraphics[width=0.22\textwidth]{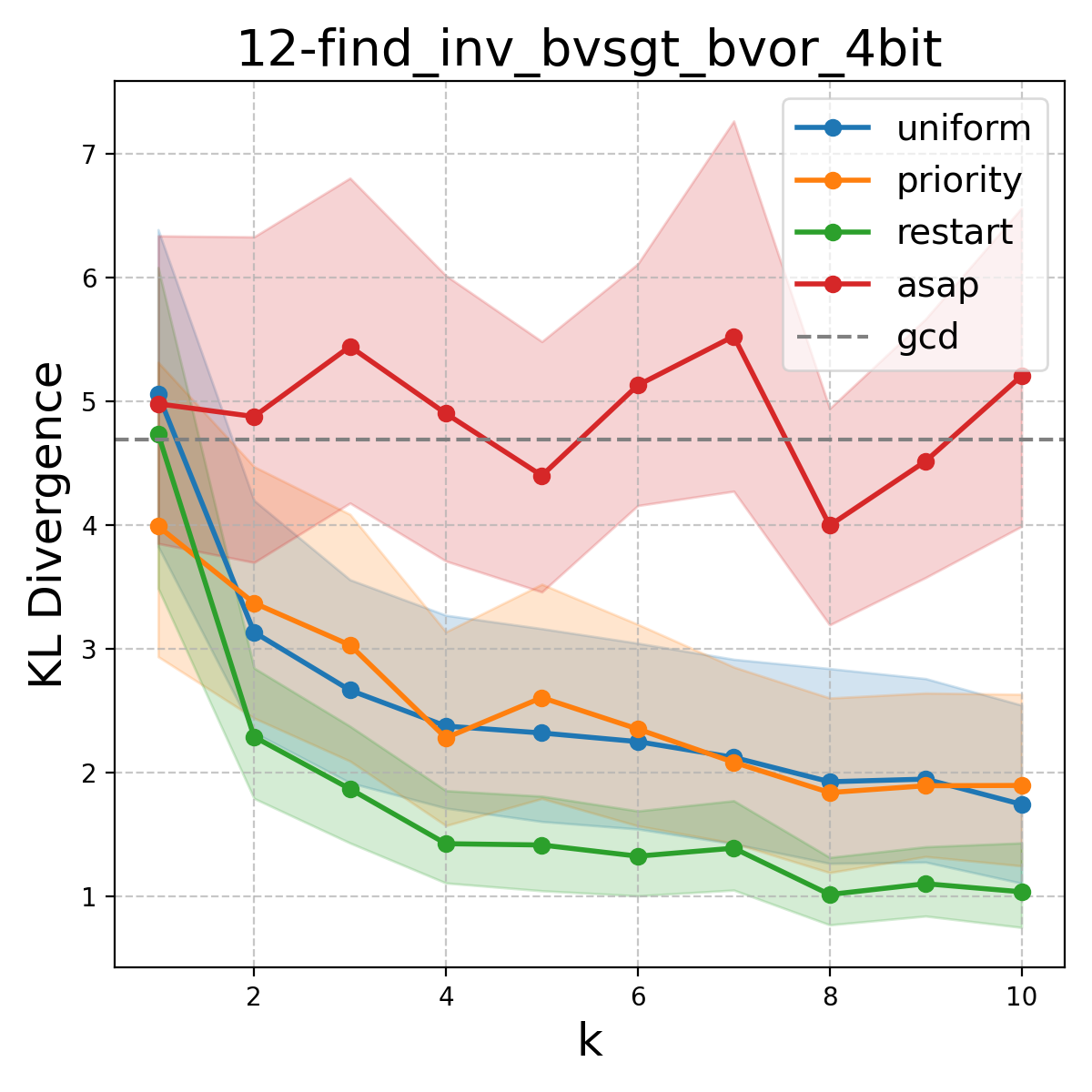} &
        \includegraphics[width=0.22\textwidth]{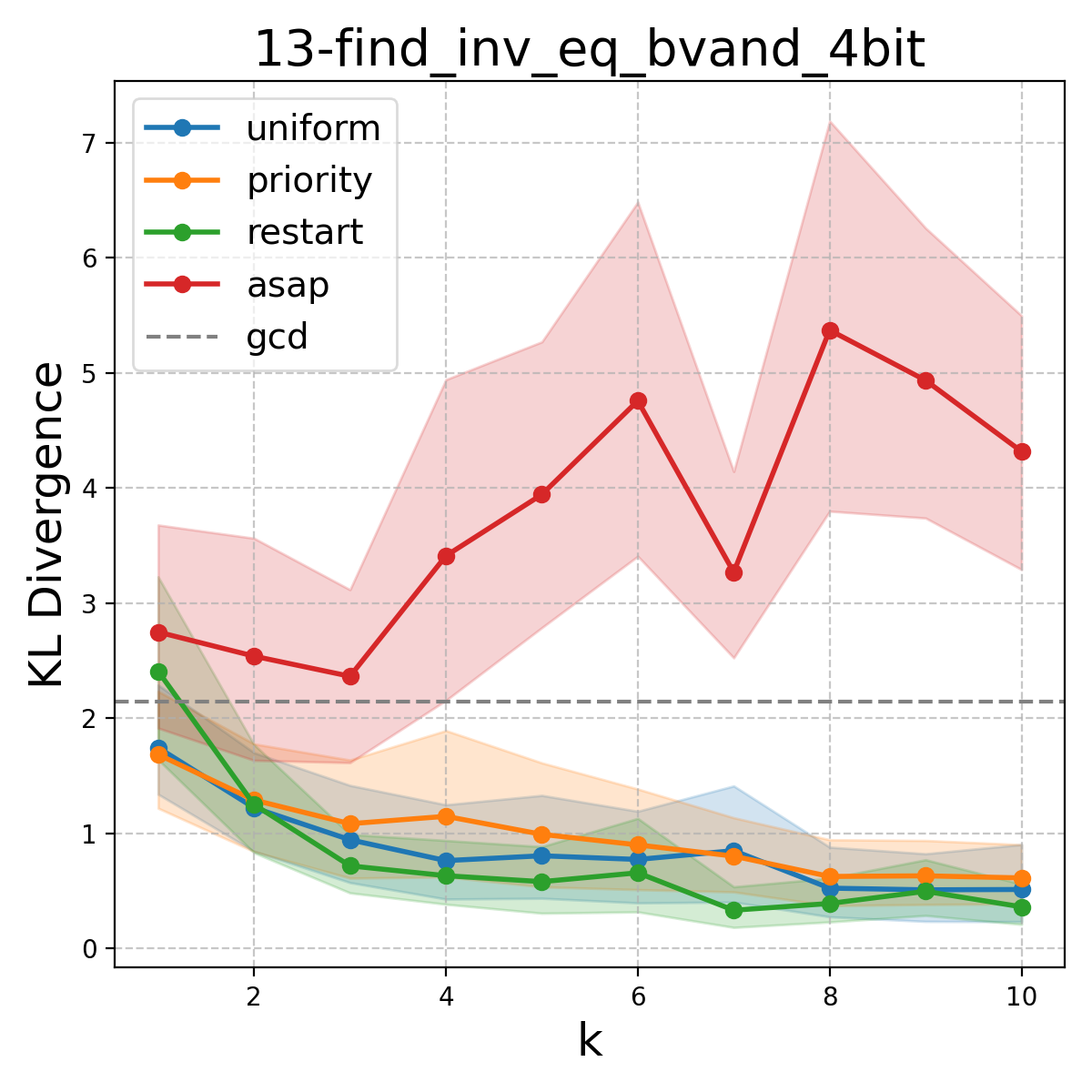} \\
    \end{tabular}
    \caption{KL-Divergence for ASAp($k$) and MCMC($k$) in BV4 subset}
    \label{fig:bv4_kl}
\end{figure}

\begin{figure}[H]
    \centering
    \begin{tabular}{cccc}
        \includegraphics[width=0.22\textwidth]{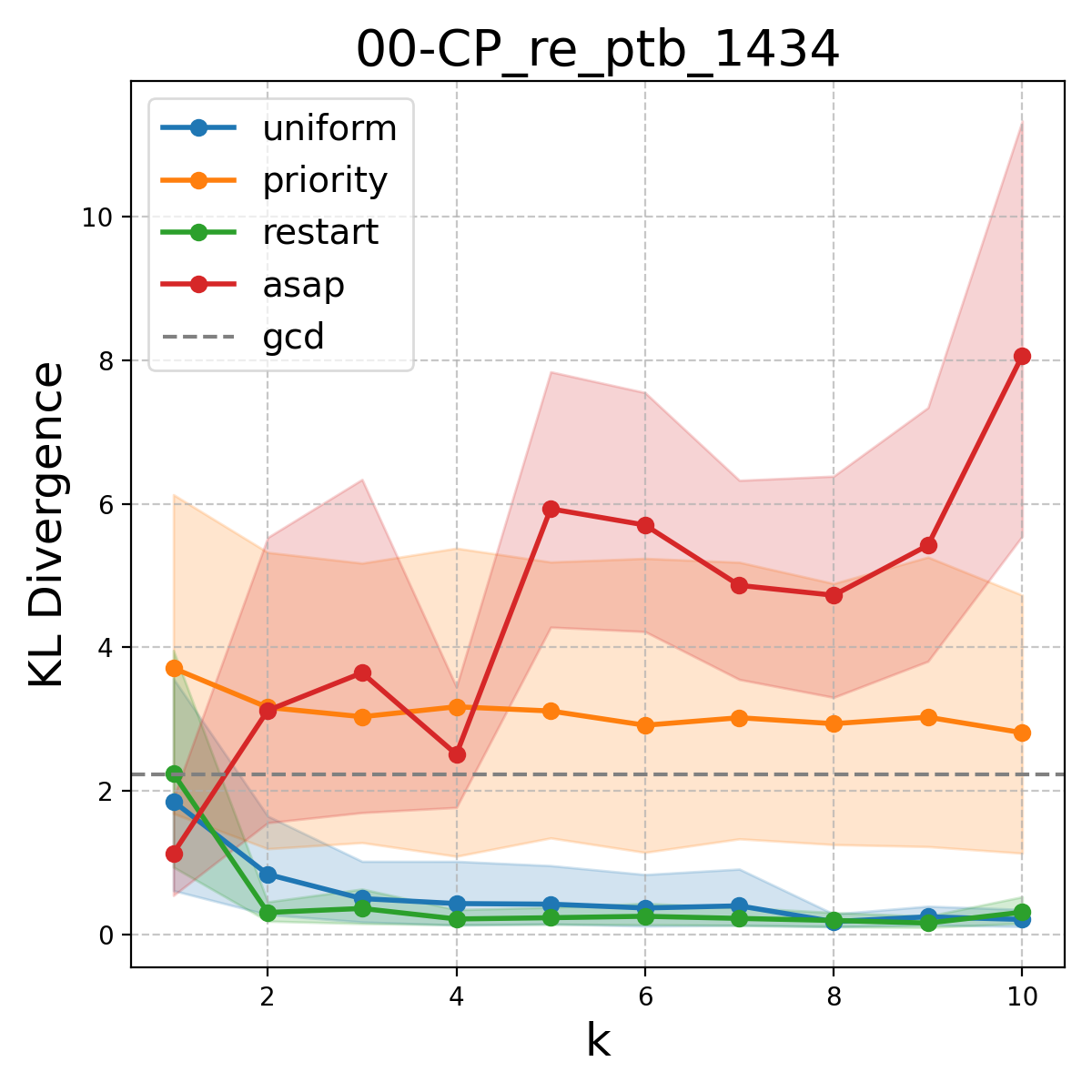} &
        \includegraphics[width=0.22\textwidth]{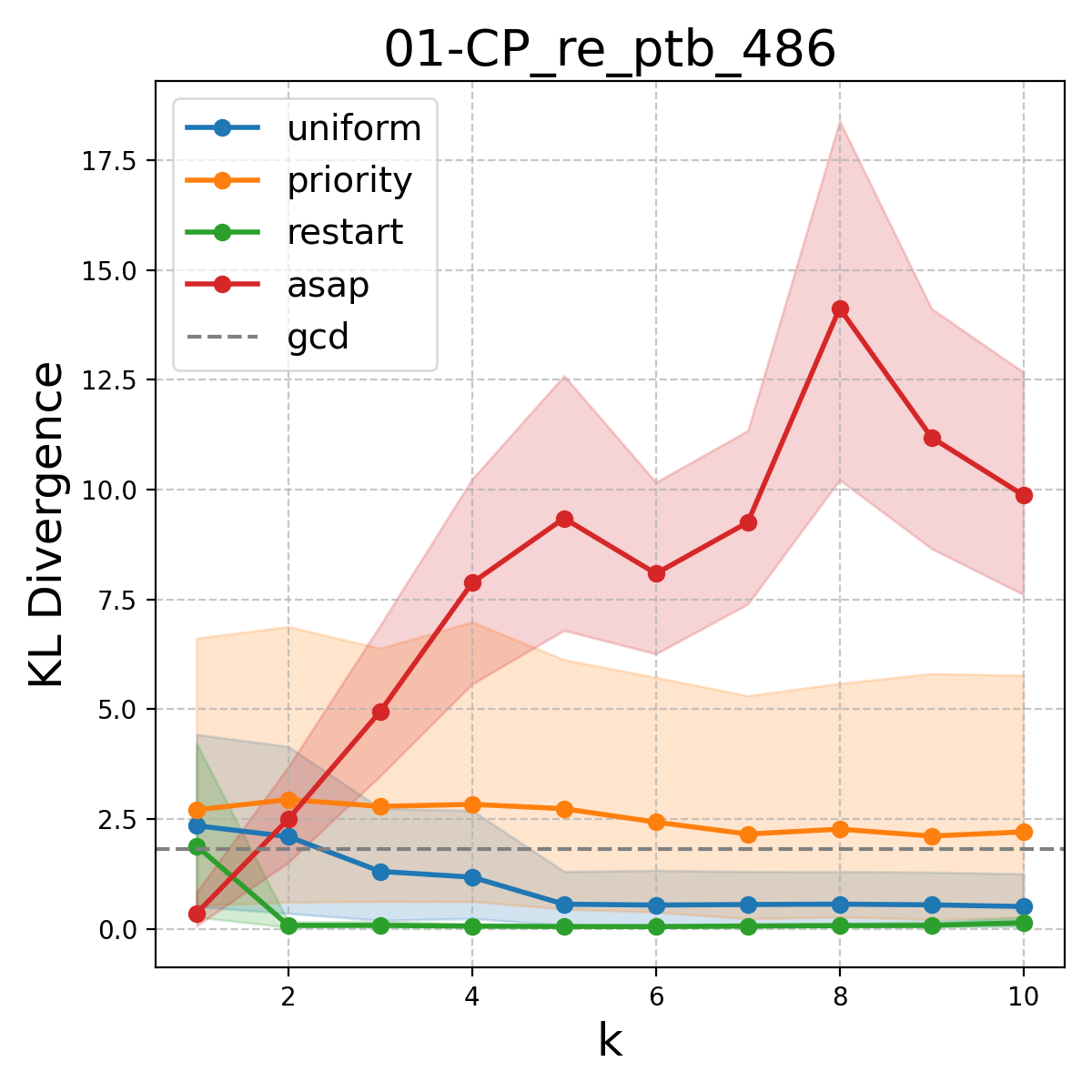} &
        \includegraphics[width=0.22\textwidth]{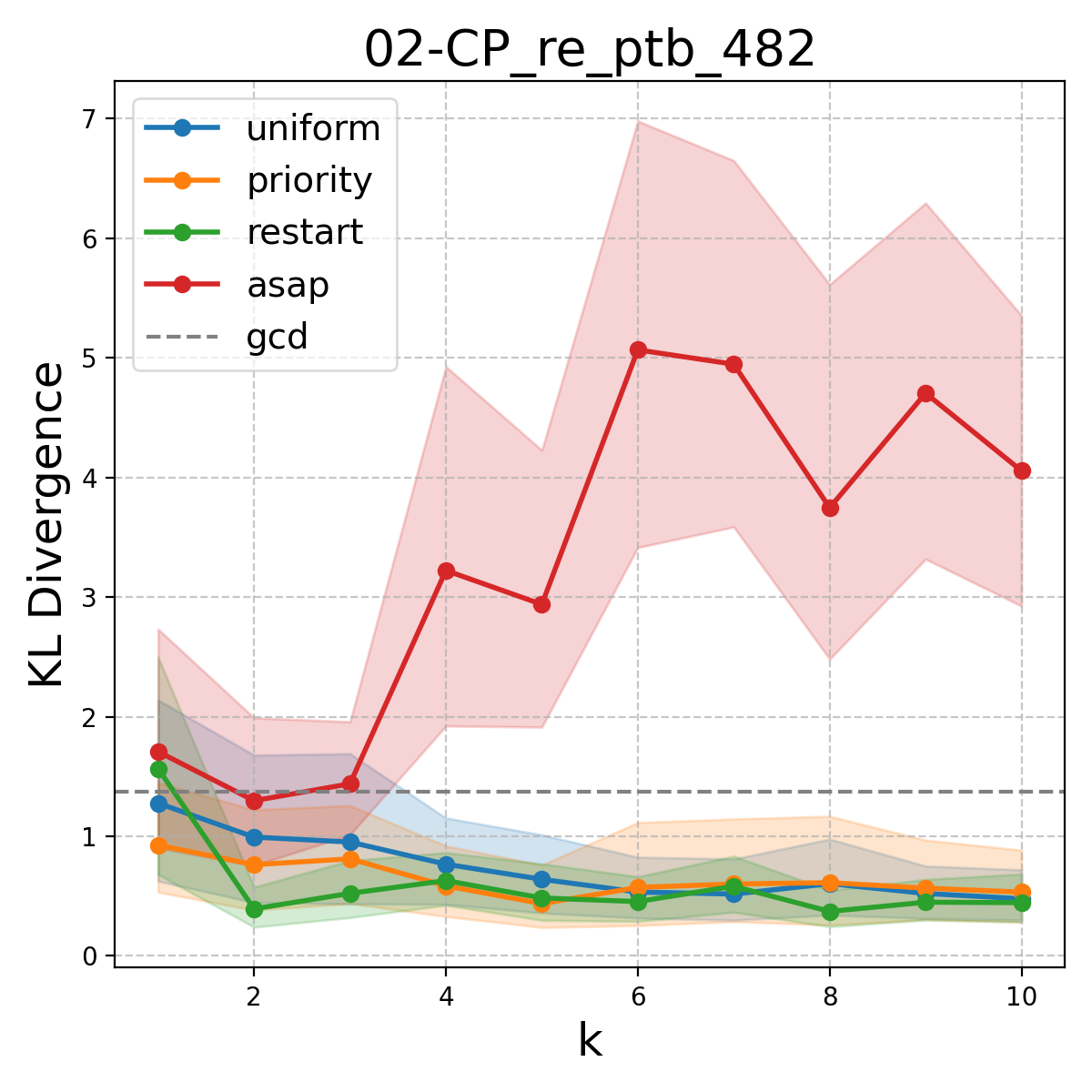} &
        \includegraphics[width=0.22\textwidth]{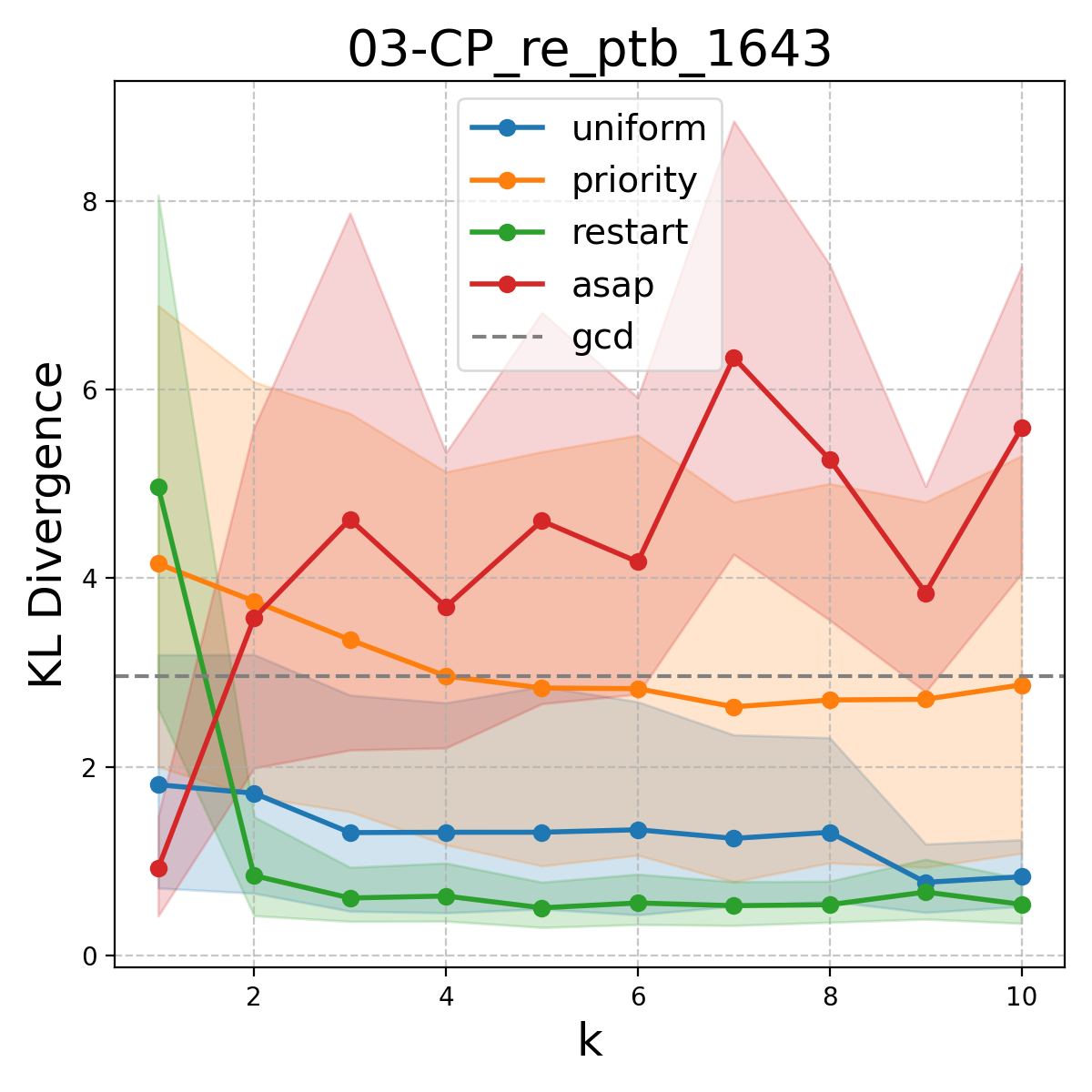} \\
        \includegraphics[width=0.22\textwidth]{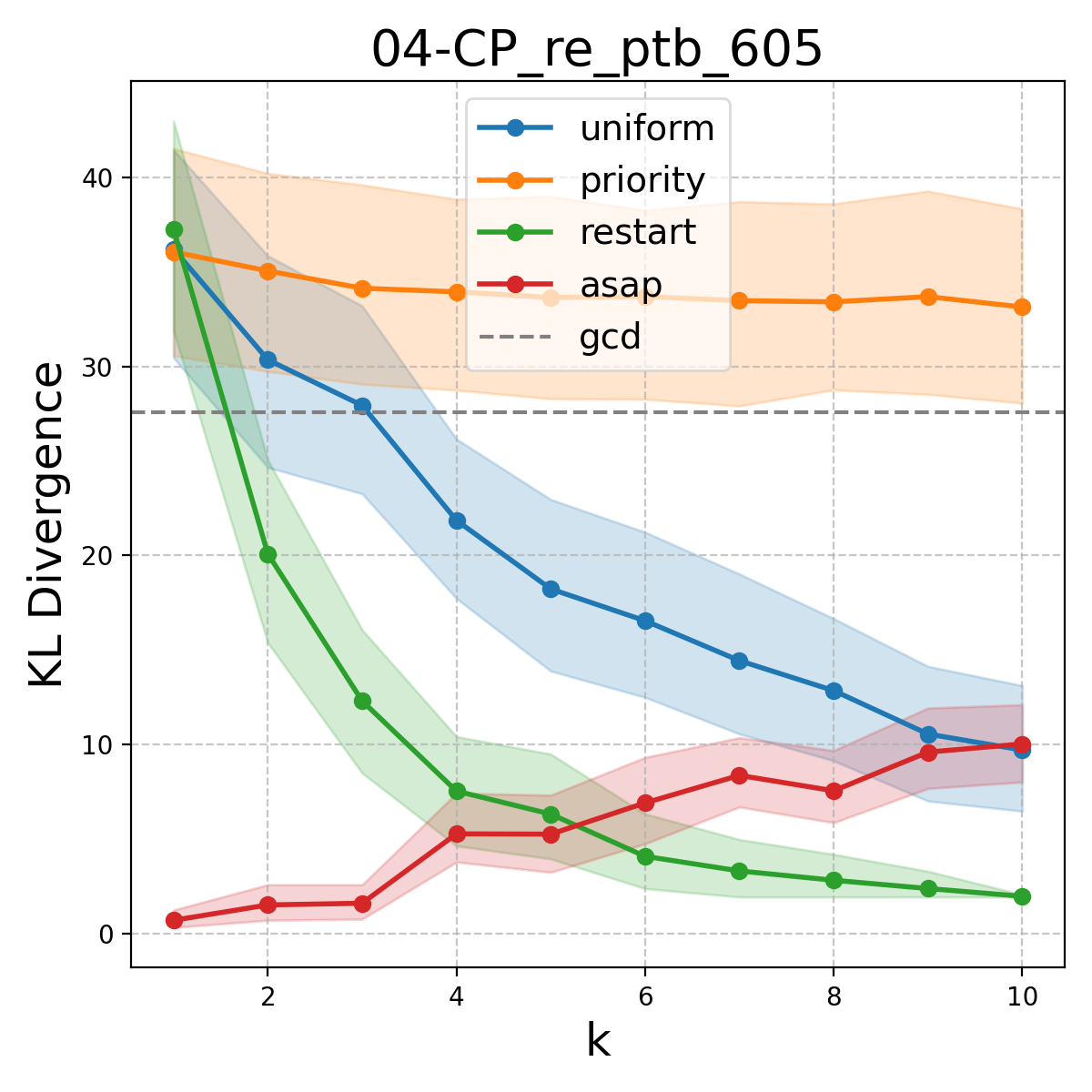} &
        \includegraphics[width=0.22\textwidth]{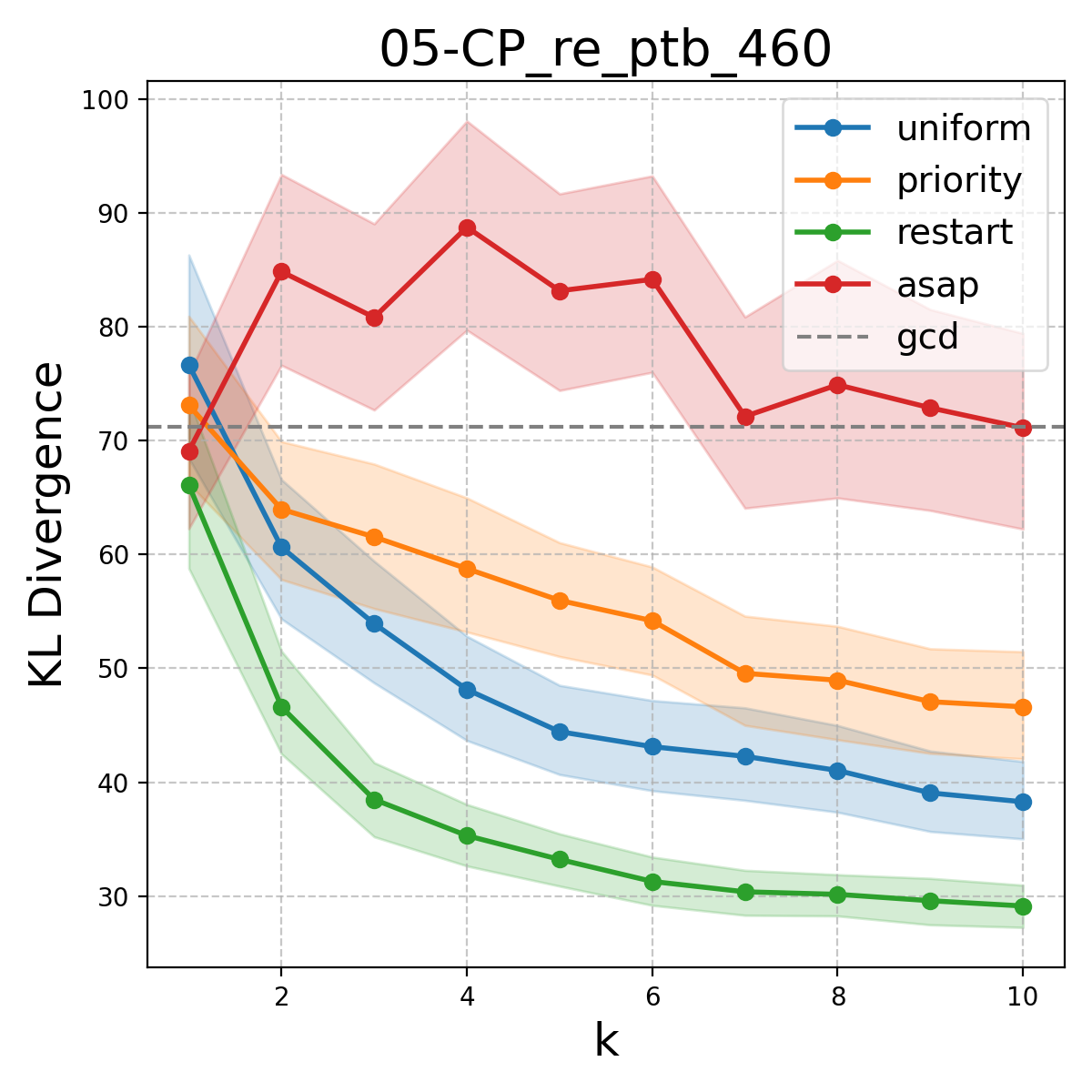} \\
    \end{tabular}
    \caption{KL-Divergence for ASAp($k$) and MCMC($k$) in CP subset}
    \label{fig:cp_kl}
\end{figure}

\end{document}